\documentclass[11pt]{article}
\usepackage{latexsym,amsmath,amscd,amssymb,graphics,mathrsfs}
\usepackage{multirow}
\usepackage{graphicx}
\usepackage{enumerate}
\usepackage{graphicx}
\usepackage{color}
\usepackage{soul}
\usepackage{calrsfs}
\usepackage{hyperref}
\usepackage{comment}
\usepackage{geometry}
\usepackage{subcaption}
\usepackage{changepage}
\usepackage{url}
\usepackage{xcolor}
\usepackage{tabularx}
\usepackage{breakcites}
\usepackage{hyperref}
\usepackage{cleveref}
\usepackage[tracking=true]{microtype}
\usepackage{changepage}
\usepackage{geometry}
\usepackage{amsthm}
\newcommand{\rem}[1]{}

\newcommand\on{\operatorname}

\renewcommand\prod{\on{prod}}

\newtheorem{theorem}{Theorem}[section]
\newtheorem{Definition}[theorem]{Definition}

\newtheorem{Remark}[theorem]{Remark}
\newtheorem{Proposition}[theorem]{Proposition}

\begin{document}

\title{Geometric Learning of Canonical Parameterizations of $2D$-curves}
\author{Ioana Ciuclea$^{1, 4}$, Giorgio Longari$^{2, 4}$, Alice Barbara Tumpach $^{3, 4}$ }
\addtocounter{footnote}{1}

\footnotetext{Faculty of Physics and Mathematics, Department of Mathematics, West University of Timi\c{s}oara, Vasile Pârvan 4,  300392 Timi\c{s}oara, Romania;
\texttt{ioana.ciuclea@e-uvt.ro}
\addtocounter{footnote}{1}}

\footnotetext{Computer Vision Lab, Technische Universität Wien,
Karlsplatz 13, 1040 Vienna, Austria; 
\texttt{giorgio.longari@unimib.it}
\addtocounter{footnote}{1}}

\footnotetext{Wolfgang Pauli Institut, Oskar-Morgensternplatz 1, 1090 Vienna, Austria;   Laboratoire Painlev\'e, Lille University, 59650 Villeneuve d'Ascq, France; Technische Universität Wien,
Karlsplatz 13, 1040 Wien. 
\texttt{alice-barbora.tumpach@univ-lille.fr}
\addtocounter{footnote}{1}}

 \footnotetext{The authors are listed in alphabetic order. See the section ``Authors contributions'' for the contributions of each author.}

\date{ }
\maketitle

\begin{abstract}
Most datasets encountered in computer vision and medical applications present symmetries that should be taken into account in classification tasks.  A typical example is the symmetry by rotation and/or scaling  in object detection. A common way to build neural networks that learn the symmetries is to use data augmentation. In order to avoid data augmentation and build more sustainable algorithms, we present an alternative method to mod out symmetries based on the notion of section of a principal fiber bundle. This framework allows to use simple metrics on the space of objects in order to measure dissimilarities between orbits of objects under the symmetry group. Moreover, the section used can be optimized to  maximize separation of classes. We illustrate this methodology on a dataset of contours of objects for the groups of translations, rotations, scalings and reparameterizations. In particular, we present a $2$-parameter family of canonical parameterizations of curves, containing the constant-speed parameterization as a special case, which we believe is interesting in its own right. We hope that this simple application will serve to convey the geometric concepts underlying this method, which have a wide range of possible applications. The code is available at the following link:  \href{https://github.com/GiLonga/Geometric-Learning}{https://github.com/GiLonga/Geometric-Learning}. A tutorial notebook showcasing an application of the code to a specific dataset is available at the following link:\href{https://github.com/ioanaciuclea/geometric-learning-notebook}{https://github.com/ioanaciuclea/geometric-learning-notebook}. 
\end{abstract}


\section{Extended~Abstract}

Our visual system is trained to identify objects that differ only by the action of a shape-preserving group, like the group of translations, rotations, and~scalings. Consequently, these symmetries need to be taken into account in the design of algorithms for object detection and classification.  A~common way to build neural networks that learn the symmetries is to use data augmentation. This involves adding to the dataset new samples obtained by letting the symmetry group act on the original samples, for~example, adding rotated images to the original images. In~addition to the fact that data augmentation increases computational cost, it is also very memory-intensive. In~this paper, we will consider, in particular, the symmetry group consisting of reparameterizations of contours in the plane, which is an infinite-dimensional Lie~group.

In order to avoid data augmentation and build more sustainable algorithms, we present an alternative method to mod out symmetries based on the notion of section (also called cross-section) of  a principal fiber bundle (see \cref{section_fiber_bundle,section_sections}). Within~this framework, a~distinguished object is selected in each orbit under the symmetry group. This amounts to normalization or standardization of samples with respect to the action of the groups of translations, rotations, scalings, and~reparameterizations.

One aim of the present paper is to investigate canonical parameterizations of curves, which allow one to mod out the action of the infinite-dimensional group of diffeomorphisms acting on curves by reparameterizations.  A~canonical parameterization can be understood as an automatic way to re-sample a curve according to some of its geometric features. An~example of a canonical parameterization is provided by the arc-length parameterization, which consists of a unit speed travel along the shape drawn by the curve. In~Section~\ref{Section_clock_parameterization},  we present a new $2$-parameter family of canonical curve parameterizations, called curvature-weighted clock parameterizations, inspired by the small hand trajectory on a traditional clock, which moves at a constant angle every hour. These canonical parameterizations are very natural and may be a good choice in many applications, particularly in the presence of~noise.

When the quotient space by the group action is unique, sections, when they exist, are numerous. In~fact, for~trivial fiber bundles like the fiber bundle of parameterized curves studied in the present paper, the~space of sections is infinite-dimensional. Therefore, the~present approach allows for a lot of flexibility and can be customized for particular applications. It also allows us to use a simple distance function on the total space of the fiber bundle in order to measure dissimilarities between orbits of objects under the symmetry group. Indeed, restricting a simple distance function, such as the $L^2$ distance, to~the range of a chosen section, we obtain a distance function on the quotient space, which is easy to compute. An~example of this construction of distance functions between curves irrespective of their parameterization is given in Section~\ref{section_distance}. They are straightforward to compute, and~do not rely on any energy minimization algorithm. During~training for a classification task, the~section used to design the distance function measuring the dissimilarities between orbits can be optimized to maximize the separation of classes, solving a metric learning problem (see \cref{sec:metric_learning,section_geometric_learning}). Moreover, the~optimal section gives rise to an optimal correspondence between points along any pair of contours in the dataset, solving a registration task. It therefore allows us to interpolate between contours, leading to optimal deformations between shapes (see Figure~\ref{illustration-optimal_parameterization}). Last but not least, our standardization procedure can be integrated into all classification algorithms for contours as a pre-processing step, allowing us to improve classification performance (see Section~\ref{sec:class_results}).

In Section~\ref{results}, we illustrate this methodology with a dataset of leaves. More precisely, we optimize the Dunn index of clustering over a $2$-parameter family of sections corresponding to the curvature-weighted clock parameterizations defined in Section~\ref{Section_clock_parameterization}. In~Section~\ref{Testing}, we show that this solution leads to good classification results for very low computational costs using classical machine learning algorithms. Indeed, with~an optimization over only 2 parameters, our algorithm reaches 0.9602 accuracy (96.02\%  of correct classifications) with SVM for the dataset of Swedish leaves, whereas the state-of-the art model VGG-16 needs 138 million parameters to reach perfect accuracy (100\% correct classifications) on the same dataset (see Section~\ref{Testing}). We also show that taking into account all the shape-preserving groups boosts classification  performance of all the classification algorithms that we considered, with~even an increase of 25.71\% of correct classifications for KNN on the Swedish leaf dataset (Section~\ref{Testing}). Therefore, we argue that our method is a good pre-processing step that should be performed before any more complex feature extraction algorithm on~contours. 

The main contributions of this paper are the~following:

\begin{itemize}
    \item The idea of using sections of principal fiber bundles in order to mod out symmetries is explained in a comprehensive manner and illustrated in the context of plane curves for classical shape-preserving groups (Section~\ref{section_fiber_bundle}).
    \item A $2$-parameter family of canonical contour parameterizations is introduced, called curvature-weighted clock parameterizations (Section~\ref{Section_clock_parameterization}). 
    \item For a labeled dataset of contours, the~separation of classes is optimized based on cluster validity indices such as the Dunn index (Section~\ref{section_geometric_learning}).
    \item We demonstrate and quantify how taking into account symmetries affects clustering and classification results (Section~\ref{Testing}). 
    \item The proposed method not only allows us to measure distances between shapes in a parameterization-invariant manner, but~also provides a registration and optimal deformation between shapes at a very low computational~cost.
\end{itemize}
  
The code is available at the following link: \href{https://github.com/GiLonga/Geometric-Learning}{https://github.com/GiLonga/Geometric-Learning}. A tutorial notebook showcasing an application of the code to some datasets is available at the following link: \href{https://github.com/ioanaciuclea/geometric-learning-notebook}{https://github.com/ioanaciuclea/geometric-learning-notebook}. For the datasets analyzed in this paper, the contours extracted from images are also available at these links.

\begin{figure*}[h!]
    \centering
    \begin{subfigure}{\textwidth}
        \centering
        \includegraphics[width = .7\textwidth]{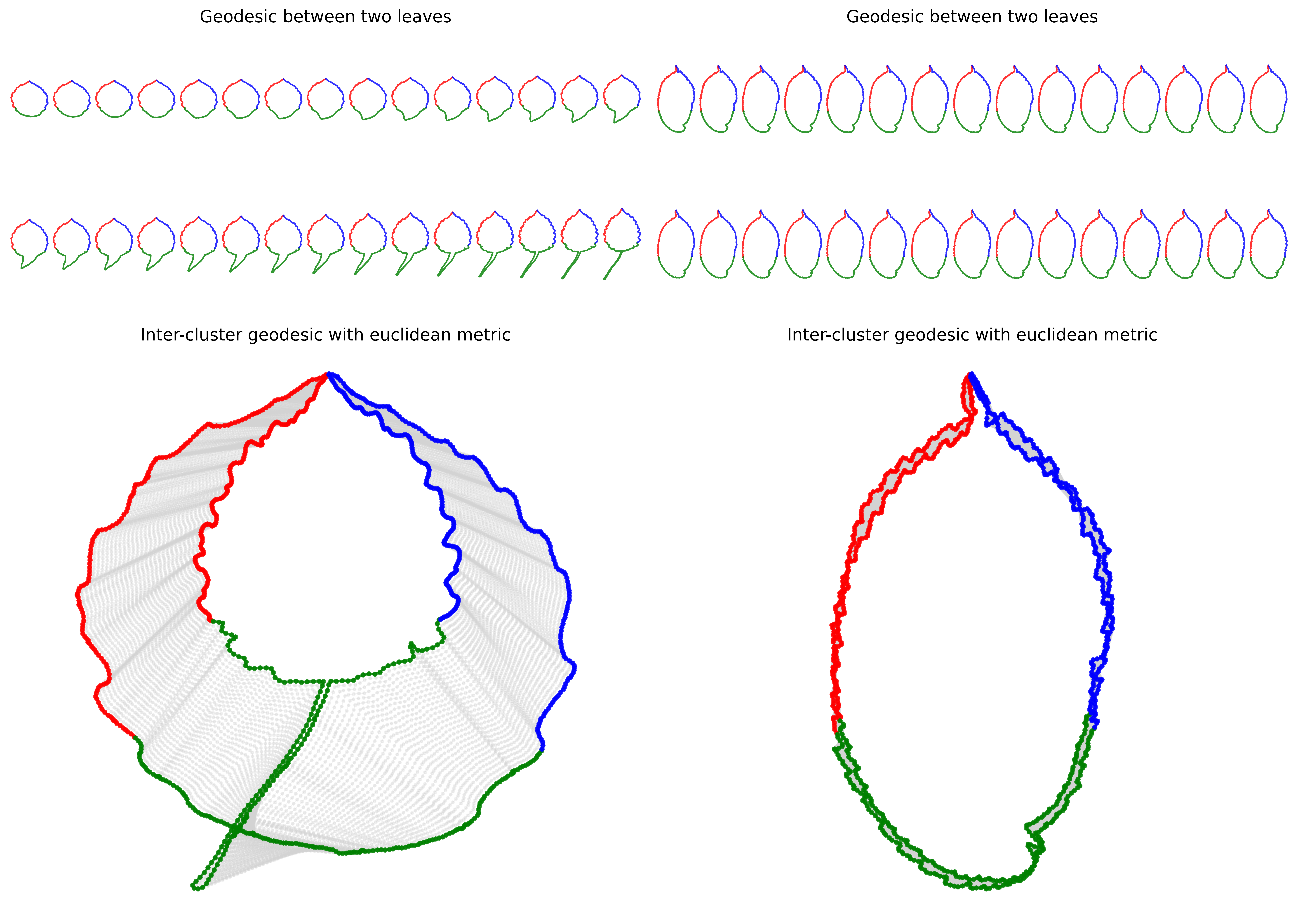}
        \caption{}
    \end{subfigure}\\
    ~ 
    \begin{subfigure}{\textwidth}
        \centering
        \includegraphics[width = .7\textwidth]{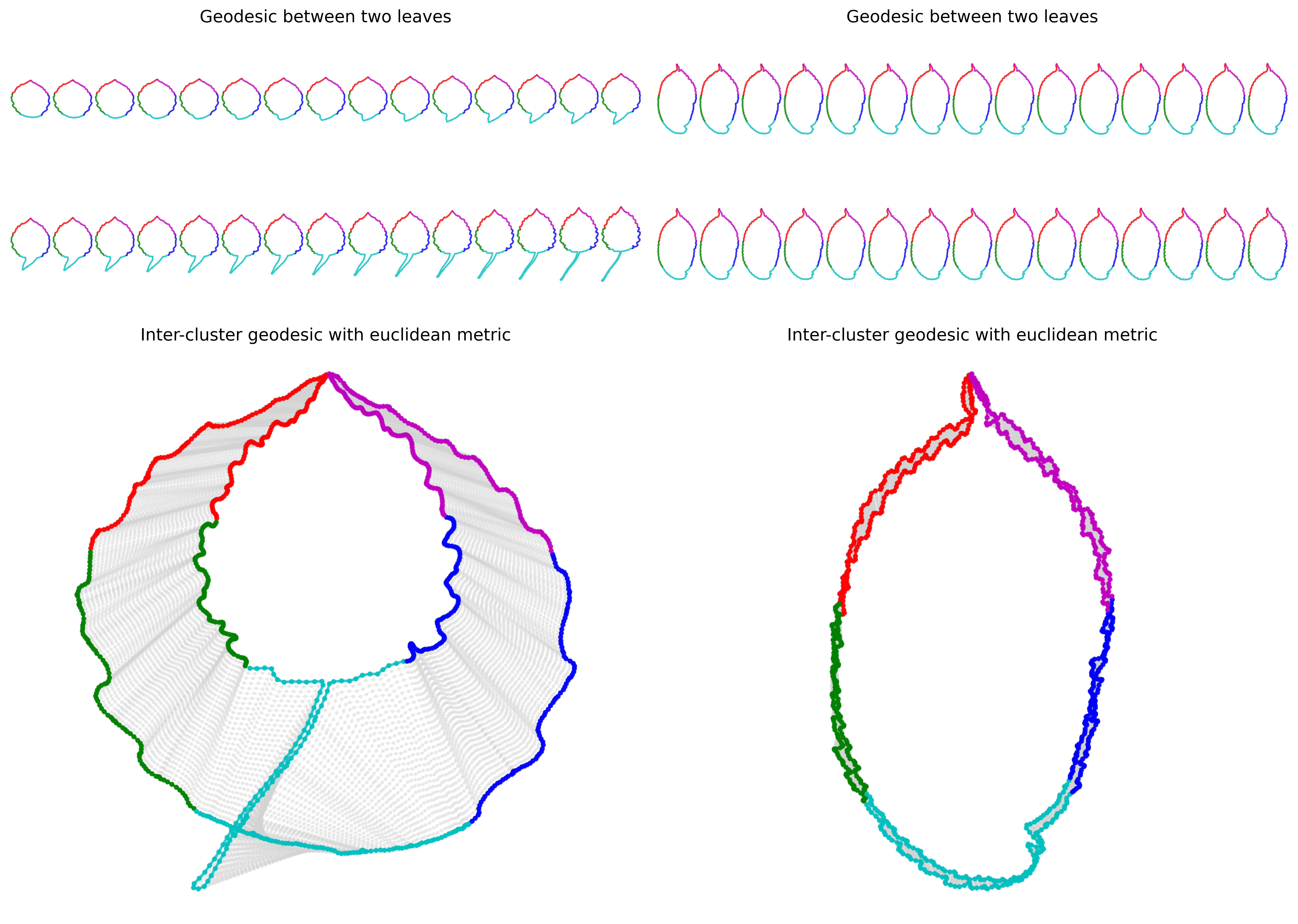}
        \caption{}
    \end{subfigure}
    \caption{(\textbf{a})  Left: The pair of leaves from the Swedish dataset that maximizes the intraclass distance is extracted from the training set, and~the interpolation of their optimal parameterizations for the Dunn index is displayed for the parameters $(n = 3, \lambda = 2000)$. Right: The pair of leaves from the Swedish dataset that minimizes the interclass distance is extracted from the training set, and~the interpolation of their optimal parameterizations for the Dunn index is displayed for $(n = 3, \lambda = 2000)$. 
(\textbf{b}) Left: The pair of leaves from the Swedish dataset that maximizes the intraclass distance is extracted from the training set, and~the interpolation of their optimal parameterizations for the Davies Bouldin index is displayed for the parameters $(n = 5, \lambda = +\infty)$. Right: The pair of leaves from the Swedish dataset that minimizes the interclass distance is extracted from the training set, and~the interpolation of their optimal parameterizations is displayed for $(n = 5, \lambda = +\infty)$. We can see that the same pair of leaves  maximizes the intraclass distance both for the Dunn index and the Davies Bouldin index, and~the same pair of leaves minimizes the interclass distance for both indices.}
    \label{illustration-optimal_parameterization}
\end{figure*}

\section{Mathematical Background and~Method}

\subsection{Parameterized Versus Unparameterized $2D$-Curves}
In this section, we recall the distinction between parameterized and unparameterized $2D$-curves~\cite{Mennucci_book,Overview}. We will be mainly interested in the contours of objects, like the contours of objects depicted in Figure~\ref{Noether}, which mathematically correspond to  Jordan curves in the plane. More precisely, we will consider the following space of smooth embedded closed curves in the plane: 
\begin{equation}\label{P}
\mathcal{P} = \{\gamma\in \mathcal{C}^{\infty}(\mathbb{S}^1, \mathbb{R}^2), \gamma \textrm{ injective}, \gamma'(s) \neq 0, \forall s\in \mathbb{S}^1\}.
\end{equation}
In what follows, the~unit circle $\mathbb{S}^1$ will be identified with $\mathbb{R}/\mathbb{Z}= \{t \in [0,1], 0\sim 1\}$ via the map $\iota:\mathbb{R}/\mathbb{Z}\rightarrow \mathbb{C}$, $[t]\mapsto e^{ 2\pi i t}$. In~particular, this identification distinguishes the point $\iota(0) = (1,0)$ in $\mathbb{S}^1\subset\mathbb{C}$.

\begin{figure}[h!]
\centering
\includegraphics[width=11 cm]{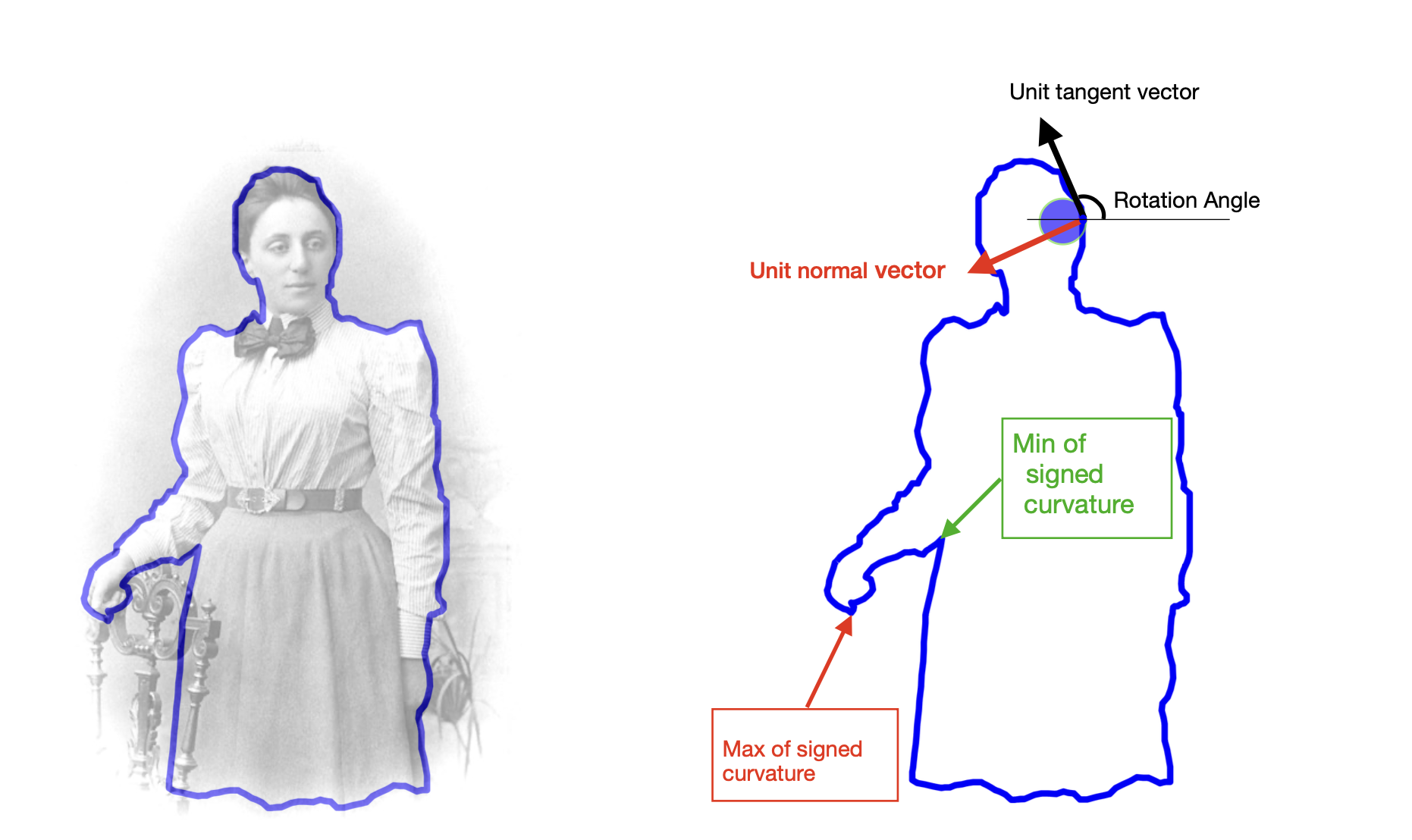}
\caption{\textbf{{Emmy Noether and the moving frame associated with her profile.}} The signed curvature $\kappa$ is defined as the rate of turning angle of the moving frame associated with a parameterized plane curve. The~maximum and the minimum of the signed curvature correspond to two points where the curvature is~extremal.}
\label{Noether}
\end{figure}   

The space $\mathcal{P}$ has a natural structure of smooth Fr\'echet manifold~\cite{Binz}.
Note that the parameterization of a contour with parameter space $\mathbb{S}^1$ is not unique. 
In fact, the~group  $\mathcal{G} = \operatorname{Diff}^+(\mathbb{S}^1)$, consisting of orientation-preserving diffeomorphisms of $\mathbb{S}^1$, is a Fr\'echet Lie group acting smoothly on $\mathcal{P}$ by precomposition:
\[
\begin{array}{lll}
\mathcal{G}\times \mathcal{P}& \rightarrow &\mathcal{P}\\
(\psi, \gamma) & \mapsto & \gamma\circ \psi^{-1}
\end{array}
\]
This action preserves the shapes of curves, and~also the direction of travel along the curves. Moreover, two parameterized curves $\gamma_1$ and $\gamma_2$ in $\mathcal{P}$ corresponding to the same oriented contour in the plane are necessarily related by a diffeomorphism $\psi\in \mathcal{G}$: $\gamma_1 = \gamma_2\circ \psi^{-1}$. 
Given a parameterized curve $\gamma\in\mathcal{P}$, one can consider its equivalence  class  $[\gamma]$ modulo the action of $\mathcal{G}$:
\begin{equation}\label{[gamma]}
[\gamma] = \{ \gamma\circ \psi^{-1}, \psi \in \mathcal{G}\},
\end{equation}
also called the orbit of $\gamma$ for the $\mathcal{G}$-action. 
The equivalence class $[\gamma]$ is uniquely characterized by the range of $\gamma: \mathbb{S}^1\rightarrow \mathbb{R}^2$, which is the shape drawn by $\gamma$ in the plane, also called the unparameterized curve associated with $\gamma$, together with its orientation (the direction of travel).
Consequently, the~shape space of oriented contours in the plane is the quotient space $\mathcal{P}/\mathcal{G}$ of the manifold of smooth embeddings $\mathcal{P}$ modulo the action of the Fr\'echet Lie group $\mathcal{G}$. It was proven in~\cite{Binz} that this quotient space admits a natural structure of smooth manifold and that the canonical projection
\begin{equation}\label{projection}
\begin{array}{cccc}
     \pi: & \mathcal{P} & \longrightarrow & \mathcal{P}/\mathcal{G},\\
     & \gamma & \longmapsto & \pi(\gamma) = [\gamma],  
\end{array}
\end{equation}
onto the quotient space defines a principal fiber bundle in the Fr\'echet category. This result was extended to freely immersed curves in~\cite{Cervera}, with~some missing arguments in the proof, which were fully fixed in~\cite{Mennucci(2021)}. A~visualization of a fiber bundle is given in Figure~\ref{fig1}.
\begin{figure}[h!]
\centering
\includegraphics[width=9 cm]{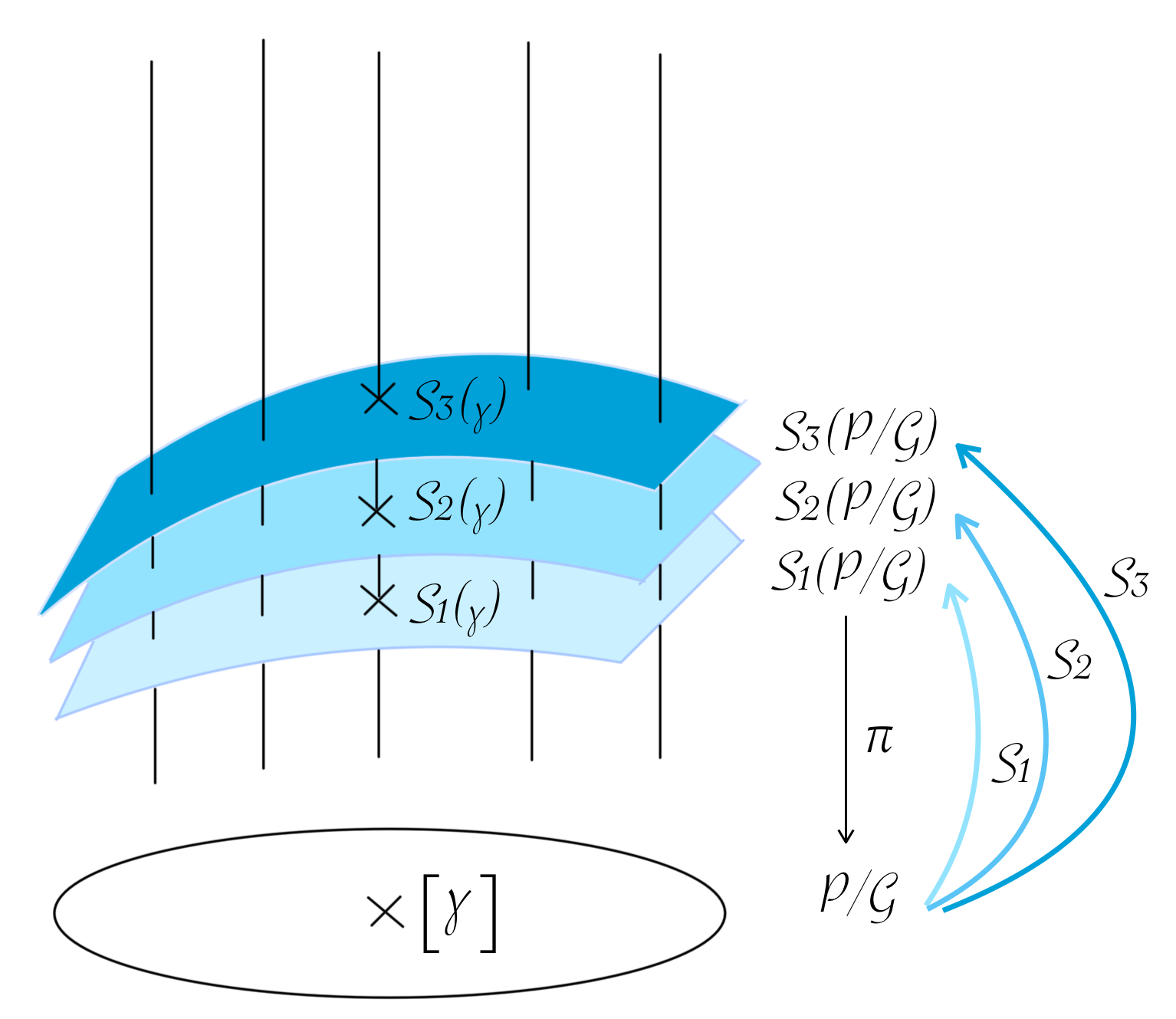}
\caption{Illustration of a fiber bundle $\pi: \mathcal{P}\rightarrow \mathcal{P}/\mathcal{G}$ with three different sections $S_i: \mathcal{P}/\mathcal{G}\rightarrow \mathcal{P}$.\label{fig1}}
\end{figure}   

\subsection{Sections of Fiber~Bundles}\label{section_fiber_bundle}
In the present paper, we will be interested in choosing smoothly a preferred parameterization in each equivalence class $[\gamma]$ defined by \eqref{[gamma]}, where $\gamma$ belongs to (some open subset of) the space of smooth embedded closed curves $\mathcal{P}$. This corresponds to the choice of a smooth section of the principal fiber bundle $\pi:\mathcal{P}\rightarrow \mathcal{P}/\mathcal{G}$ (see Figure~\ref{fig1}). Let us recall the following definition.

\begin{Definition}\label{def_section}
A \emph{(global) smooth section}
  of a fiber bundle $\pi: \mathcal{P} \rightarrow \mathcal{B}$ is a smooth map $s: \mathcal{B} \rightarrow \mathcal{P}$ such that $\pi\circ s = \textrm{Id}_{\mathcal{B}}$. 
\end{Definition}

\begin{Remark}
It can be shown that the range of a smooth section  $s: \mathcal{B}\rightarrow \mathcal{P}$ of a principal fiber bundle $\pi:\mathcal{P}\rightarrow \mathcal{B}$ is a smooth submanifold of $\mathcal{P}$. In~particular, the~manifold consisting of closed curves parameterized by arc length is a smooth manifold~\cite{Preston,  TumPre2}. Using the parametrization with arc length of some particular curves, the~authors of~\cite{MakBor} were able to give the exact analytical solution of the linear static equation of curved Bernoulli–Euler beam. 
\end{Remark}

The notion of section can be applied to different quotient spaces, in~particular to the quotient space  of the space of embedded closed curves modulo shape-preserving groups. 
We will see in Sections~\ref{results} and~\ref{Testing} how the choice of a particular section can influence downstream~analysis. 

\subsection{Canonical Parameterizations of $2D$-Curves as Smooth~Sections}\label{section_sections}

An example of a smooth section for the fiber bundle $\pi:\mathcal{P}\rightarrow \mathcal{P}/\mathcal{G}$ is provided by the submanifold of curves parameterized proportional to arc-length. Let us recall how this particular parameterization is defined. 
Given a smooth parameterized curve in the plane $\gamma\in \mathcal{P}$, its length is defined as
\begin{equation}\label{length}
\operatorname{Length}(\gamma) = \int_0^1 \|\gamma'(t)\| dt,
\end{equation}
where $\|\cdot\|$ denotes the Euclidean norm in $\mathbb{R}^2$.
The length is a geometric invariant of the curve, i.e.,~it does not depend on the parameterization.
Given a starting point, which in our case will be the image of $0\in\mathbb{R}/\mathbb{Z}$, there is a canonical way to reparameterize a curve $\gamma\in\mathcal{P}$ by arc length, producing a unit speed curve. This procedure will change the parameter domain when
the length of the curve is not equal to $1$, and~therefore may not belong to $\mathcal{P}$. However, there is a unique constantspeed reparameterization of $\gamma\in \mathcal{P}$ with parameter domain $\mathbb{R}/\mathbb{Z} = \{t \in [0,1], 0\sim 1\}$, given as follows.

\begin{Proposition}
Given a curve $\gamma\in \mathcal{P}$, consider
the map $\psi$ defined as
\begin{equation}\label{arclength_reparameterization}
\psi(t)  = \frac{1}{\operatorname{Length}(\gamma)}\int_0^t \|\gamma'(s)\|ds,
\end{equation}
where $t \in [0,1]$.
Then, $\psi:\mathbb{R}/\mathbb{Z}\rightarrow\mathbb{R}/\mathbb{Z}$ is an orientation-preserving diffeomorphism, fixing $0\in\mathbb{R}/\mathbb{Z}$. Moreover, the~parameterized curve 
 $p(\gamma) = \gamma\circ \psi^{-1}\in \mathcal{P}$ 
 is the unique constant-speed reparameterization of $\gamma$ with parameter space $\mathbb{R}/\mathbb{Z} = \{t \in [0,1], 0\sim 1\}$, which maps $0\in \mathbb{R}/\mathbb{Z}$ to $\gamma(0)$ and has the same orientation as $\gamma$.  
\end{Proposition}

\begin{Definition}\label{prop_arclength}
    We will denote by $\mathcal{A}$ the subset of $\mathcal{P}$ consisting of constant speed curves with parameter space $\mathbb{R}/\mathbb{Z} = \{t \in [0,1], 0\sim 1\}$. One has
\begin{equation}\label{A}
        \mathcal{A} = \{\gamma\in\mathcal{P}, \|\gamma'(t)\| = \operatorname{Length}(\gamma), \forall t\in\mathbb{R}/\mathbb{Z}\}.
    \end{equation}
\end{Definition}

The space $\mathcal{A}$ of constant-speed parameterized curves with parameter space\linebreak   \mbox{$\mathbb{R}/\mathbb{Z} = \{t \in [0,1], 0\sim 1\}$} \textls[-15]{is just one example of space of canonically parameterized curves. The~possible choices are infinite. 
In the present paper, we will use the following~terminology:}

\begin{Definition}
Let $\mathcal{P}$ be the infinite-dimensional manifold of parameterized closed embedded curves in $\mathbb{R}^2$ defined in \eqref{P}, and~$\mathcal{G} = \operatorname{Diff}^+(\mathbb{S}^1)$ the Fr\'echet Lie group of orientation-preserving reparameterizations.
A \emph{canonical parameterization}  will refer to the choice of a smooth section $s: \mathcal{P}/\mathcal{G}\rightarrow \mathcal{P}$ of the principal fiber bundle  $\pi: \mathcal{P} \rightarrow \mathcal{P}/\mathcal{G}$, which depends only on the geometric features of oriented contours.  It can be understood as an automatic procedure to parameterize curves. It allows us to single out a distinguished parameterization of an oriented contour  $[\gamma]\in\mathcal{P}/\mathcal{G}$ by associating with $\pi(\gamma) = [\gamma]$ the parameterized curve $s([\gamma])\in\mathcal{P}$. It also provides a (non-linear) projection $p: \mathcal{P}\rightarrow s(\mathcal{P}/\mathcal{G})$, i.e.,~satisfying $p^2 = p$, given by
\begin{equation}
    p(\gamma) = s([\gamma]).
\end{equation}
\end{Definition}

\subsection{Examples of Curvature-Weighted Canonical~Parameterizations}\label{sec_curvature_weighted}

In Definition~\ref{prop_arclength}, the~parameterization proportional to arc-length with parameter space $\mathbb{R}/\mathbb{Z} = \{t \in [0,1], 0\sim 1\}$ is defined, and~the corresponding submanifold $\mathcal{A}\subset \mathcal{P}$ is given in \eqref{A}. In~\cite{TumCan}, we have introduced the parameterization proportional to curvature-length,  as~well as a variant called the parameterization proportional to curvarc-length. In~fact, these particular procedures to automatically parameterize curves belong to a one-parameter family of canonical parameterizations, and~we recall their construction below (see Equation~\eqref{ulambda}). This family provides an interpolation between the parameterization proportional to curvature-length ($\lambda = 0$),  the~parameterization proportional to curvarc-length ($\lambda = 1$), and~converges to the parameterization proportional to arc-length when $\lambda\rightarrow +\infty$ \cite{TumCan}. In~order to have a picture in mind (see
 Figure~\ref{fig_one_parameter_family}) where the contour of Emmy Noether is sampled according to five different parameterizations from this~family.
 \begin{figure}[h!]
\centering
\includegraphics[width=13 cm]{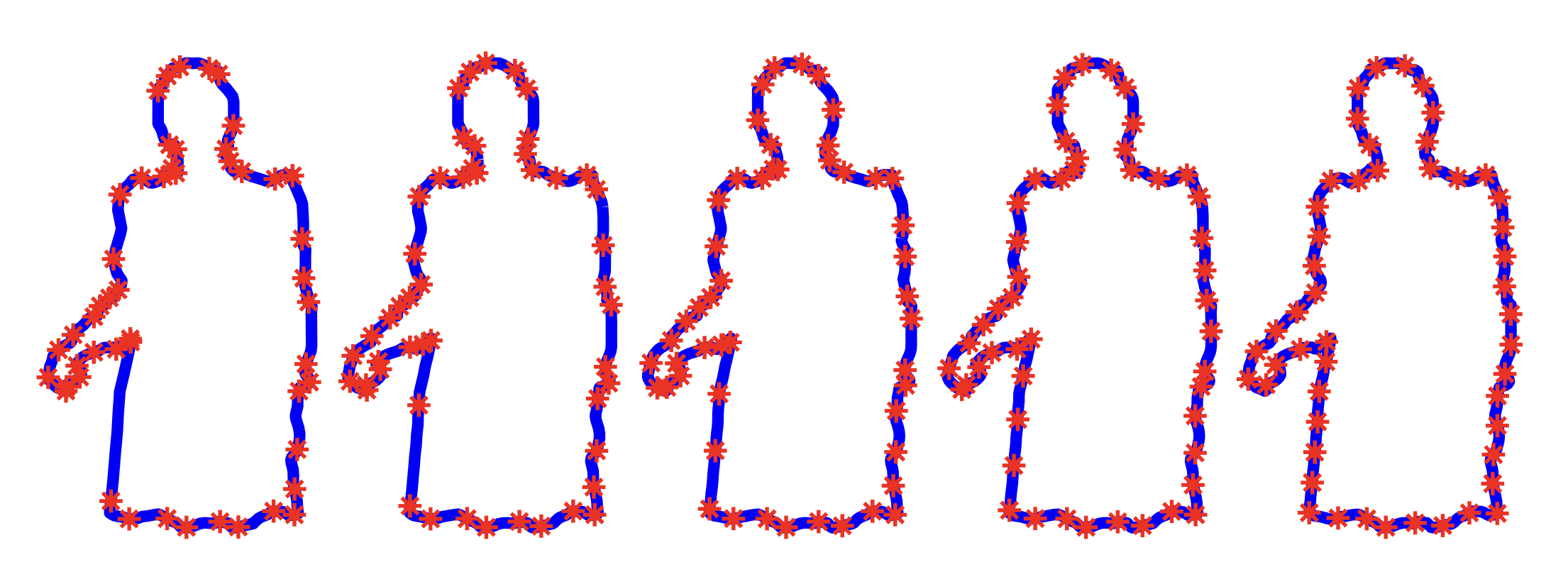}
\caption{\textbf{{A one-parameter family of canonical parameterizations:}} Each contour of Emmy Noether is parameterized in a unique way using Equation~\eqref{ulambda} for a given parameter $\lambda$. The~sample points are the images of a uniform sampling of the interval $[0;1]$. The~leftmost contour is parameterized proportionally to the curvature-length with parameter space $\mathbb{R}/\mathbb{Z} = \{t \in [0,1], 0\sim 1\}$ and corresponds to $\lambda = 0$.  For~this parameterization, sample points are concentrated on high-curvature portions of the curve, whereas flat pieces contain no sample points. The~rightmost 
contour is parameterized proportionally to arc-length with parameter space $\mathbb{R}/\mathbb{Z} = \{t \in [0,1], 0\sim 1\}$ and corresponds to $\lambda = +\infty$. In~this case, sample points are uniformly distributed along the contour. In~between, from~left to right, the~following parameters are used $\lambda = 0.3$, $\lambda = 1$, $\lambda = 2$ (see Equation~\eqref{ulambda}). }
\label{fig_one_parameter_family}
\end{figure}   

Equivalently, this one-parameter family of canonical parameterizations corresponds to a one-parameter family of sections $s_\lambda:\mathcal{P}/\mathcal{G}\rightarrow \mathcal{P}$, where  $s_{+\infty}(\mathcal{P}/\mathcal{G}) = \mathcal{A}$ (see Figure~\ref{fig_one_parameter_family}). These parameterizations are defined using the local differential invariant of curves given by the signed curvature $\kappa$.  The~signed curvature $\kappa$ is the rate of turning angle of the moving frame attached to a parameterized curve. A~visualization of this moving frame is illustrated in Figure~\ref{Noether}. 

More precisely, we introduce a one-parameter family of canonical reparameterizations of curves $\gamma\in\mathcal{P}$ as follows. For~a given $\lambda\in(0, +\infty)$, the~corresponding reparameterization of a curve $\gamma\in\mathcal{P}$ is given by $p_\lambda(\gamma) = \gamma\circ\Phi_\lambda^{-1}$, where $\Phi_\lambda$ depends on $\gamma$ through the following equation involving the signed curvature  $\kappa$ 
of $\gamma$:
\begin{equation}\label{ulambda}
\Phi_\lambda(s) = \frac{\int_0^s \left(\lambda \operatorname{Length}(\gamma) + |\kappa(\gamma(s))|\right) \|\gamma'(s)\| ds}{\int_0^1 \left(\lambda \operatorname{Length}(\gamma) + |\kappa(\gamma(s))|\right) \|\gamma'(s)\| ds}, \quad \lambda >  0.
\end{equation}

Note that the function $s\mapsto \int_0^s \left(\lambda \operatorname{Length}(\gamma) + |\kappa(\gamma(s))|\right) \|\gamma'(s)\| ds$ is strictly increasing when $\lambda>0$, or~when $[\gamma]$ does not contain flat pieces. In~these cases, $\Phi_\lambda$ is an orientation-preserving diffeomorphism of $\mathbb{R}/\mathbb{Z}$ fixing $0\in\mathbb{R}/\mathbb{Z}$.  In~the case $\lambda = 0$ and $\kappa = 0$ on some non-empty interval, the~map $\Phi_0$ defined by
\begin{equation}
\Phi_0(s) = \frac{\int_0^s  |\kappa(\gamma(s))| \|\gamma'(s)\| ds}{\int_0^1  |\kappa(\gamma(s))|\|\gamma'(s)\| ds}, 
\end{equation}
is not injective and its graph presents horizontal portions.
Consequently, $\Phi_0$ is not a diffeomorphism, but~it belongs to the semi-group of generalized
reparametrizations~\cite{Bruveris}. In~other words, $\Phi_0$ is the limit of the  diffeomorphisms $\Phi_\lambda$ when $\lambda \rightarrow 0$, and~$p_0(\gamma)$ can be defined as the limit of $p_\lambda(\gamma)$ in an appropriate~topology.

\begin{Remark}
    In Equation (6)~\cite{TumCan}, another family of curvature-weighted parameterizations was introduced to assign a prescribed anatomical location to sample points on bone contours extracted from X-ray scans. It was used to measure the evolution of Rheumatoid Arthritis in a consistent way.
\end{Remark}

\subsection{Different Ways to Define a Riemannian Metric on Unparameterized~Curves}

In~\cite{TumPre2}, the~authors present three different methods for quantifying dissimilarities in quotient spaces based on Riemannian geometry. These methods consist of defining a Riemannian metric on the quotient space $\mathcal{P}/\mathcal{G}$, which allows us to compute the length of paths in $\mathcal{P}/\mathcal{G}$. The~distance between two points $[\gamma_1]$ and $[\gamma_2]$ in $\mathcal{P}/\mathcal{G}$ (hence between two contours in the plane) is then defined as the infimum of the length of all paths connecting $[\gamma_1]$ to $[\gamma_2]$. We recall, briefly, these three points of~view.

\subsubsection{Quotient~Metric} 
The first method consists of endowing the space $\mathcal{P}$ with a $\mathcal{G}$-invariant Riemannian metric. In~this case, the~Riemannian metric on $\mathcal{P}$ descends to a Riemannian metric on the quotient space, called the quotient metric. A~large body of literature is devoted to this method (see~\cite{Mennucci_book,Sundaramoorthi,Overview} and the references therein). 
For this~method,
\begin{itemize}
    \item[(i)] Computing the distance between two points $[\gamma_1]$ and $[\gamma_2]$ relies on two optimization steps: First, the computation of the minimal path between  $\gamma_1$ and a element in the orbit of $\gamma_2$. Second, the optimization over the infinite-dimensional group of reparameterizations acting on $\gamma_2$.
    \item[(ii)]  The Riemannian metric on $\mathcal{P}$ is, in general, difficult to adjust to applications since the horizontal space may be difficult to compute.
    \item[(iii)]  The added dimensions (infinitely many) that are going from $\mathcal{P}/\mathcal{G}$ to $\mathcal{P}$ are dimensions that are irrelevant for the analysis of data living in the quotient space, but they~need to be taken into account, particularly in the second optimization step.
\end{itemize}
A class of reparameterization-invariant Riemannian metrics on curves, called elastic metrics, was introduced in~\cite{MioSrivastavaJoshi}. It corresponds to a $2$-parameter family of Riemannian metrics $G^{a,b}$  penalizing bending as well as stretching. In~\cite{Srivastava2011b},  it was shown that, for~a certain relation between the parameters, the~resulting metric is flat on parameterized open curves. A~similar method for simplifying the analysis of plane curves was introduced in~\cite{Younes2008}. 
These results have been generalized in~\cite{Bauer et al.}, where the authors introduced another family of metrics, including the  metrics from~\cite{MioSrivastavaJoshi,Younes2008}, which can be described using the restrictions of flat metrics to some cones. The~flattening map has been significantly simplified in~\cite{NeedhamKurtek} and the previous cones interpreted as Regge cones.
In~\cite{Lahiri}, a~precise algorithm for the matching problem of piecewise linear curves is implemented, giving a tool to compare contours in a meaningful way. For~other parameter values, the~$F^{a,b}$ transform introduced in~\cite{NeedhamKurtek} allows us to extend the precise algorithm of~\cite{Lahiri} to arbitrary parameter values $(a, b)$. Approximations of these algorithms using neural networks were implemented in~\cite{Hartman}. We believe that the results obtained do not justify the choice of these computationally intensive designs and are looking for more sustainable~solutions.

\subsubsection{Immersion~Metric} 
The second method consists of identifying the quotient space with the range of a smooth section $s: \mathcal{P}/\mathcal{G}\rightarrow \mathcal{P}$ and endowing the submanifold $s(\mathcal{P}/\mathcal{G})\subset \mathcal{P}$ with a Riemannian metric, such as those induced by a Riemannian metric on $\mathcal{P}$. In~this case, the~Riemannian metric on $\mathcal{P}$ does not need to be $\mathcal{G}$-invariant. For~this~method, 
\begin{itemize}
    \item[(i)] Computing the distance between two points $[\gamma_1]$ and $[\gamma_2]$ relies on one optimization step with constraint: it consists of minimizing the length of  paths constrained to remain in the submanifold $s(\mathcal{P}/\mathcal{G})\subset \mathcal{P}$.
    \item[(ii)]  The dimension of the space is preserved, since the quotient space $\mathcal{P}/\mathcal{G}$ and the range of the section $s$ are diffeomorphic.
    \item[(iii)] The section $s$ can be adapted to applications (we will see some optimization for sections $s$ in the present paper).
\end{itemize}
Let us mention that, since the quotient space $\mathcal{P}/\mathcal{G}$ and the range of any section $s:\mathcal{P}/\mathcal{G}\rightarrow \mathcal{P}$ are diffeomorphic, any quotient metric on $\mathcal{P}/\mathcal{G}$ can be push-forward to the range $s(\mathcal{P}/\mathcal{G})$ of any section $s$. In~\cite{TumPre}, the~authors have transported a particular family of quotient metrics, called elastic metrics, to~the space of arc-length parameterized~curves.

\subsubsection{Gauge-Invariant~Metric} 
The third method was introduced in~\cite{Tpami} (see also~\cite{Notices}) and consists of defining a non-negative  metric on $\mathcal{P}$ (i.e., a non-negative symmetric bilinear form on the tangent bundle $T\mathcal{P}$), called a gauge-invariant metric, whose kernel coincides exactly with the direction of the fibers of the canonical projection $\pi:\mathcal{P}\rightarrow \mathcal{P}/\mathcal{G}$, hence descending to a non-degenerate Riemannian metric on the quotient space. The~idea behind this construction is that the vertical directions of the fiber bundle $\pi:\mathcal{P}\rightarrow \mathcal{P}/\mathcal{G}$ are irrelevant for the analysis of the data in the quotient space $\mathcal{P}/\mathcal{G}$; therefore, they should not interfere in the computation of distances in the quotient space. For~this~method,
\begin{itemize}
    \item[(i)] The dimensions irrelevant to the analysis of the quotient space do not play any role, since they do not contribute to the cost function. 
    \item[(ii)] A reparameterization of curves can be performed on the fly without affecting the minimization algorithms. 
    \item[(iii)] During a path-straightening algorithm for determining a geodesic in the quotient space, the~paths can be lifted to $\mathcal{P}$ and reparameterized with time-dependent reparameterizations without affecting downstream analysis, allowing for more robust algorithms to be designed and improving their convergence. 
\end{itemize}
An example of application of this method to curves for action recognition is given in~\cite{Curves}.

\subsection{The Geodesic Distance Function Associated with a Riemannian~Metric}

Recall that the geodesic distance between two points in a Riemannian manifold is defined as the infimum of the lengths of curves connecting these two points.  For~a finite-dimensional manifold, this distance is non-degenerate and allows one to separate points. In~other words, the~geodesic distance between two points in a finite-dimensional Riemannian manifold is zero if and only if these two points~coincide.

In an infinite-dimensional setting, the~geodesic distance function associated with a Riemannian metric can be degenerate. The~first example of this infinite-dimensional phenomenon was explicitly given in~\cite{MM}. In~this paper, the~authors considered  the reparameterization-invariant $L^2$-Riemannian metric on the space of parameterized $2D$-curves, and~the induced quotient metric on the space of unparameterized $2D$-curves. They proved that the quotient metric admits a vanishing geodesic distance function. In~other words, the~geodesic distance between any pair of curves is~zero. 

Clearly, when the distance function is degenerate, it cannot be used to measure the dissimilarities between pairs of points in the manifold. For~this reason, as well as to avoid computationally costly optimization steps, we propose in this paper another strategy to measure the dissimilarity between contours in the~plane.

\subsection{Proposed Distance Between Oriented~Contours}\label{section_distance}

Recall that $\mathcal{P}$ defined in \eqref{P} is the space of embedded closed $2D$-curves. As~a space of smooth functions on the compact manifold $\mathbb{S}^1$ with values in $\mathbb{R}^2$, it is contained in the Hilbert space of square-integrable functions on $\mathbb{S}^1$ with values in $\mathbb{R}^2$, denoted by $L^2(\mathbb{S}^1, \mathbb{R}^2)$. Recall that the scalar product in $L^2(\mathbb{S}^1, \mathbb{R}^2)$ is given by
\begin{equation}\label{L2}
    \langle f, g\rangle_{L^2} = \int_{\mathbb{S}^1} f(t)\cdot g(t) dt,
\end{equation}
where the dot denotes the scalar product on $\mathbb{R}^2$. The~corresponding norm is given by
\begin{equation}\label{L2norm}
    \| f\|_{L^2} = \left(\int_{\mathbb{S}^1} \|f(t)\|^2 dt\right)^{\frac{1}{2}}.
\end{equation}

Since the scalar product \eqref{L2} is not invariant by the group of reparameterizations $\mathcal{G}$, it cannot be used directly to measure the dissimilarity between oriented contours, since the result would depend on the way the contours are parameterized. However, if~we fix the way contours are parameterized by choosing a canonical parameterization $s:\mathcal{P}/\mathcal{G}\rightarrow \mathcal{P}$, then any oriented contour $[\gamma]$ is associated with a unique function $s([\gamma])$ in $L^2(\mathbb{S}^1, \mathbb{R}^2)$, and~we can measure the distance between $[\gamma_1]$ and $[\gamma_2]$ as
\begin{equation}\label{d2}
d_s([\gamma_1], [\gamma_2]) = \|s([\gamma_1]) - s([\gamma_2])\|_{L^2}.
\end{equation}

In other words, the~$L^2$-distance is restricted to the subset $s(\mathcal{P}/\mathcal{G})$, which is in one-to-one correspondence with the quotient space $\mathcal{P}/\mathcal{G}$ consisting of oriented contours.
The distance on the space of contours $\mathcal{P}/\mathcal{G}$ given by \eqref{d2} is non-degenerate:
\begin{Proposition}\label{distance_non_zero}
For any section $s:\mathcal{P}/\mathcal{G}\rightarrow \mathcal{P}$, and~any oriented contours $[\gamma_1]$ and $[\gamma_2]$ in $\mathcal{P}/\mathcal{G}$, one has
\begin{equation}\label{ds}
    d_s([\gamma_1], [\gamma_2]) = 0 \Leftrightarrow [\gamma_1] = [\gamma_2].
\end{equation}
\end{Proposition}

\begin{proof}[Proof of Proposition~\ref{distance_non_zero}]
    \textls[-12]{Suppose that $d_s([\gamma_1], [\gamma_2]) = 0$. By~definition \eqref{d2},   {$\|s([\gamma_1]) - s([\gamma_2])\|_{L^2}$$ = 0$}.} Since $L^2(\mathbb{S}^1, \mathbb{R}^2)$ is a Hilbert space, this implies that $s([\gamma_1]) = s([\gamma_2])$ as elements in $L^2(\mathbb{S}^1, \mathbb{R}^2)$, and is thus almost everywhere. Since both $s([\gamma_1])$ and $s([\gamma_2])$ are smooth functions, one has $s([\gamma_1])(t) = s([\gamma_2])(t)$ for any $t\in \mathbb{R}/\mathbb{Z}$.  Consequently, 
    $\pi(s([\gamma_1])) = \pi(s([\gamma_2]))$. But~by the Definition~\ref{def_section} of a section, $\pi\circ s = \textrm{Id}_{\mathcal{P}/\mathcal{G}}$. Hence, $[\gamma_1] = [\gamma_2]$.
    The other implication is trivial.
\end{proof}

\begin{Proposition}\label{triangular_inequality}
    For any smooth section $s: \mathcal{P}/\mathcal{G}\rightarrow \mathcal{P}$,  $d_s:\mathcal{P}/\mathcal{G}\times \mathcal{P}/\mathcal{G}\rightarrow [0, +\infty)$ defined by \eqref{d2} satisfies the triangular inequality. 
\end{Proposition}

\begin{proof}
    Consider any smooth section $s: \mathcal{P}/\mathcal{G}\rightarrow \mathcal{P}$, as well as the contours $[\gamma_1], [\gamma_2]$ and $[\gamma_3]$ in $\mathcal{P}/\mathcal{G}$.
    One has
    \[
    \begin{array}{ll}
d_s([\gamma_1], [\gamma_3]) & = \|s([\gamma_1]) - s([\gamma_3])\|_{L^2} \leq \|s([\gamma_1]) - s([\gamma_2])\|_{L^2} + \|s([\gamma_2]) - s([\gamma_3])\|_{L^2} \\ & \leq d_s([\gamma_1], [\gamma_2])+ d_s([\gamma_2], [\gamma_3]),
\end{array}
    \]
where we used the triangle inequality in $L^2(\mathbb{S}^1, \mathbb{R}^2)$.
\end{proof}
\begin{Remark}
    It follows from Proposition~\ref{distance_non_zero} and \ref{triangular_inequality} that $d_s$ is indeed a distance function on the quotient space $\mathcal{P}/\mathcal{G}$, i.e.,~it is non-negative, symmetric, non-degenerate, and satisfies the triangle inequality.
\end{Remark}

\begin{Remark}
Propositions~\ref{distance_non_zero} and \ref{triangular_inequality} can be generalized to any norm on the space of functions from $\mathbb{S}^1$ to $\mathbb{R}^2$. The~$L^2$-norm was chosen since it is well suited for the datasets we are considering in Sections~\ref{results} and \ref{Testing}. For~instance, a~leaf with peduncle and a leaf without peduncle belonging to the same class of leaves, see Section~\ref{section_scaling}, are close in distance when induced by the $L^2$-norm but distant if we use the $L_\infty$-norm instead.
\end{Remark}

\subsection{Metric~Learning}\label{sec:metric_learning}
Metric learning is a branch of Geometric Learning devoted to learning a distance function from a dataset. It emerged from the observation that the Euclidean distance of the ambient space in which the dataset is encoded may not be the best choice for measuring distances. Application-driven metric learning aims to design a distance function that measures similarities between sample points in a pertinent way for the application at~hand. 

In the present paper, we propose a metric learning algorithm based on an optimization over the section $s$. 
The distance defined in \eqref{d2} clearly depends on the choice of section $s:\mathcal{P}/\mathcal{G}\rightarrow \mathcal{P}$. Given a contour classification task, we can optimize the section $s$ to obtain the best separation between classes on the training set. The~quality of a clustering in a metric space can be measured using different validation indices (see Section~\ref{validation_indices}), such as the Dunn index (Equation~\eqref{Dunn}). In~Section~\ref{Section_clock_parameterization}, we present a $2$-parameter family of sections $s_{\lambda, n}$ that is used to define distance functions on contours using Equation~\eqref{d2}. The~optimization of the corresponding cluster validation indices is performed in Section~\ref{section_geometric_learning} for the leaf dataset.
The improvement of the classification performance is analyzed in Section~\ref{Testing}.

In the finite-dimensional context, the~field of supervised PAC (Probably Approximately Correct) learning provides theoretical guaranties that explain why and when supervised learning algorithms work. 
For PAC-Bayes guaranties to learning settings with non-compact, finite-dimensional symmetries, we refer the reader to the recent paper~\cite{Beck}. As~far as we know, such a Bayesian approach has not been investigated for infinite-dimensional groups of~symmetries.

\subsection{Validation Indices of a~Clustering}\label{validation_indices}

In order to quantify how the choice of different sections influences the distances between samples, we use two cluster validation indices, the~Dunn index (\cref{subsection_Dunn}) and the Davies Bouldin index (\cref{Davies Bouldin}). A~comparison of these two indices is made in Section~\ref{Testing}. For~a discussion and comparison of more general cluster validation techniques, we refer the reader to~\cite{Bezdek,BolAzu}. 

\subsubsection{Dunn Index of a~Clustering}\label{subsection_Dunn}
In order to measure clustering efficiency in the algorithms we describe below, we use the Dunn index \cite{Dunn}, which measures the ratio between the minimal interclass distance to the maximal intraclass distance. A~high Dunn index characterizes dense and well-separated clusters, with~a small variance between members of a cluster and different clusters sufficiently far apart, as~compared to the within-cluster~variance.

The Dunn index is computed as follows. For~each class $C_k$, $1\leq k\leq K$ ($K$ being the number of classes), we compute the centroid $c_k$ of class $C_k$  as the mean of this class. 
In practice, the~average of the positions of the points along the contours 
gives the average shape.
The distance between classes $D_{\textrm{inter}}(k_1,k_2)$ is calculated as the distance between the centroid $c_{k_1}$  of class $C_{k_1}$ and the centroid and $c_{k_2}$ of class $C_{k_2}$
\begin{equation}
    D_{\textrm{inter}}(k_1, k_2) = d_{s}(c_{k_1}, c_{k_2}),
\end{equation}
where $d_s$ is defined in Equation~\eqref{d2} for a given section $s: \mathcal{P}/\mathcal{G}\rightarrow \mathcal{P}$ of the fiber bundle of parameterized contours (we will start with the section of arc-length parameterized contours, and~optimize over a two-parameter family of sections in Section~\ref{section_geometric_learning}). 
The distance between classes $D_{\textrm{intra}}(k)$ is measured as the maximum distance between any pair of elements in class $C_k$:
\begin{equation}
    D_{\textrm{intra}}(k) = \max_{i, j\in C_k} d_s(i,j).
\end{equation}
The Dunn index is defined as follows, with $K$ being the number of classes:
\begin{equation}\label{Dunn}
    \operatorname{Dunn}_{\lambda, n} = \frac{\min_{1\leq k_1<k_2\leq K}D_{\textrm{inter}}(k_1,k_2)}{\max_{1\leq k\leq K} D_{\textrm{intra}}(k)}. 
\end{equation}

\subsubsection{Davies Bouldin Index of a~Clustering}\label{Davies Bouldin}

An alternative measure of clustering efficiency is the \textit{{Davies Bouldin index}}~\cite{Davies_Bouldin} that measures the maximal ratio between the spread of two classes and the distance between their centroids. The~Davies Bouldin index varies between $0$ and $+\infty$, where a low index corresponds to a better classification. 
As with the Dunn index, for~each class $C_k$, $1\leq k\leq K$ ($K$ being the number of classes), the~centroid $c_k$ of class $C_k$ is computed  as the mean of this class. Then, the mean distance $\bar{\delta_k}$ of the elements of the class $C_k$ to their centroid $c_k$ is computed as
\begin{equation}
\bar{\delta_k} = \frac{1}{|C_k|} \sum_{i\in C_k}d_s(i, c_k).
\end{equation}
Finally, the~Davies Bouldin index  $DB_{\lambda, n}$ is defined as follows, with $K$ being the number of classes:
\begin{equation}\label{DB_index}
DB_{\lambda, n} = \frac{1}{K}\sum_{k = 1}^K\max_{k'\neq k}\left(\frac{\bar{\delta_k'}+ \bar{\delta}_k}{d_s(c_k', c_k)}\right).
\end{equation}
Due to the averaging of the distances to a centroid over all elements of a class, the~Davies Bouldin index is more stable than the Dunn index in the presence of outliers (see Section~\ref{sec:Flavia}). On~the other hand, the~Dunn index can help detect outliers. By~extracting from the dataset the pairs of samples from the same class maximizing the intraclass distance and the pair of samples from different classes that minimizes the interclass distance (see Section~\ref{section_geometric_learning}), one can spot some~inconsistencies.

\section{Illustration of the~Methodology}\label{results}

\subsection{Database and Pre-Processing~Steps}\label{section_normalization}

\subsubsection{Database} \label{section_dataset}

We used the Swedish leaves dataset from the Link\"opling University, which can be freely downloaded from \href{https://www.cvl.isy.liu.se/en/research/datasets/swedish-leaf/}{https://www.cvl.isy.liu.se/en/research/datasets/swedish-leaf/} (accessed on 8 September 2025).
 This dataset consists of pictures of leaves organized into $15$ classes, with each class containing 75~leaves of the same variety. An~element of each class is illustrated in Figure~\ref{fig0}a, and~the names of the corresponding varieties are listed in Figure~\ref{fig0}b. In~a preliminary step, we extract the contours of the leaves by transforming the pictures into black and white imprints, and~then extract the boundaries of the resulting shapes with an appropriate algorithm (e.g., \texttt{bwboundaries} in Matlab). The~resulting contours are illustrated in Figure~\ref{fig0}c, and~consist of an ordered set of points along the boundary of the leaves. This ordering gives us an initial parameterization $\gamma$ of each~contour.
 
\begin{figure}[h!]
\centering
\subfloat[\centering]{\includegraphics[width = .3\textwidth]{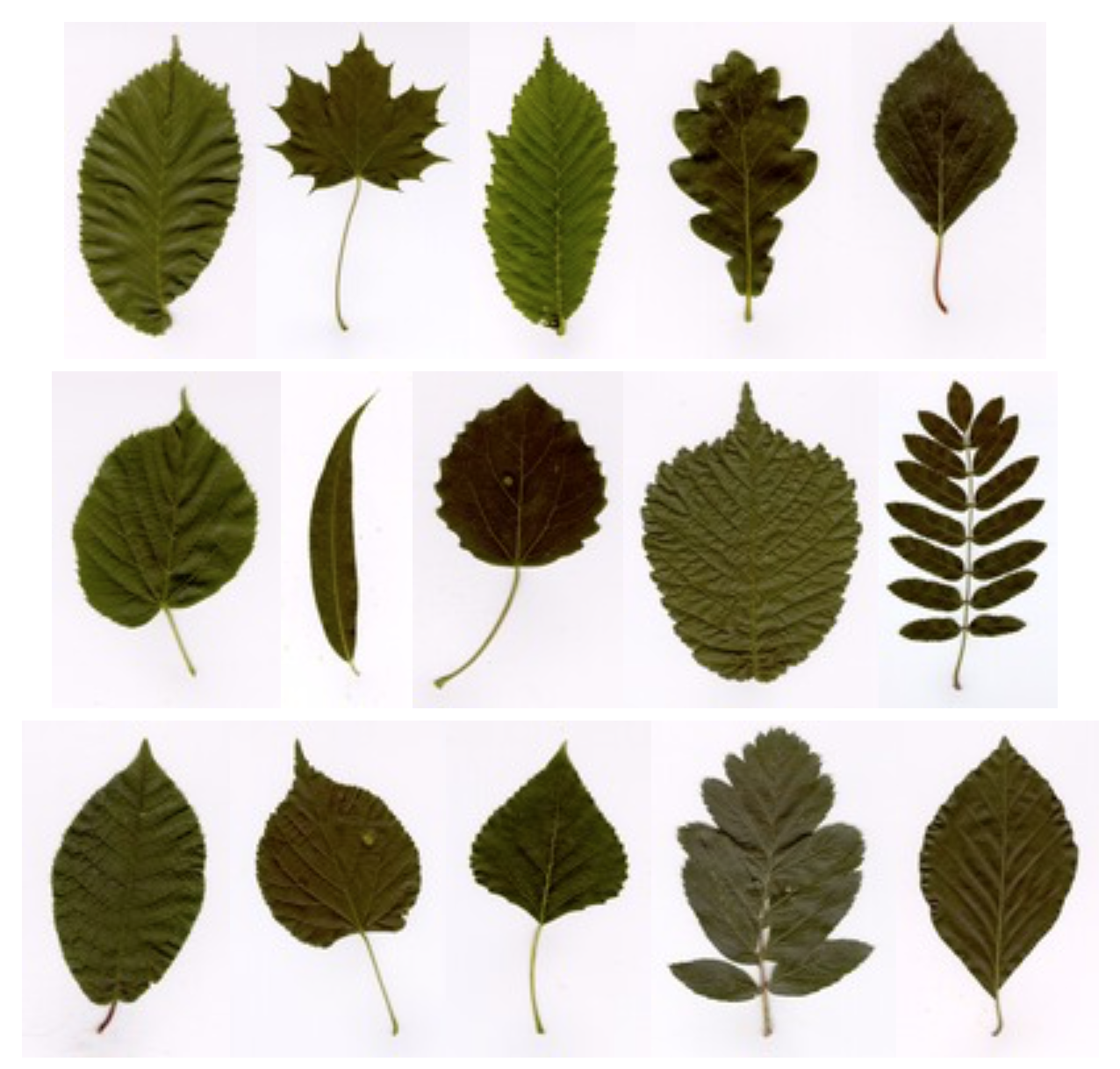}}
\subfloat[\centering]{\includegraphics[width = .3\textwidth]{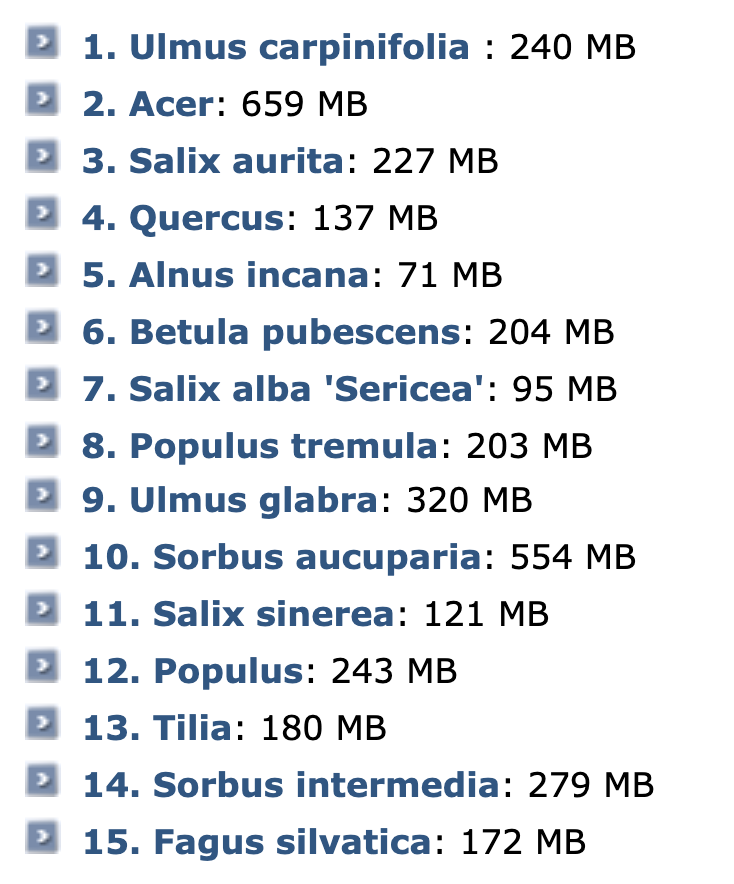}}\\
\subfloat[\centering]{\includegraphics[height = 5cm]{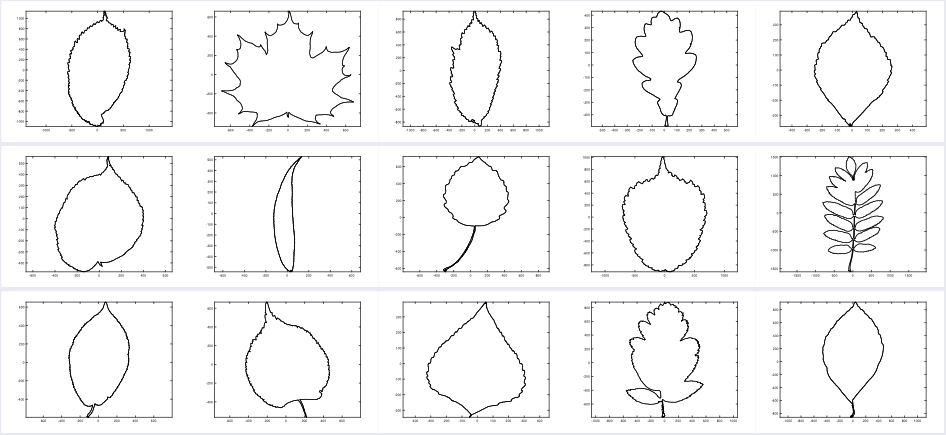}}

\caption{\textbf{Dataset of Swedish leaves from the Link\"opling University dataset} \href{https://www.cvl.isy.liu.se/en/research/datasets/swedish-leaf/}{https://www.cvl.isy.liu.se/en/research/datasets/swedish-leaf/} (accessed on 8 September 2025). (\textbf{a}) A sample image from each class of leaves is depicted (the classes are ordered from left to right and top to bottom) (\textbf{b}) corresponding classes  (\textbf{c}) extracted contours using Matlab's function \texttt{bwboundaries} on binarized images.}
\label{fig0}
\end{figure} 

We divide the resulting set of contours into a training set, containing $50$ contours from each class, as well as~a testing set containing the remaining contours. \textbf{{In particular, the~training set and the testing set are disjointed.}}

\subsubsection{Standardizing the Direction of~Travel}\label{section_counterclockwise}

The initial parameterizations of the contours obtained from the boundary extraction algorithm explained in Section~\ref{section_dataset} induce an orientation, leading to contours following clockwise or counterclockwise. As~a first normalization step, we check if the contours are traveling counterclockwise, and~flip the parameterization of those contours following clockwise. 
In order to automatically detect the orientation of a given contour, we compute the signed area enclosed by the contour.  A~positive signed area corresponds to a contour that traveled counterclockwise, and~a negative area corresponds to a contour that traveled clockwise.
The signed area can be computed using Stokes' Theorem by integration along the contour of a leaf:
\begin{equation} \label{equation_area}
    \operatorname{Area}(\gamma) = \int_\gamma x dy
\end{equation}
In practice, for~the dataset of Swedisch leaves, we did not encounter any contour following clockwise. 
The Dunn index defined by Equation~\eqref{Dunn}, calculated on the training set containing $50$~leaves of each of the $K = 15$ classes, is equal to $0.0286$ when all contours have traveled counterclockwise. It decreases to $0.0127$ when half of the contours chosen randomly have traveled clockwise, and~the other half are traveled counterclockwise. To~have a visual representation of the distance distribution of the leaves according to the distance function given by~\eqref{d2} with respect to the section $s$ consisting of arc-length parameterized contours, we use the  \texttt{tsne} algorithm. The~resulting distribution of leaves in $2$ dimensions with random direction of travel, as well as~for contours that traveled counterclockwise, is given in Figure~\ref{fig_orientation}. 

\begin{figure}[h!]
\centering
\hspace{2cm}
\subfloat[\centering]{\includegraphics[height=7cm,trim= 2cm 2cm 4cm 3cm, clip]{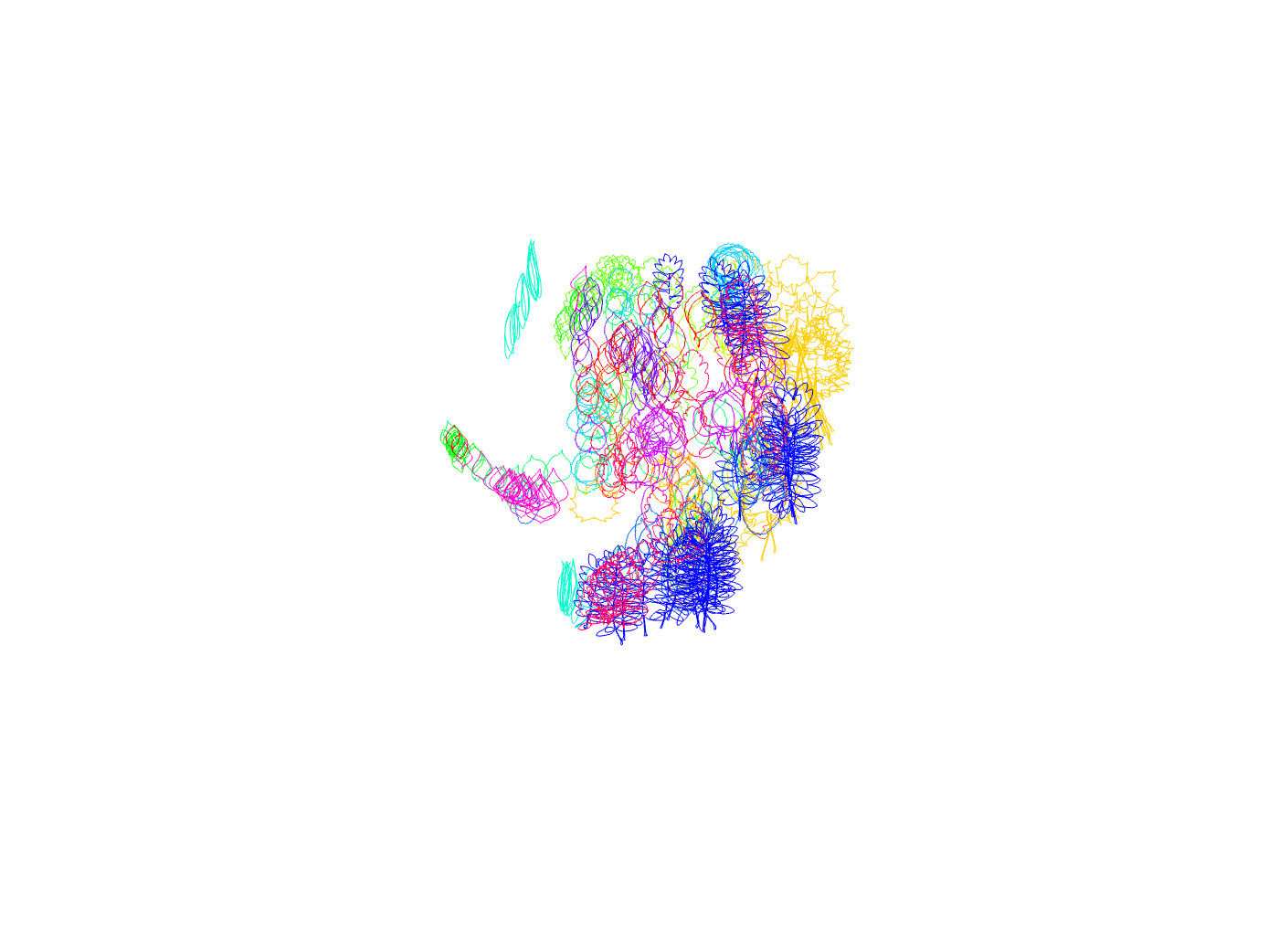}}
\subfloat[\centering]{\includegraphics[height=6cm,trim= 4cm 1cm 2cm 1.7cm,clip]{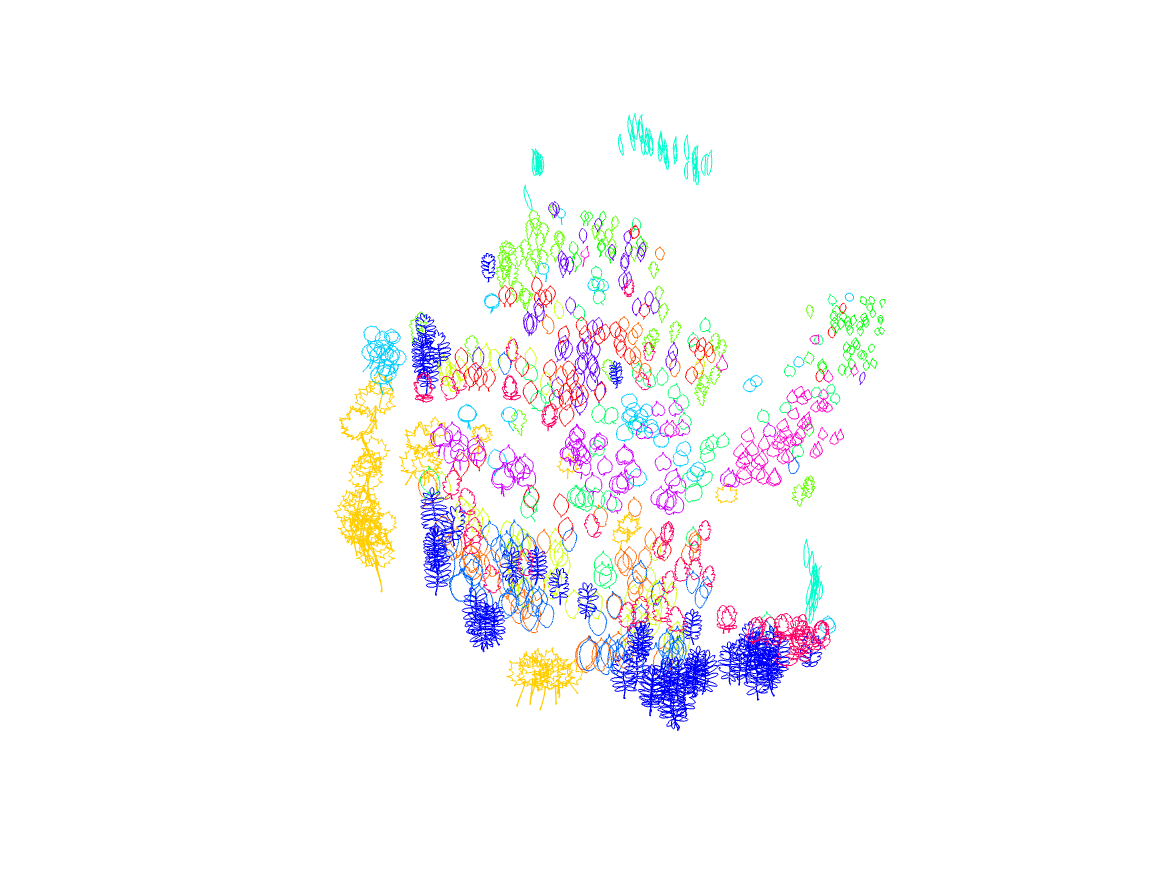}}\\
\caption{\textbf{{Normalization of the orientation variability}}.  Two-dimensional representation of the distance distribution along the dataset using \texttt{tsne} algorithm (\textbf{a}). Before normalization of orientation (half of the contours are traveling clockwise, the~other half counterclockwise), the~Dunn index equals $0.0127$. (\textbf{b}) After orientation normalization (all the contours are traveled counterclockwise), the~Dunn index increases to $0.0286$.} \label{fig_orientation}
\end{figure}

\subsubsection{Standardizing the Starting Point of~Parameterizations}\label{section_rot_parameter}

Since the contour of a leaf is represented by an ordered set consisting of finitely many sample points along the contour, the~starting point of this discretization induces variability that we need to take into account. In~the continuous case, this amounts to standardizing the position of the starting point of the parameterization of contours. This corresponds to the normalization with respect to rotation in parameter space $\mathbb{S}^1$, i.e.,~with respect to the subgroup of rotations $\operatorname{Rot}(\mathbb{S}^1)\subset \operatorname{Diff}^+(\mathbb{S}^1)$, where $\operatorname{Diff}^+(\mathbb{S}^1)$ is the group of orientation-preserving~reparameterizations.

For the dataset of leaves at hand, we detect automatically the point of each contour with the largest vertical component  (which was unique for all contours) and reorder the sample points in such a way that this particular point becomes the starting point. 
In Figure~\ref{fig_starting_point}, we illustrate the distance distribution using the \texttt{tsne} algorithm before and after normalization of the starting points. The~starting points are showcased as black dots along the contours. The~Dunn index increases from $0.0286$ to $0.0328$ after this normalization~step.

\begin{figure}[h!]
\centering
\subfloat[\centering]{\includegraphics[width=8cm,trim= 5cm 1cm 2cm 1cm, clip]{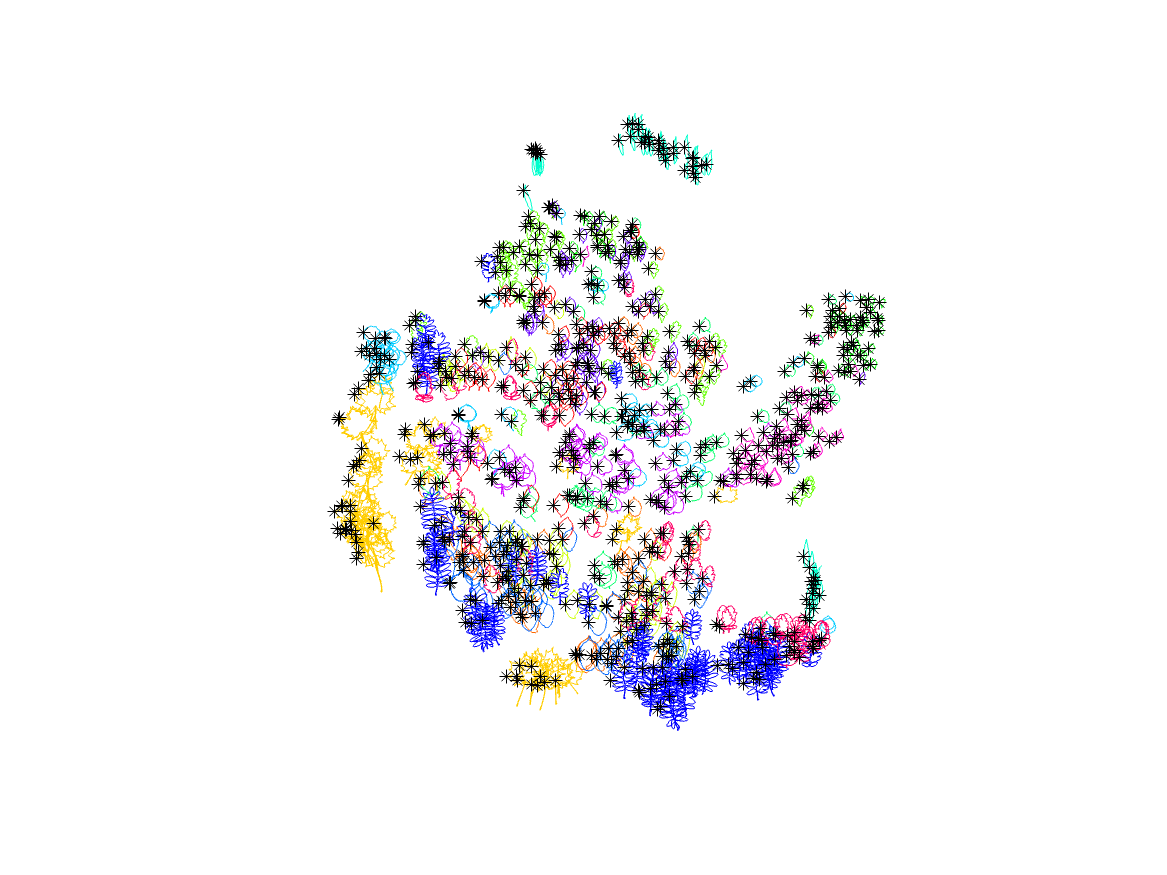}}
\subfloat[\centering]{\includegraphics[width=7.5cm,trim= 2cm 1cm 1cm 4cm, clip]{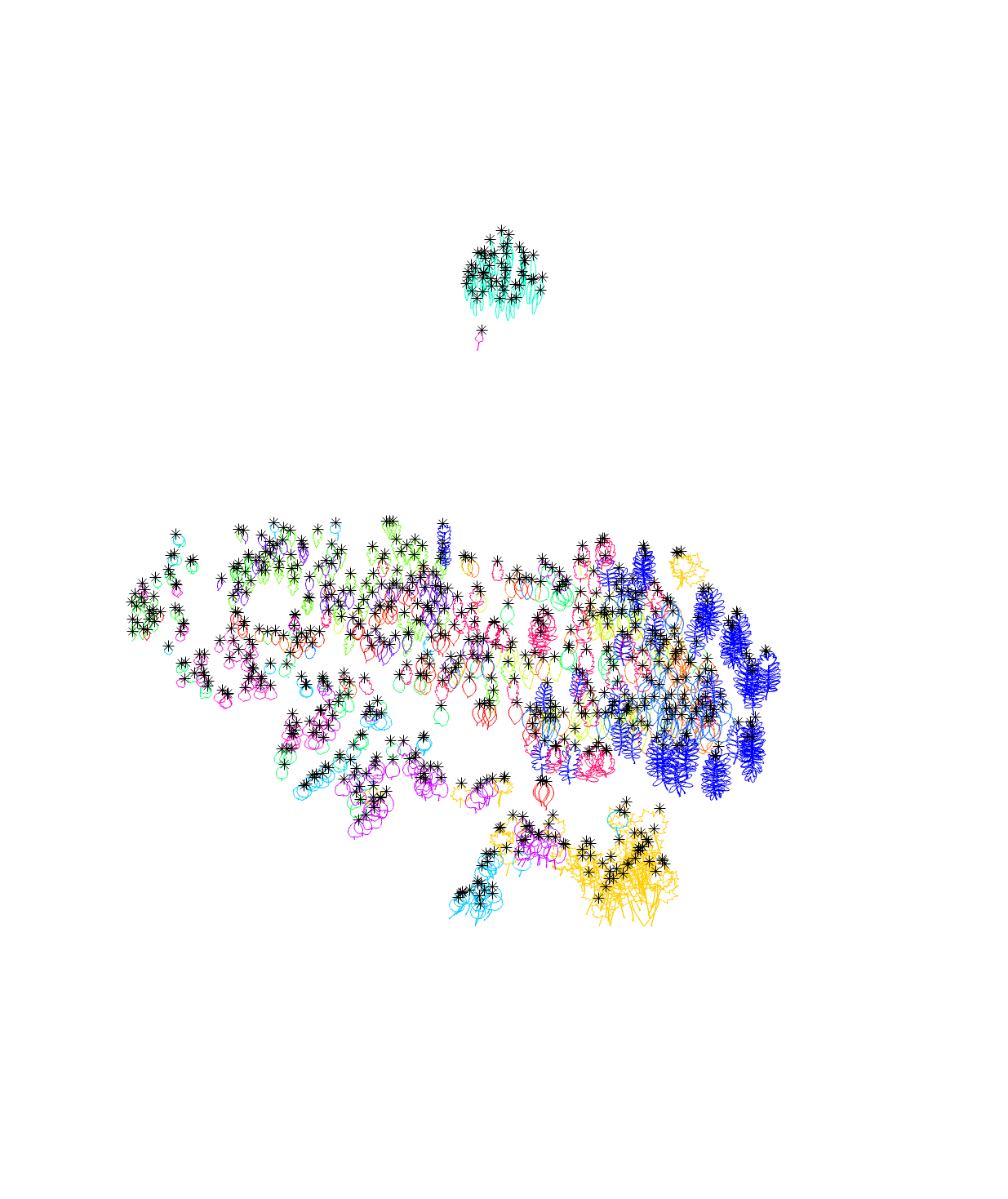}}\\
\caption{\textbf{{Normalization of the starting point variability}}.  Two-dimensional representation of the distance distribution along the dataset using \texttt{tsne} algorithm. (\textbf{a}) Before normalization of the starting points, the Dunn index equals $0.0286$. (\textbf{b}) After starting point normalization, the~Dunn index increases to $0.0328$. The~starting points are depicted as black~dots.} \label{fig_starting_point}
\end{figure}

\subsubsection{Standardizing the Scale~Variability}\label{section_scaling}
The dataset contains leaves of different sizes, as~can be seen in Figure~\ref{fig2}a on 7 samples of Acer leaves. In~order to recognize the class of a leaf irrespective of its size, we need to eliminate the variability of the scale. We tested two normalization procedures:

\begin{itemize}
\item[(a)] \textbf{{Normalization of the length of contours:}} In this normalization method, we first compute the contour length of each leaf and then divide the initial parameterization by this length. 
\item[(b)] \textbf{{Normalization of the enclosed area:}} Each contour is a Jordan curve in the plane and encloses a domain in the plane that corresponds to the surface of the corresponding leaf. In~this normalization step, we compute the area of each leaf using Equation~\eqref{equation_area} and renormalize the initial parameterization to have a unit area by dividing the parameterization by the square-root of (the absolute value of) the area.
\end{itemize}

As can be seen in Figure~\ref{fig2}b,c, normalization to unit-length induces greater intraclass variability compared to normalization to unit-enclosed area.
This is mainly due to the fact that, in the same class,  leaves with peduncles as well as  leaves without peduncles are present. Normalization by unit-length is heavily affected by the presence or absence of a peduncle. In~contrast, the~normalization to curves with unit-enclosed area is not affected by the presence or absence of peduncles, as~peduncles barely contribute to the~area. 

\begin{figure}[h!]
\centering
\subfloat[\centering]{\includegraphics[width=11cm]{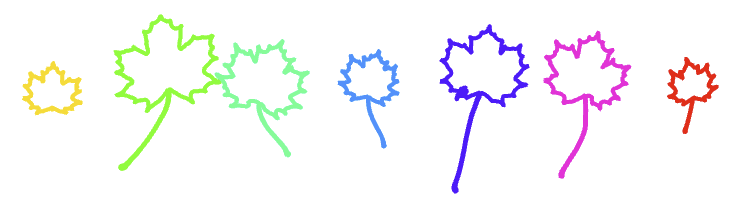}}\\
\subfloat[\centering]{\includegraphics[width=11cm]{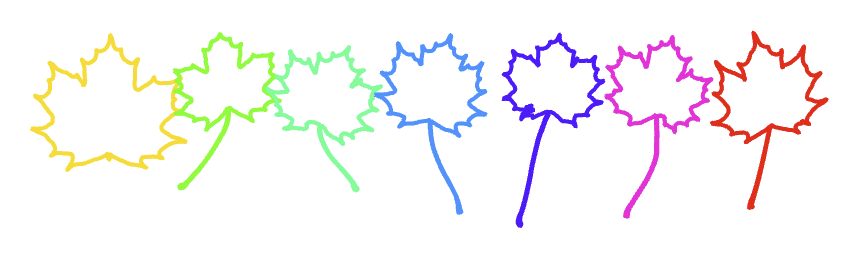}}\\
\subfloat[\centering]{\includegraphics[width=11cm]{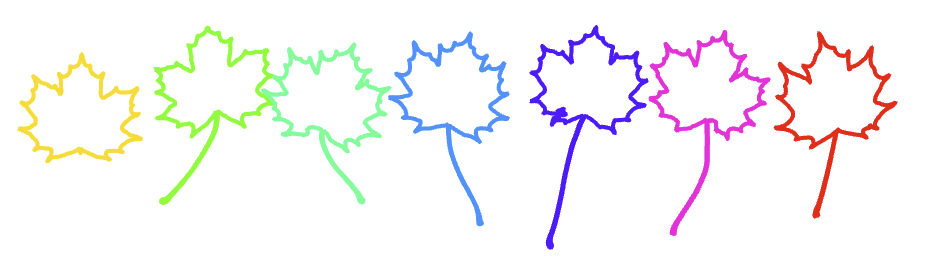}}

\caption{\textbf{{Normalization of the scale variability}}. Seven Acer leaves from the Swedish leaves dataset are used to illustrate two different normalizations of scaling.   (\textbf{a}) Initial contours.    (\textbf{b}) Each contour is rescaled in such a way that the length of the contour is equal to one. This scaling method has the effect of enlarging significantly the first leaf without peduncle. (\textbf{c})  Each contour is rescaled in such a way that the area enclosed by the contour is equal to one. For~this scaling method, the~leaves appear with the similar proportions.}
\label{fig2}
\end{figure}

Despite this fact, the~Dunn index increases to $0.0587$ after normalization by unit-length, and~only to $0.0381$ after normalization by unit area. This is due to the fact that the interclass distance increases more when normalization by the length is used, due to the characteristic boundary shape of different varieties of leaves. This can be seen in Figure~\ref{fig_scale}. In~the sequel, we therefore select the normalization by~unit-length.

\begin{figure}[h!]
\centering
\hspace{2cm}
\subfloat[\centering]{\includegraphics[width=6cm]{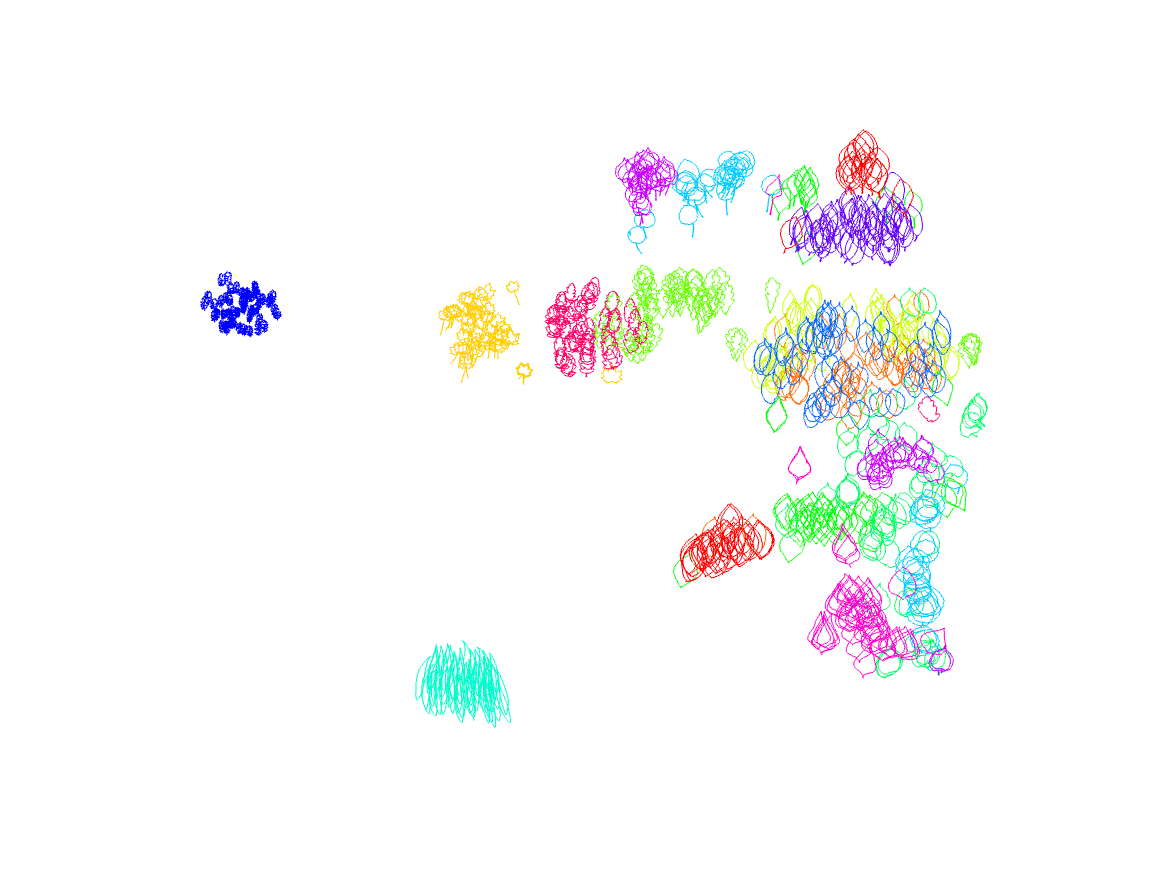}}
\subfloat[\centering]{\includegraphics[width=6cm]{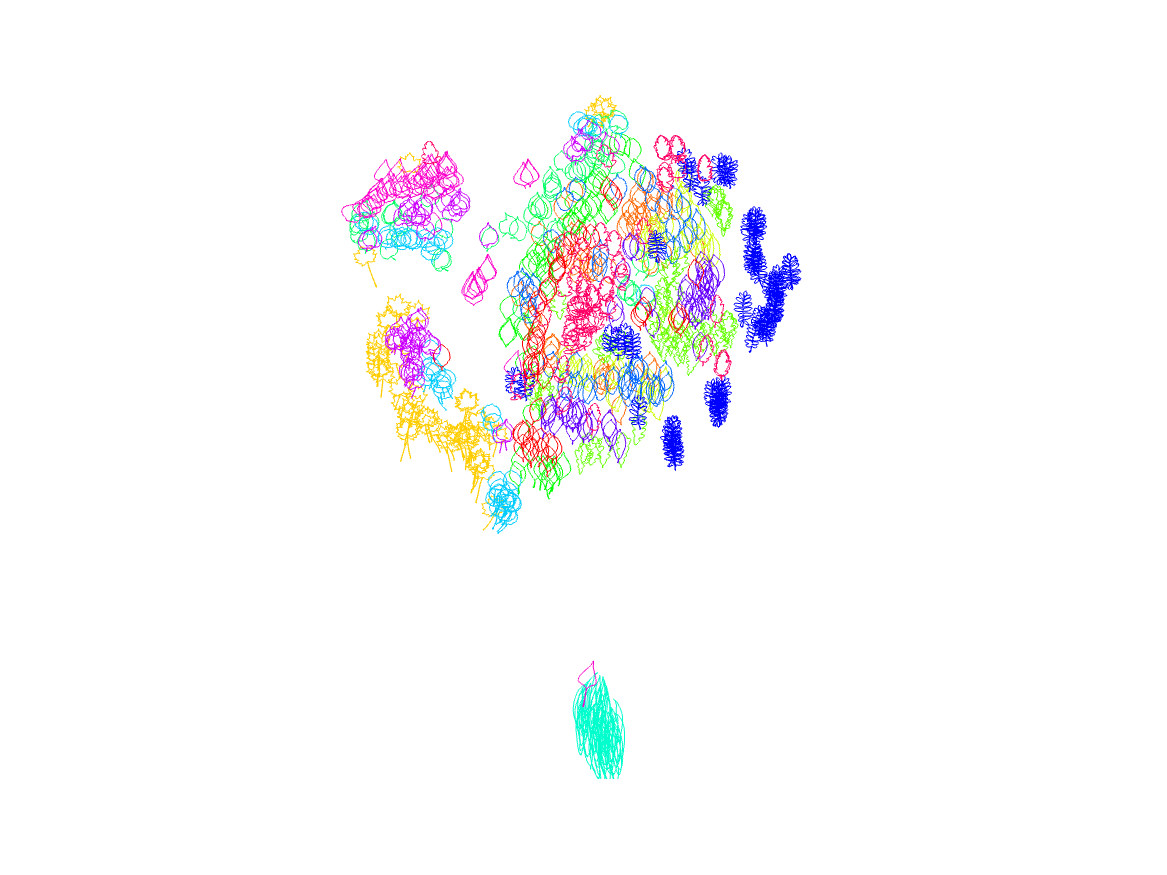}}\\

\caption{\textbf{{Normalization of the scale variability}}.  Two-dimensional representation of the distance distribution along the dataset using \texttt{tsne} algorithm. (\textbf{a}) After normalization to unit-length curves, the Dunn increases to $0.0587$. (\textbf{b}) After normalization to curves enclosing a unit area, the Dunn index increases to $0.0381$.} \label{fig_scale}
\end{figure}

\subsubsection{Standardizing the Position in~Space}\label{section_translation}
The shape of a leaf is invariant by translation in space.
We have tested three normalization procedures that can be used to eliminate the variability in~positions.

\begin{itemize}
\item[(a)] \textbf{{Starting point at the origin:}} for this normalization method, we simply substract the coordinates of the first point visited by the initial parameterized contour, leading to a parameterized curve starting at $(0,0)\in\mathbb{R}^2$.
\item[(b)] \textbf{{Center of mass of the contour at the origin:}} in this normalization method, we compute the coordinates $(\bar{x}, \bar{y})$ of the center of mass of the contour as the mean of the coordinated of points visited by the initial parameterization $\gamma(s) = (x(s), y(s))$:
\begin{equation}\label{center_line}
    \begin{array}{l}
    \bar{x} = \frac{1}{\operatorname{Length}(\gamma)}\int_0^1 x(s) \|\gamma'(s)\|ds\\
    \bar{y} = \frac{1}{\operatorname{Length}(\gamma)}\int_0^1 y(s) \|\gamma'(s)\|ds\\
    \end{array}
\end{equation}
and then we substract the coordinates of this center of mass from the initial parameterization.
\item[(c)] \textbf{{Center of gravity of the enclosed area at the origin:}} in this normalization method, we compute the coordinated $(\hat{x}, \hat{y})$ of the center of gravity of the area enclosed by the contour (i.e., of the surface of the corresponding leaf) by using Stokes theorem:
\begin{equation}\label{center_line_Stokes}
    \begin{array}{l}
    \hat{x} = \frac{1}{2\operatorname{area}(\gamma)}\int_\gamma x^2 dy\\
    \hat{y} = -\frac{1}{2\operatorname{area}(\gamma)}\int_\gamma y^2 dx\\
    \end{array}
\end{equation}
and then we substract the coordinates of this center of gravity from the initial parameterization.
\end{itemize}

As can be seen in Figure~\ref{fig3}a, on seven Acer leaves from the Swedish dataset, the~positions of the first points (in black), the~centers of mass of the contours (in orange), and~the centers of gravity of the enclosed areas (in purple) are different. In~this experiment, the~initial parameterization is counterclockwise, the~starting point of each parameterized curve coincides with the point of the contour with the largest vertical coordinate (see Section~\ref{section_rot_parameter}) and the scaling is by unit-length.

\begin{figure}[h!]
\centering

\subfloat[\centering]{\includegraphics[width=16.0cm, trim = 0cm 7cm 1cm 7cm, clip]{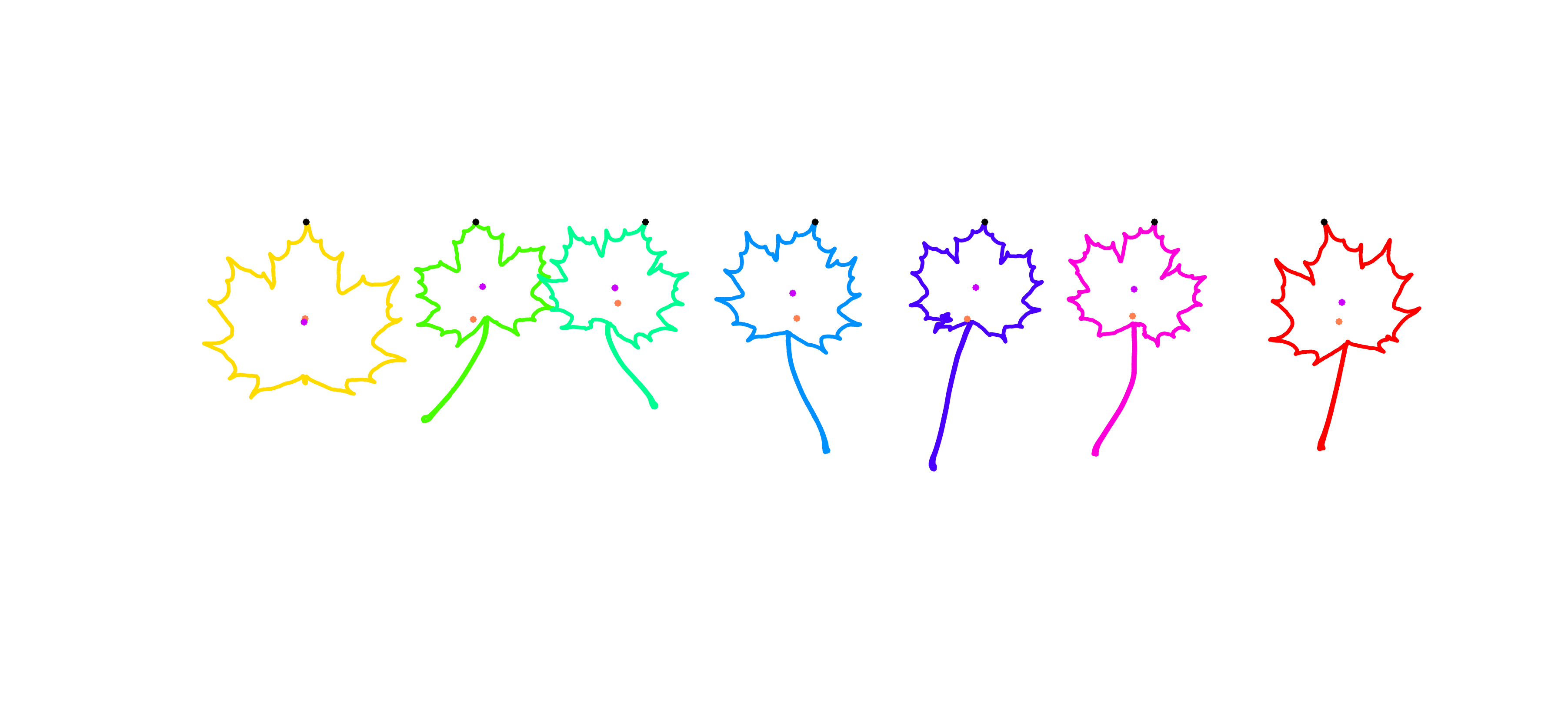}}\\

\subfloat[\centering]{\includegraphics[height=3cm,trim = 0cm 1cm 4cm 1cm, clip]{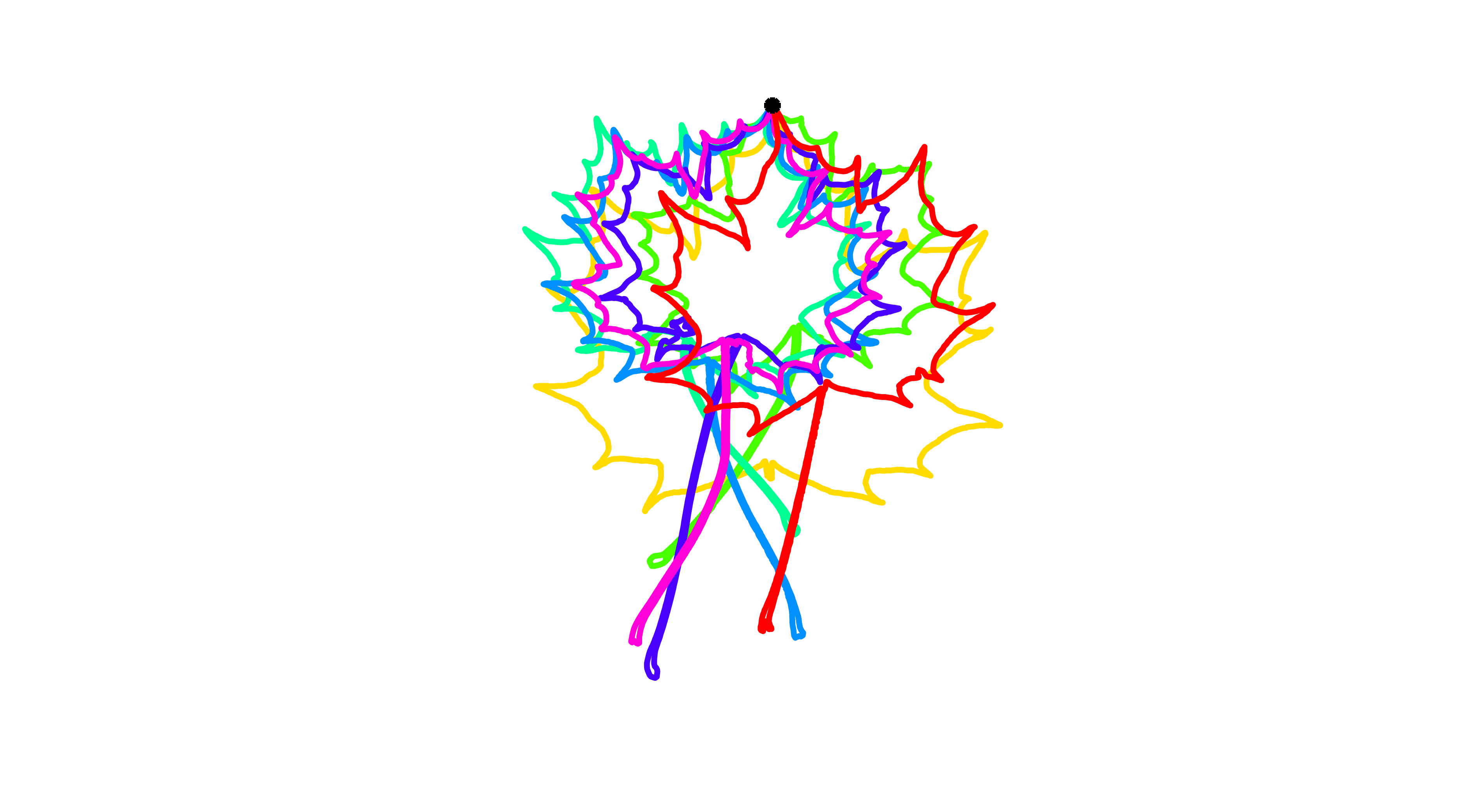}}
\subfloat[\centering]{\includegraphics[height=3cm, trim = 4cm 1cm 4cm 1cm, clip]{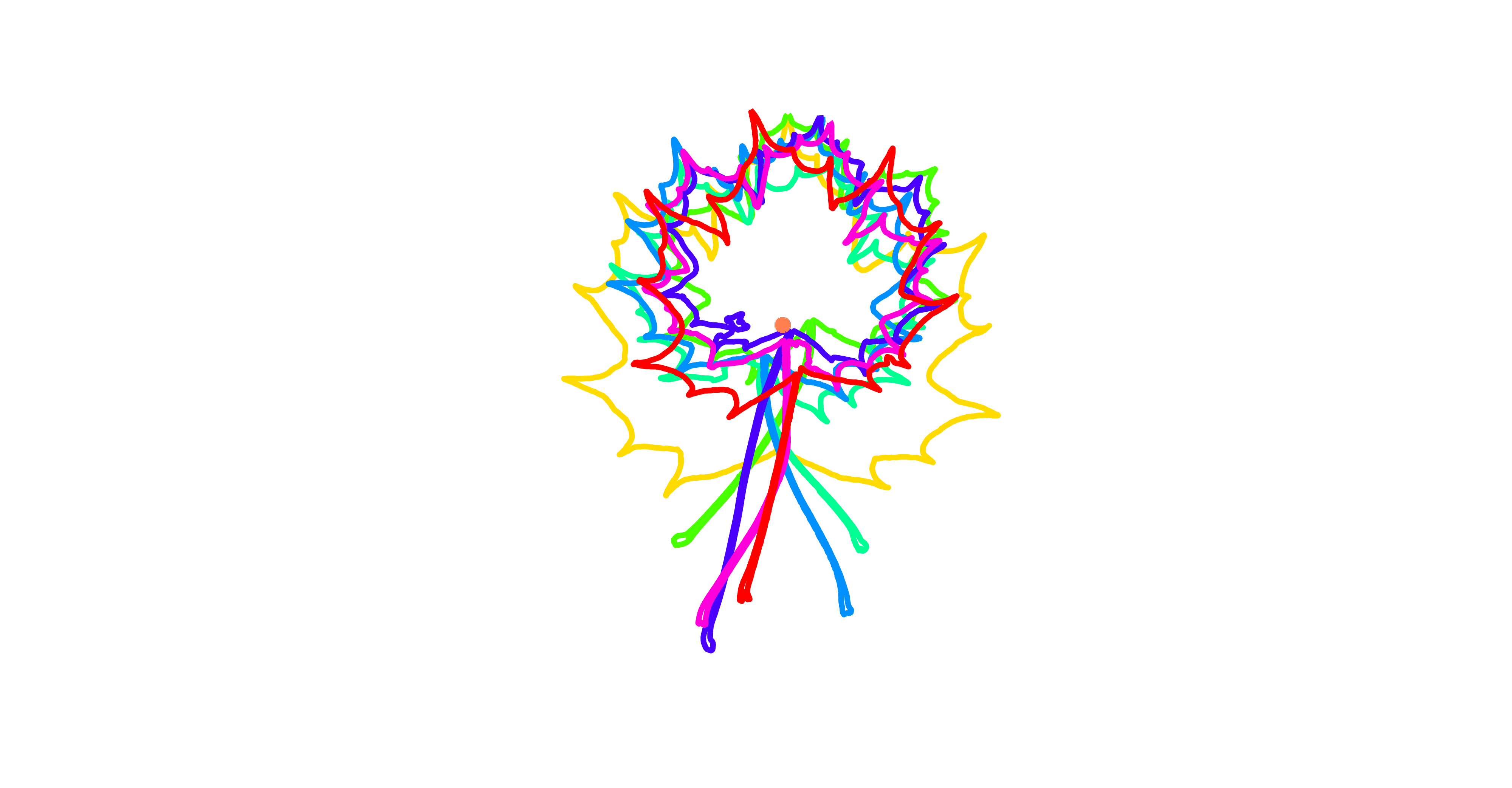}}
\subfloat[\centering]{\includegraphics[height=3cm, trim = 4cm 1cm 3cm 1cm, clip]{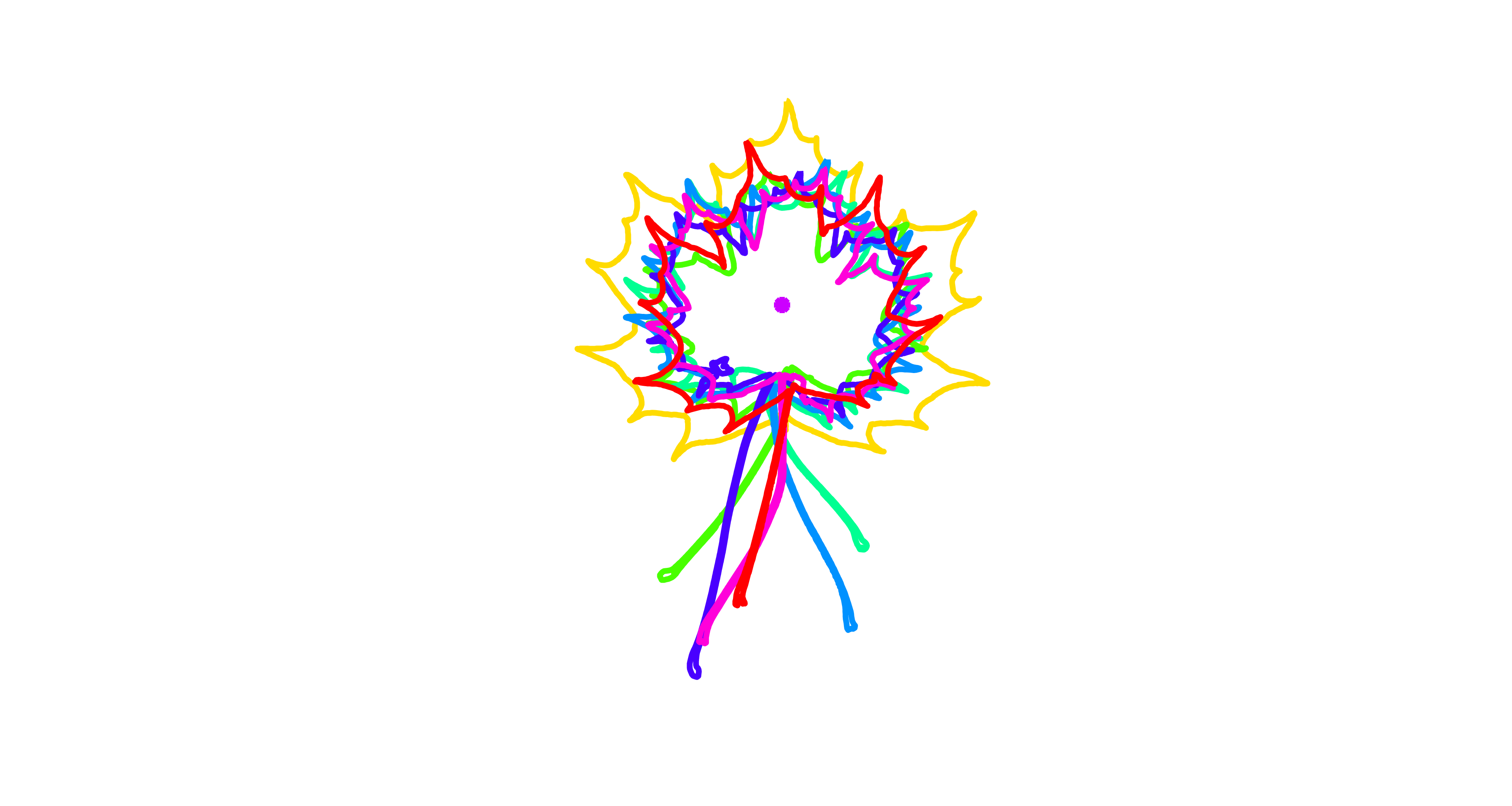}}

\caption{\textbf{{Normalization of the position in space}}. Seven Acer leaves from the Swedish leaves dataset are used to illustrate three different methods to normalize the position of contours in space. (\textbf{a}) Initial contours. Each black dot corresponds to the starting point of the parameterization and has been selected as the point of the contour with largest vertical coordinate. Each orange point corresponds to the center of mass of the contour. Each purple point corresponds to the center of gravity of the enclosed area. One can see that the length of the peduncle influences the position of the center of mass of the contour, but it~has little effect on the position of the center of gravity of the enclosed area.  (\textbf{b}) Each contour is translated in such a way that the starting point of the parameterization  of the contour is at the origin. (\textbf{c}) Each contour is translated in such a way that the center of mass of the contour is at the origin.  (\textbf{d}) Each contour is translated in such a way that the center of gravity of the enclosed area is at the origin.}
\label{fig3}
\end{figure}

In Figure~\ref{fig3}b, the~contours of the Acer leaves are centered so that the first point of their parameterization coincides with the origin. The~corresponding clustering of the training set after this normalization can be visualized in Figure~\ref{fig_centered}a. The~Dunn index decreased by this normalization process from $0.0587$ (see Section~\ref{section_translation}) to $0.0482$. 

In Figure~\ref{fig3}c, the~contours are centered so that the center of mass of the contours is at the origin. The~corresponding clustering of the training set after this normalization can be visualized in Figure~\ref{fig_centered}b. The~Dunn index slightly decreases after this normalization process from $0.0587$ to $0.0573$. 

In Figure~\ref{fig3}d, the~contours are centered so that the center of gravity of the enclosed area is at the origin.
One can see that the length of the peduncle influences the position of the center of mass of the contour, but~not the position of the center of gravity of the enclosed area, leading to a better alignment of the contours. 
The corresponding clustering of the training set after this normalization can be visualized in Figure~\ref{fig_centered}c. The~Dunn index increases after this normalization process from $0.0587$ to $0.0702$. 
Therefore, in~what follows, the~center of the enclosed area is used to center~contours.

\begin{figure}[h!]
\centering
\subfloat[\centering]{\includegraphics[width=6cm]{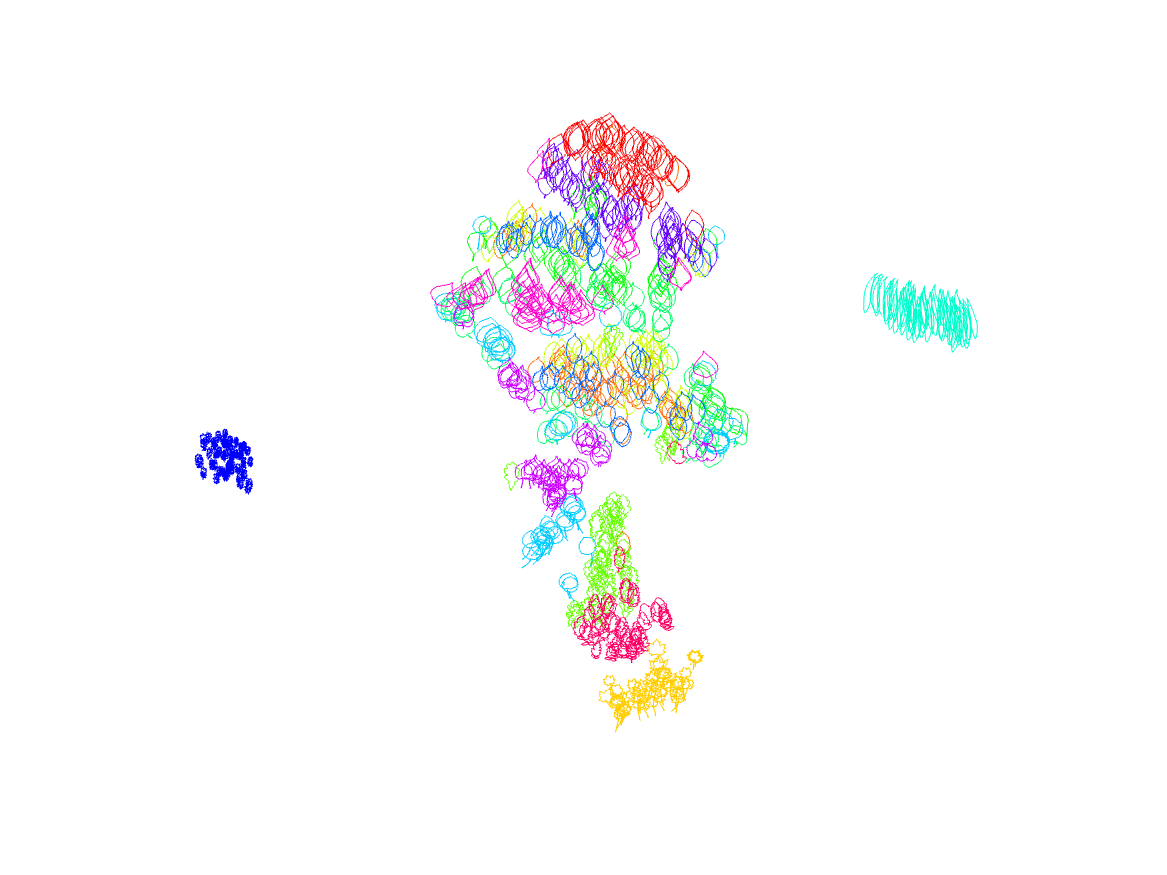}}
\subfloat[\centering]{\includegraphics[width=6cm]{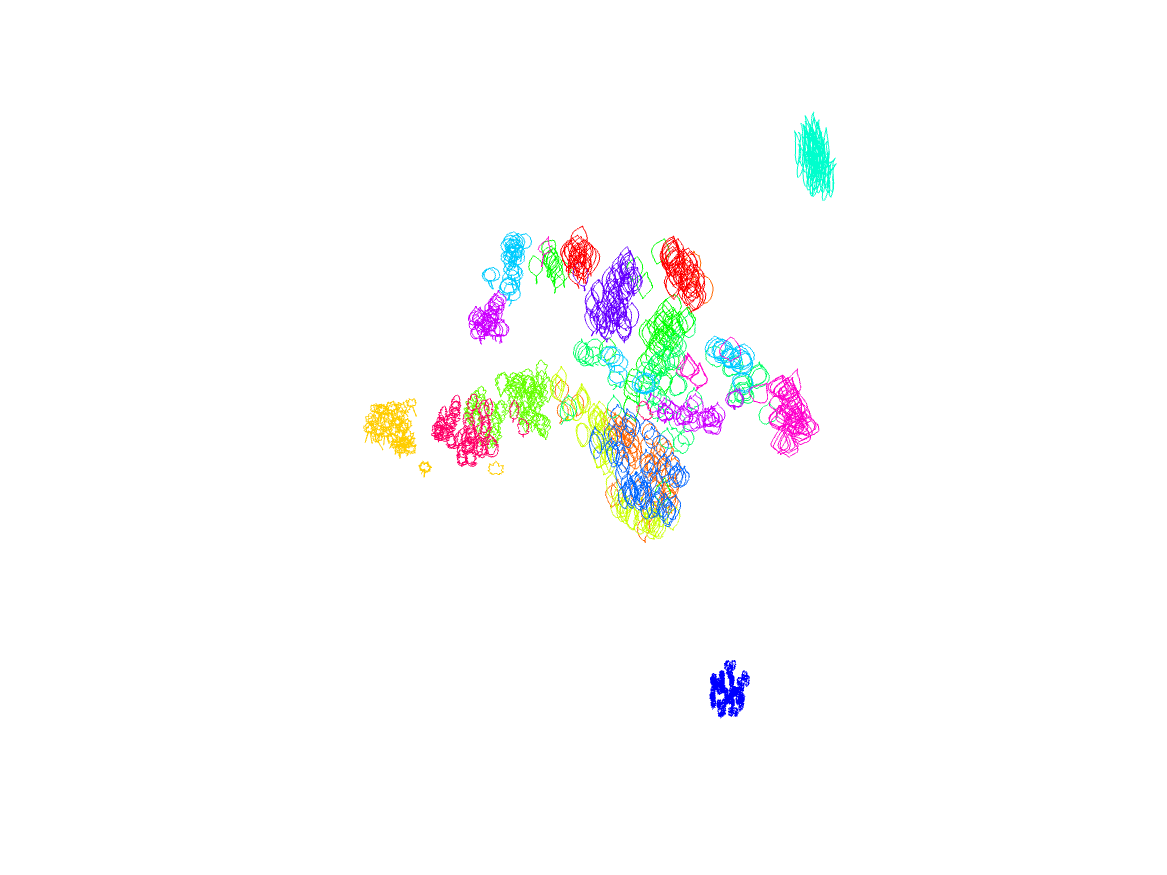}}
\subfloat[\centering]{\includegraphics[width=6cm]{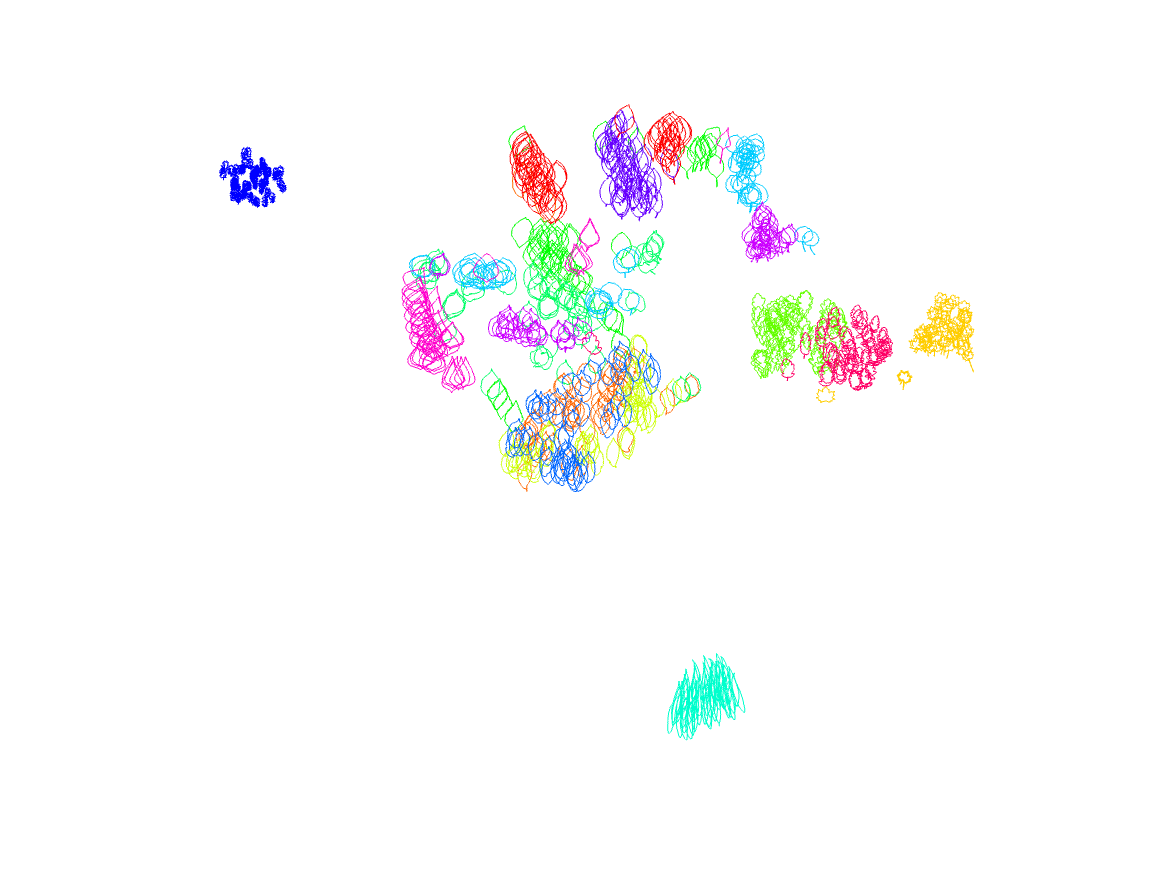}}

\caption{\textbf{{Normalization of the position in space}}.  Two-dimensional representation of the distance distribution along the dataset using \texttt{tsne} algorithm. (\textbf{a}) After centering the curves to the same starting point, the~Dunn index decreases from $0.0587$ to $0.0482$. (\textbf{b}) After centering the curves to have the center of mass at the origin, the~Dunn index decreases from $0.0587$ to $0.0573$. (\textbf{c}) After centering the curves to have the center of gravity of enclosed area at the origin, the~Dunn index increases from $0.0587$ to~$0.0702$.}\label{fig_centered}
\end{figure}

\subsubsection{Standardizing the Orientation in~Space}\label{section_orientation}

The leaves in the dataset we are considering have different orientations in space and need to be rotated in a consistent way to eliminate the orientation variability. We have tested two normalization procedures to align the orientations through the~dataset.
\begin{itemize}
    \item[(a)] \textbf{{Axes of the approximating ellipse aligned:}}  Each contour is rotated so that the ellipse that best approximates the contour has its minor axis along the horizontal axis, and~its major axis vertically. We did not encounter contours with equal minor and major axes. 
    \item[(b)] \textbf{{Segment that joins the tip of the leaf to the center of the enclosed area is placed vertically:}} Each contour is rotated so as to position the center of the enclosed area vertically below the highest point of the contour.
\end{itemize}

The first normalization method does not lead to good results because of the presence of leaves with a peduncle and leaves without a peduncle in the same class. As~can be seen in Figure~\ref{fig4_rotation}b on the example of Acer leaves, the~alignment of the major and minor axis of the approximating ellipse leads to inconsistent orientation of the leaf without peduncle with respect to the other leaves. After~this normalization procedure, the~Dunn index decreases from $0.0702$ to $0.0268$. The corresponding clustering can be visualized in Figure~\ref{fig_rotated}a.

The second normalization method gives better results (see Figure~\ref{fig4_rotation}c), although~the Dunn index decreases slightly from $0.0702$ to $0.0636$. We will choose this second normalization method, in~order to normalize the orientation variability and obtain consistent classification~results. The corresponding clustering can be visualized in Figure~\ref{fig_rotated}b.

\begin{figure}[h!]
\centering
\subfloat[\centering]{\includegraphics[width=12.0cm]{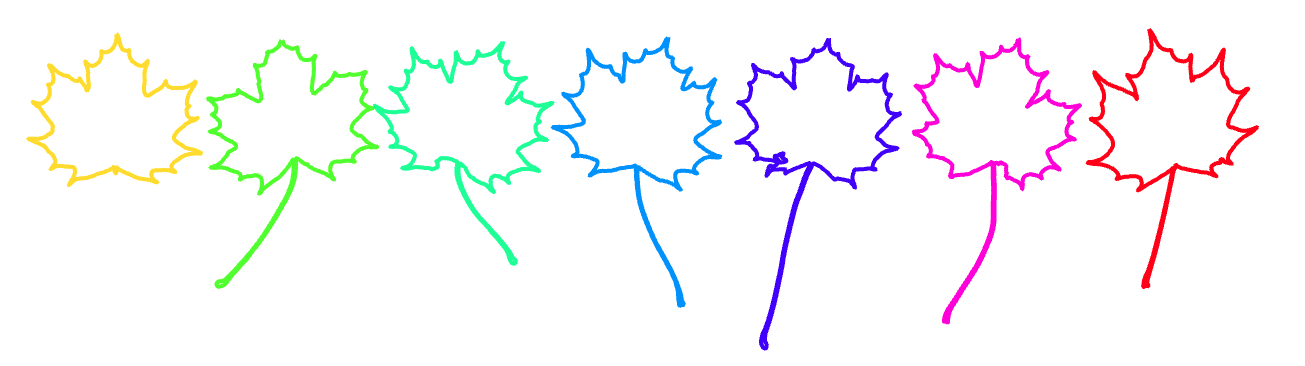}}\\
\subfloat[\centering]{\includegraphics[width=12.0cm]{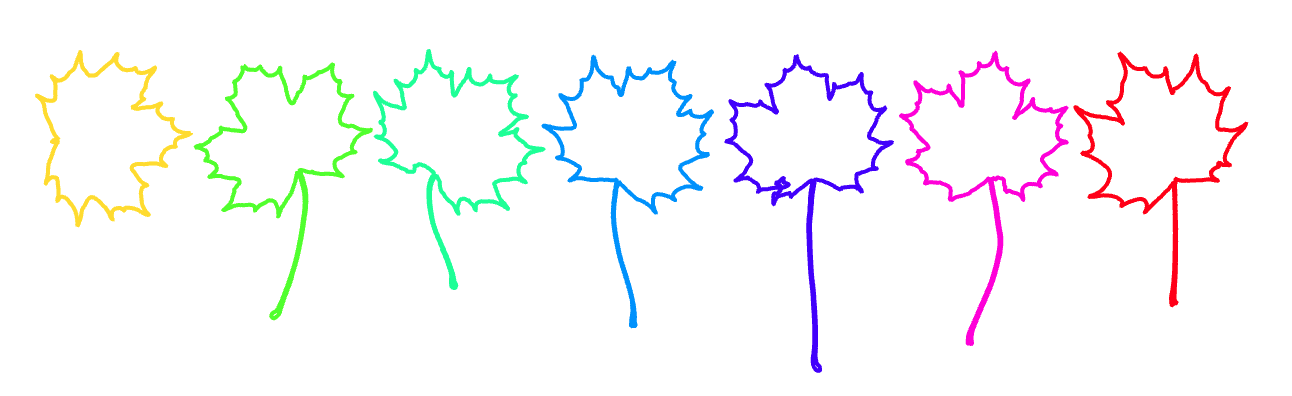}}\\
\subfloat[\centering]{\includegraphics[width=12.0cm]{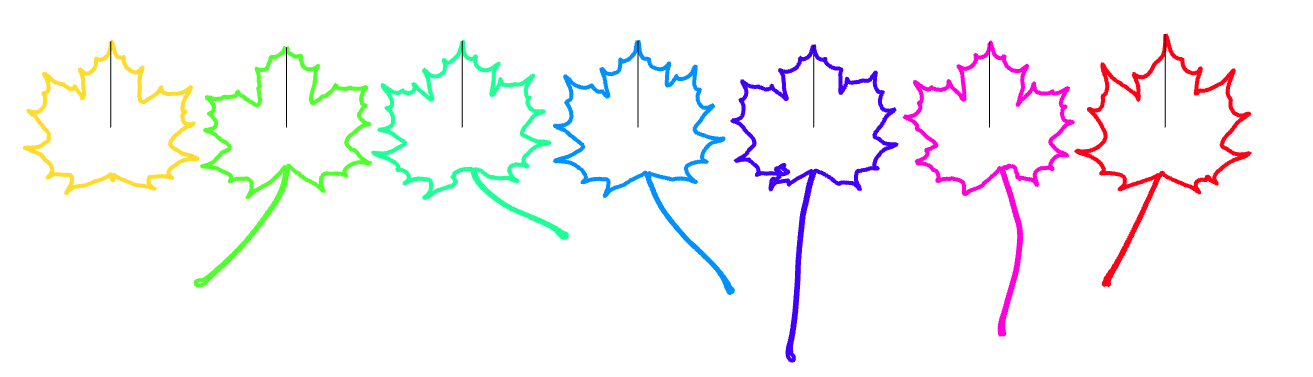}}
\caption{\textbf{{Normalization of the orientation variability}}. Seven Acer leaves are used to illustrate different methods to normalize the orientation in space in a consistent manner through the dataset.  (\textbf{a}) Initial contours. (\textbf{b})  Each contour is rotated in such a way that the approximating ellipse has its minor axis along the horizontal axis, and~its major axis vertically. Note that the first Acer leaf has an inconsistent orientation with respect to the other leaves with peduncles. (\textbf{c})  Each contour is rotated in such a way that the segment (in black)  joining the center of gravity to the first point is vertical. }
\label{fig4_rotation}
\end{figure}

\begin{figure}[h!]
\centering
\subfloat[\centering]{\includegraphics[width=6cm]{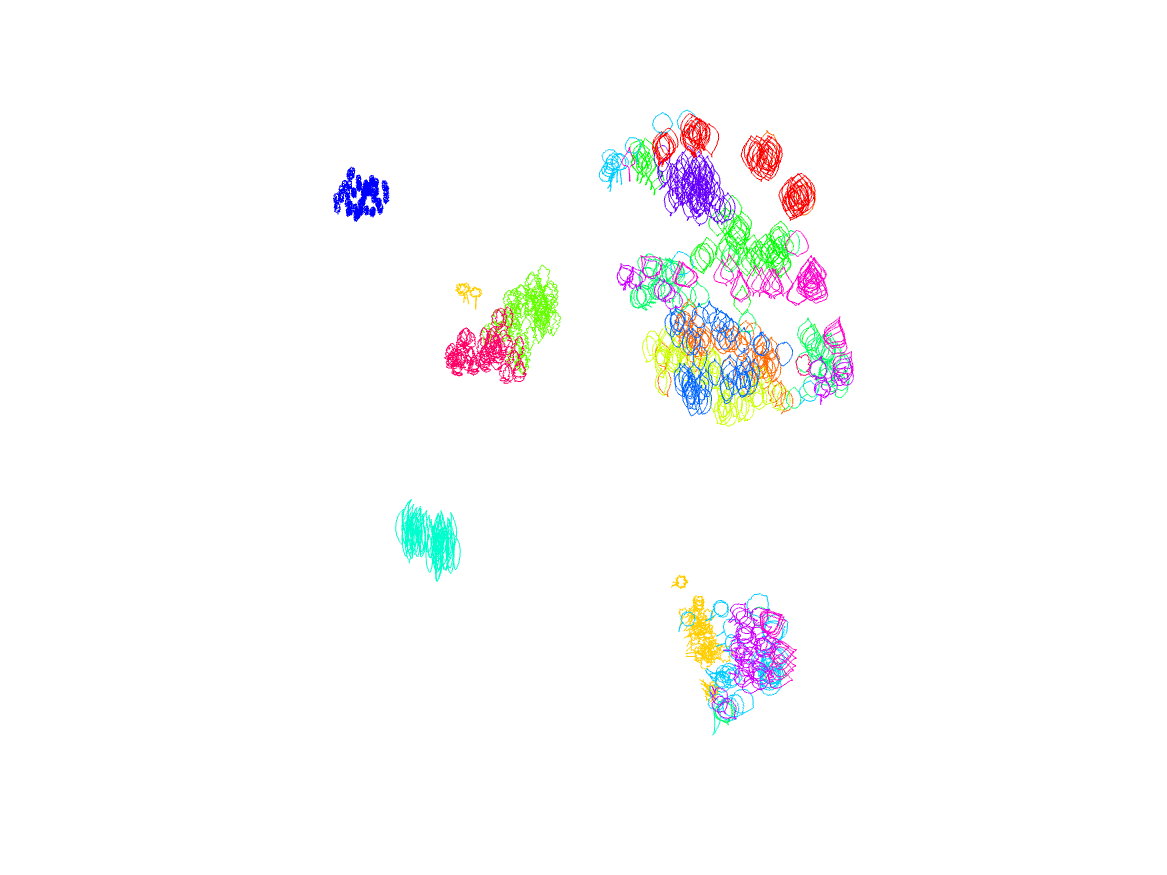}}
\subfloat[\centering]{\includegraphics[width=6cm]{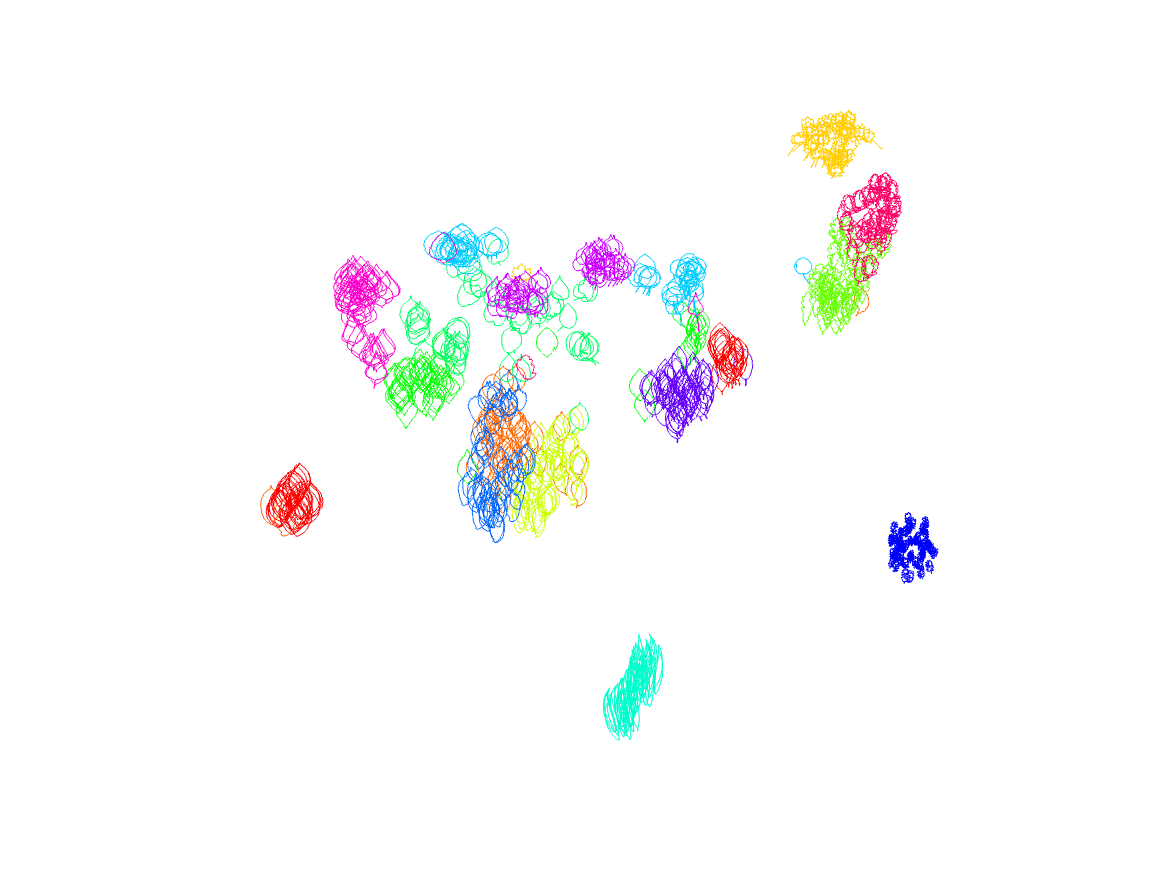}}

\caption{\textbf{{Normalization of the orientation in space}}.  Two-dimensional representation of the distance distribution along the dataset using \texttt{tsne} algorithm. (\textbf{a}) After rotation of the curves to have their approximating ellipse aligned with the axis, the~Dunn index decreases from $0.0702$ to $0.0294$. (\textbf{b})~After rotation of the curves so that the segment joining the center of gravity of the enclosed area and the tip of the leave is vertical, the~Dunn index decreases slightly from $0.0702$ to $0.0658$.} \label{fig_rotated}
\end{figure}

\subsubsection{Resulting Normalization over Finite-Dimensional Shape-Preserving~Groups}

The resulting normalization over the finite-dimensional shape-preserving group consisting of scalings, translations, rotations in space and rotations in parameter space is illustrated for different classes of leaves in Figure~\ref{fig_normalization}. 
Let us summarize here the normalization steps that were~selected:
\begin{itemize}
    \item {Counterclockwise travel along the curves} (Section~\ref{section_counterclockwise}).
    \item {Starting point at the tip of the leaves} (Section~\ref{section_rot_parameter}).
    \item {Unit-length curves} (Section~\ref{section_scaling}).
    \item {Center of gravity of the enclosed area at the origin} (Section~\ref{section_translation}).
    \item {Segment joining the tip of the leaf to the center of gravity vertical} (Section~\ref{section_orientation})
\end{itemize}

The remaining shape-preserving group is infinite-dimensional and consists of orientation-preserving reparameterizations fixing the starting (and ending) point. Mathematically, this group corresponds to the following subgroup of $\operatorname{Diff}^+(\mathbb{S}^1)$:
\[\operatorname{Diff}^+_0(\mathbb{S}^1) = \{\Phi \in \operatorname{Diff}^+(\mathbb{S}^1), \Phi(0) = 0\}.\]

\begin{figure}[h!]
\centering
\subfloat[\centering]{\includegraphics[width=3.0cm]{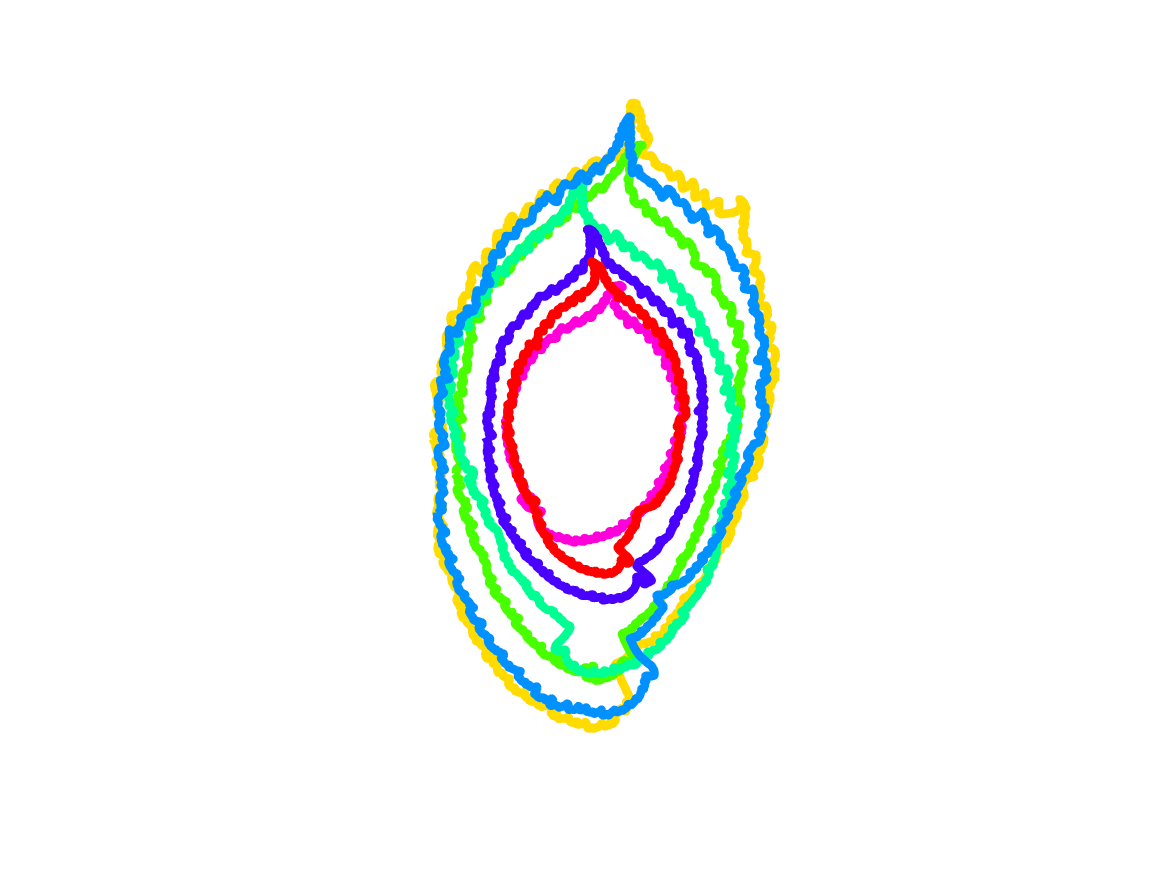}}
\subfloat[\centering]{\includegraphics[width=3.0cm]{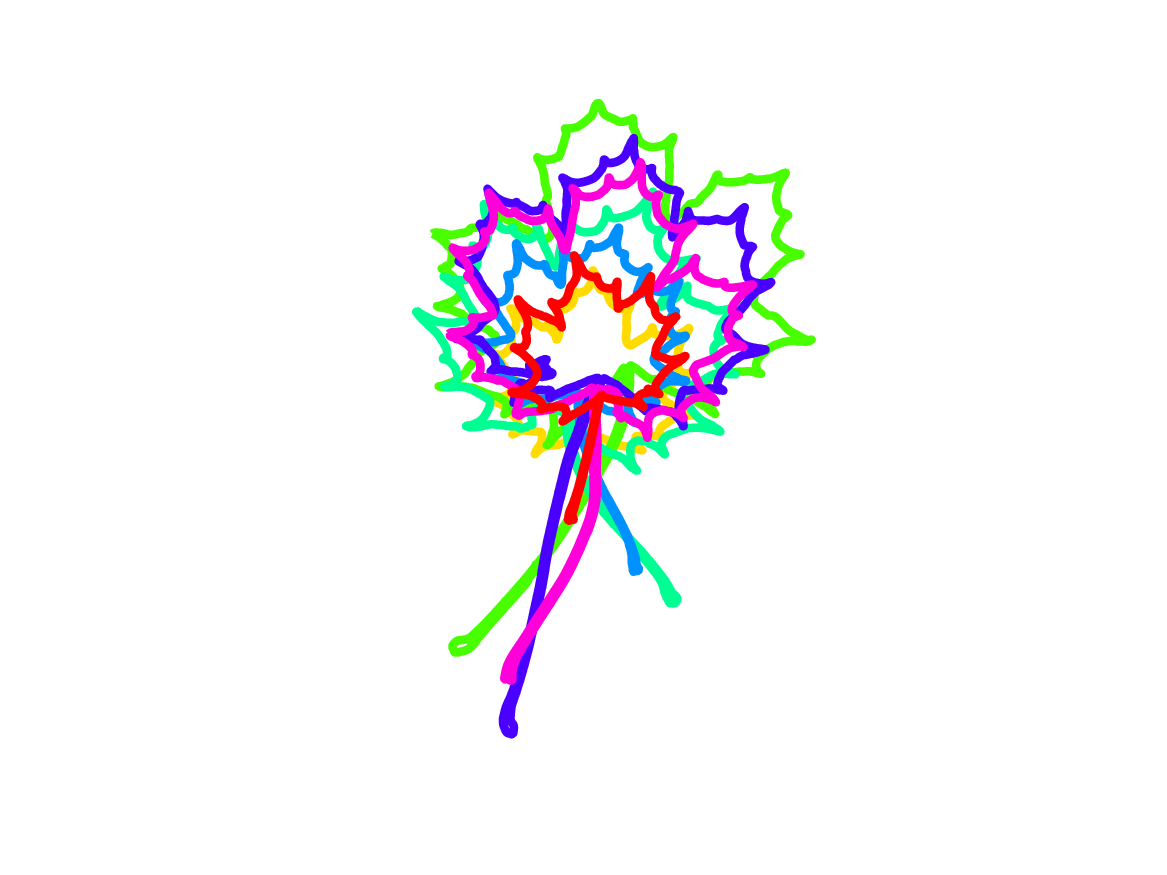}}
\subfloat[\centering]{\includegraphics[width=3.0cm]{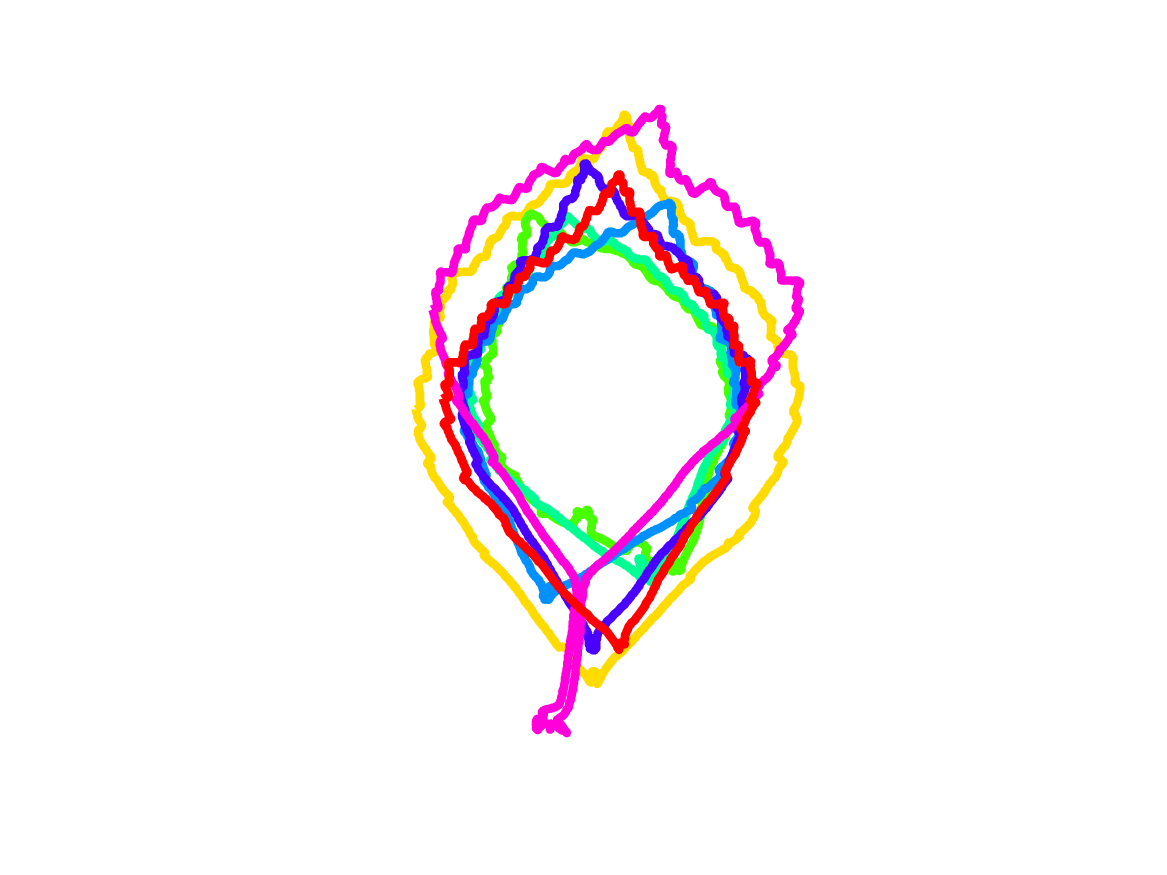}}
\subfloat[\centering]{\includegraphics[width=3.0cm]{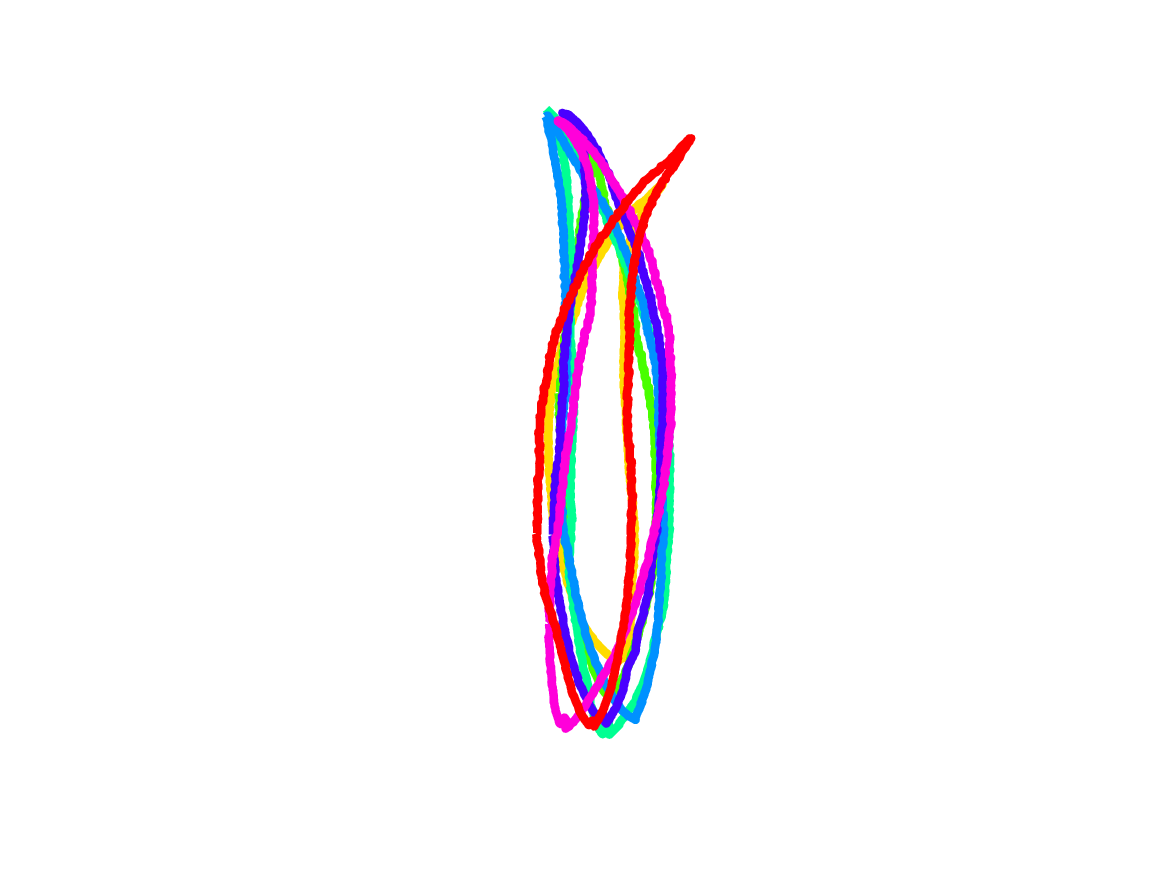}}
\subfloat[\centering]{\includegraphics[width=3.0cm]{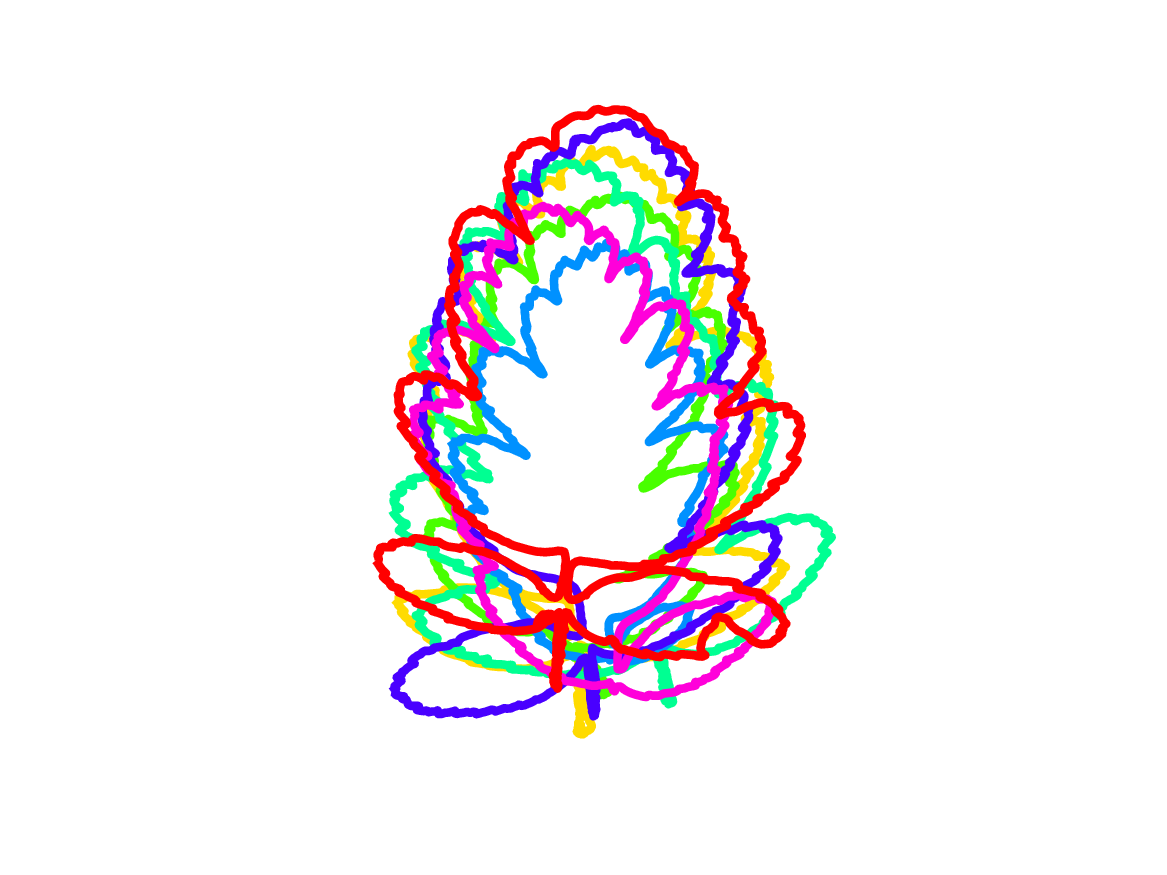}}\\
\subfloat[\centering]{\includegraphics[width=3.0cm]{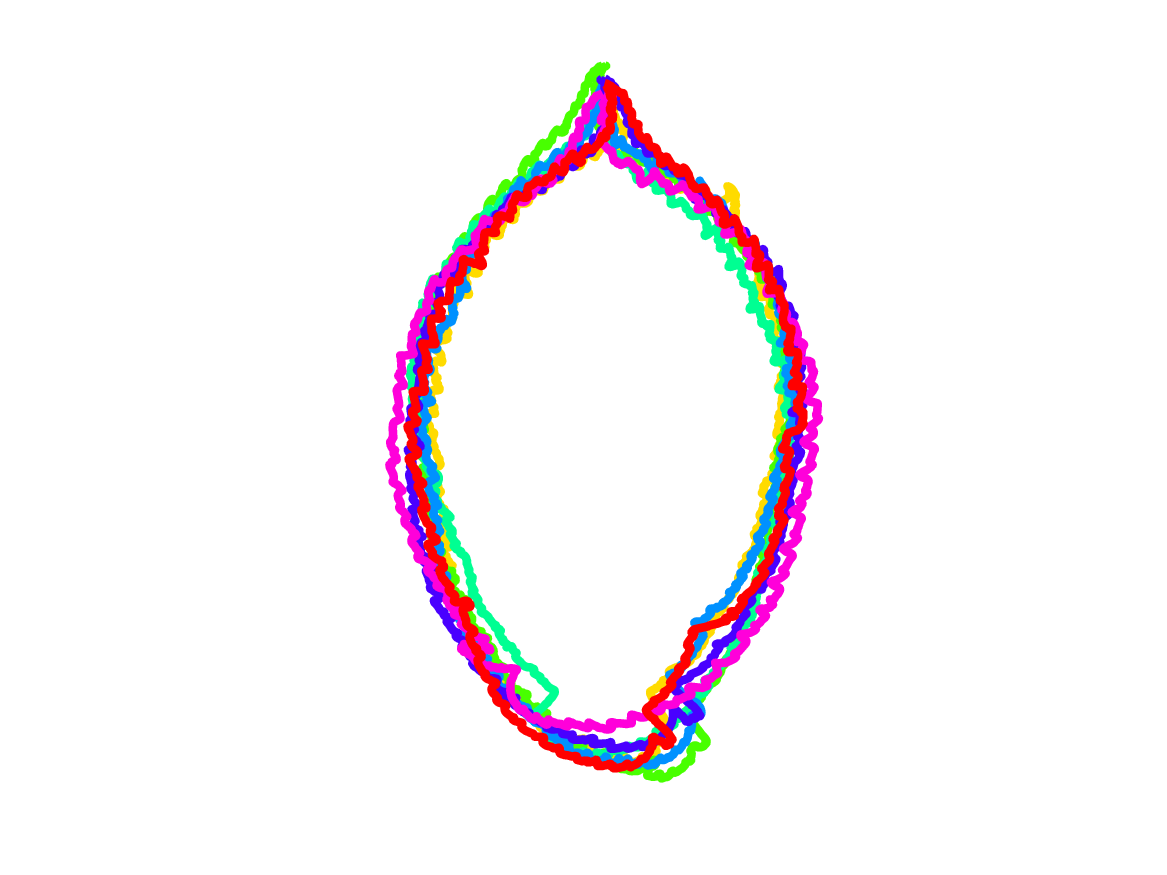}}
\subfloat[\centering]{\includegraphics[width=3.0cm]{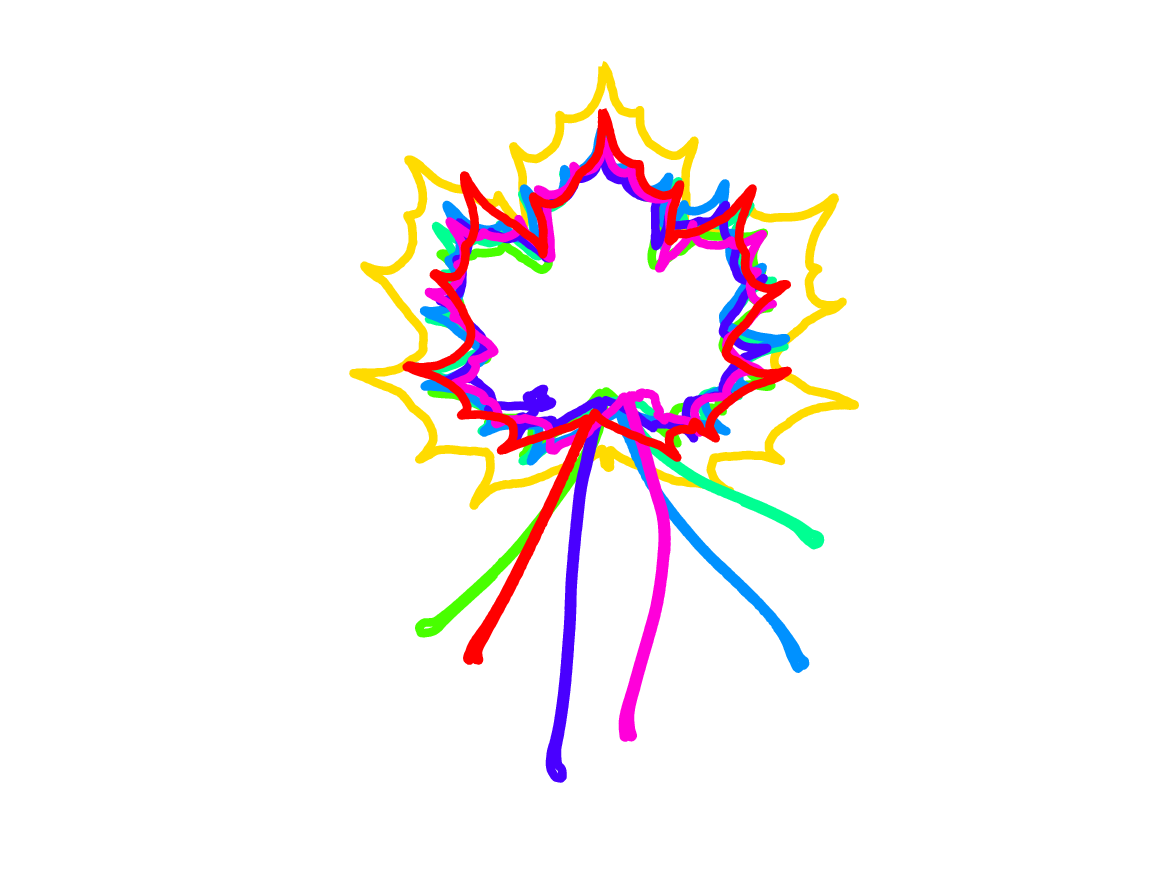}}
\subfloat[\centering]{\includegraphics[width=3.0cm]{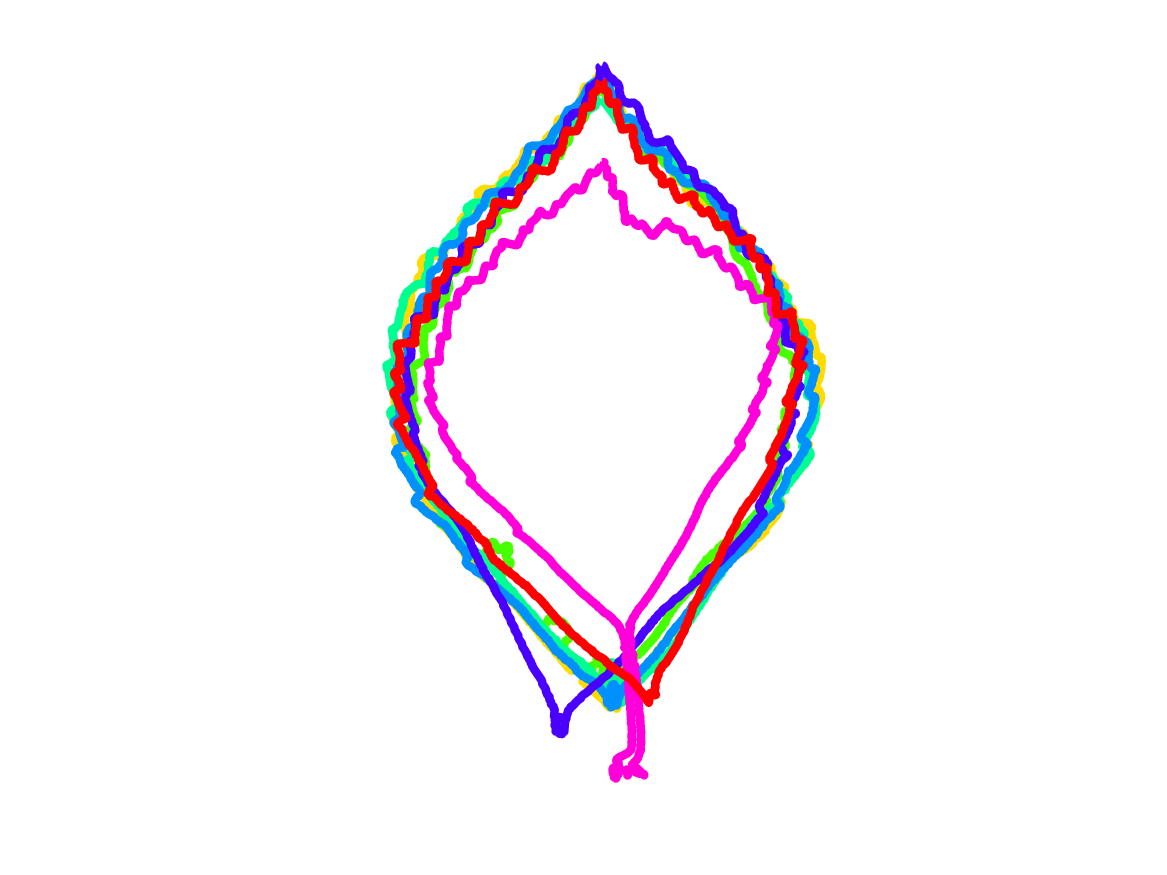}}
\subfloat[\centering]{\includegraphics[width=3.0cm]{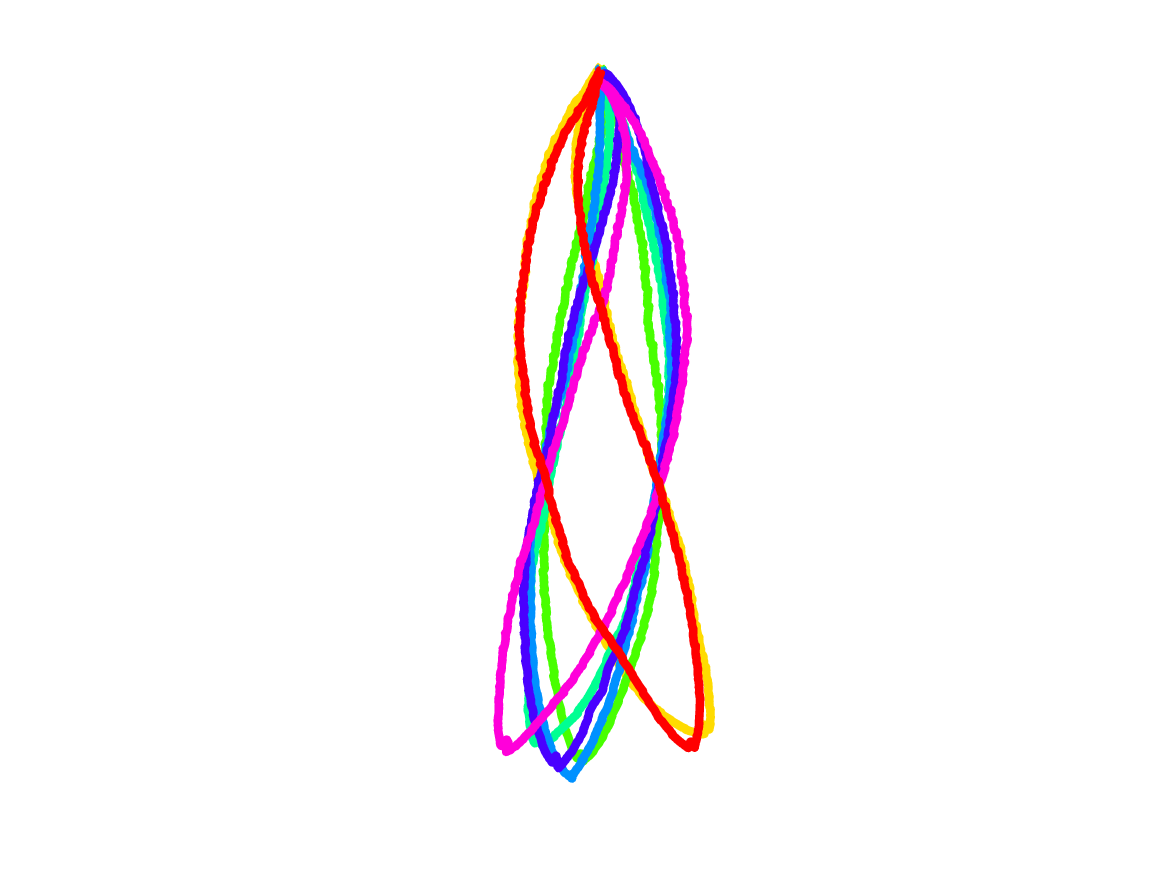}}
\subfloat[\centering]{\includegraphics[width=3.0cm]{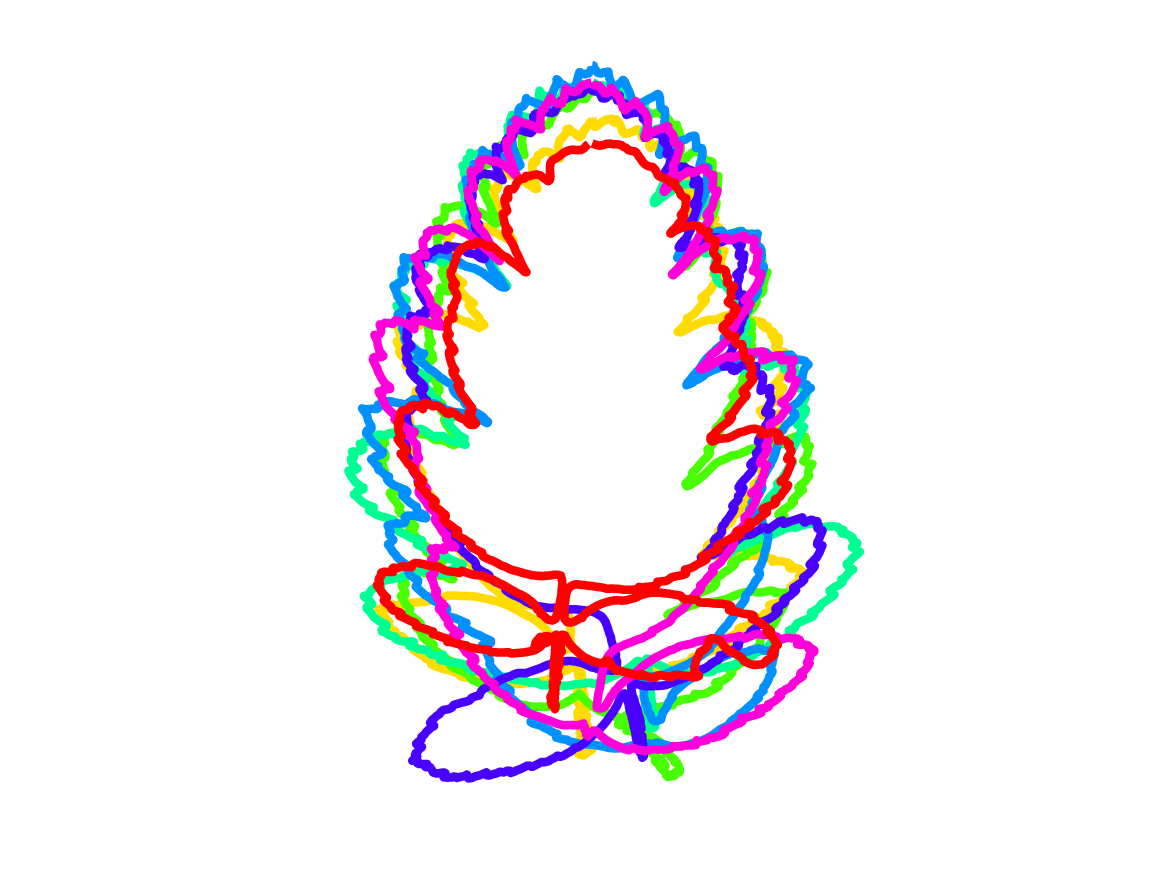}}\\

\caption{\textbf{{Resulting normalization over the group of scalings, translations, rotations in space, and rotation in parameter space}}. Several leaves of the same class are depicted before normalization (upper row (\textbf{a}--\textbf{e})) and after normalization (lower row (\textbf{f}--\textbf{j})). }
\label{fig_normalization}
\end{figure}

\subsection{A New $2$-Parameter Family of Canonical~Parameterizations}\label{Section_clock_parameterization}
\unskip
\subsubsection{Clock Parameterization of Jordan~Curves}\label{sec_just_clock}
In this section, we introduce a new canonical parameterization of simple plane curves, called the clock parameterization. We will make use of the analogy with a traditional clock to explain how this parameterization is constructed.  Suppose that we have \mbox{$720= 12\times 60$ points} to place along the contour of the Acer leaf depicted in Figure~\ref{fig_clock}a.  If~we place 720 points uniformly along the contour and cut the enclosed area as a pizza from its center of gravity to the points corresponding to a multiple of 60, then we obtain 12~pieces of different angles. This is illustrated in Figure~\ref{fig_clock}a by a color change with every 60~points. In~contrast, the~clock parameterization automatically places each point numbered by a multiple of 60 in such a way that the corresponding angle is precisely 360/12 degrees, hence at the positions of the hours on a traditional clock (see Figure~\ref{fig_clock}b). To~place these 12~keypoints at the hours positions, we compute the angle between the vertical line and the segment connecting the center of gravity to a point traveling along the contour at constant speed. The~graph of the angle function for the Acer leaf is illustrated in Figure~\ref{fig_clock}c. It allows us to detect the constant speed parameter of the first point reaching an angle multiple of 360/12 degrees.  In~Figure~\ref{fig_clock}c, the~horizontal lines are spaced every 360/12 degrees and hit the angle function graph precisely at these constant speed parameters.
Between two consecutive points that have these particular constant speed parameters, we distribute exactly 60 points uniformly along the portion of the curve between them. The~resulting reparameterization of the Acer leaf is such that each colored portion of the curve describes the same angle with respect to the center of gravity and contains exactly the same number of points. In~Figure~\ref{fig_clock}d--f, the same procedure is applied to the more challenging shape of Sorbus leaf. Note the difference in the density of points on the light blue portion and on the dark blue portion of the curve. This is due to the structure of compound of Sorbus leaves, which are made up of multiple leaflets arranged along a central~stalk.

In the previous procedure, the~number of subdivisions  of 360 was set to 12. For~each choice of the number $n$ of subdivisions, we obtain a different reparameterization procedure for Jordan curves. In~Figure~\ref{fig_clock_n}, we illustrate how the resulting parameterization of a curve depends on the number of subdivisions $n$. In~this case, we distribute 1000 points along the contour of an Acer leaf, this time with a peduncle. The~first row in Figure~\ref{fig_clock_n} corresponds to a parameterization with constant speed.  From~left to right, we use 20, 50, and~100~subdivisions to color the curve. The~corresponding clock parameterizations are depicted in Figure~\ref{fig_clock_n}d--f, respectively.  The~graph of the angle function with equally spaced horizontal lines is depicted in Figure~\ref{fig_clock_n}g--i for 20, 50, and 100~subdivisions of 360~degrees.
In contrast to the constant speed parameterization, for~the clock parameterization,  the~density of points along the peduncle decreases with the number of subdivisions. Indeed, while the number of subdivisions increases, the~angle formed by each colored piece of curve decreases. Since, on each colored piece of curve, we distribute the same number of points, the~density of points on the piece containing the long peduncle decreases~drastically. 

\begin{Remark}
    The clock parameterization is well defined as long as the center of gravity is within the interior of the contour. In~practice, this was generally the case, but~we encountered some leaves with a center of gravity outside the interior. In~these cases, the~center of gravity was replaced by a reference point nearby but located inside the leaf. There are many possible automatic procedures for doing~so:
    
\begin{itemize}
        \item After computing the closest point of the contour to the center of gravity, the~reference point is initialized at the center of gravity and moved in the direction of this closest point until its index with respect to the contour  increases from $0$ to $1$.
        \item The reference point is initialized at the center of gravity and moved in the direction of the tip of the leaf until its index with respect to the contour increases from $0$ to $1$.
        \item After computing the Delaunay triangulation for the contour and subsequently creating the Voronoi diagram, the~leaf is translated so that the closest Voronoi vertex is at the origin.
        \item After computing the closest point of the contour to the center of gravity, we consider triangles with one vertex at the closest point and two other vertices on the contour, and~compute their centroids. We choose as a reference point a centroid near the center of gravity that has the property to be inside the shape, and we~perform a translation so that this reference point is at the~origin. 
\end{itemize} 
    
    The first solution has the advantage of generalizing to datasets without distinguished point along the contour (which could take the role of the tip of the leaf), the~second solution has the advantage of being compatible with the rotation alignment performed in Section~\ref{section_orientation}. However, these two solutions are dependent on the step size of the displacement, which is an extra data-dependent parameter. In~contrast, the~last two translation procedures do not require learning an extra hyperparameter and are therefore preferred.
\end{Remark}

\begin{figure}[h!]
\centering
\subfloat[\centering]{\includegraphics[width=8.0cm]{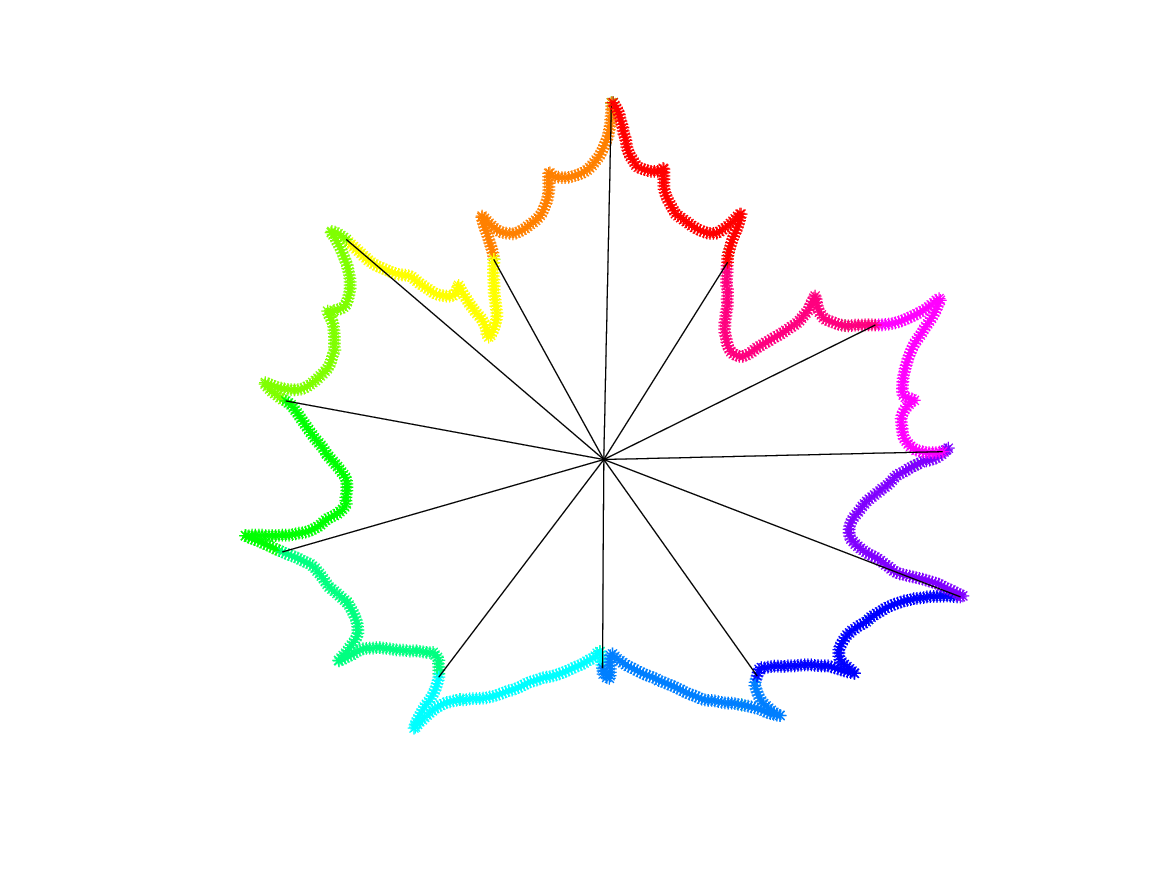}}
\subfloat[\centering]{\includegraphics[width=8.0cm]{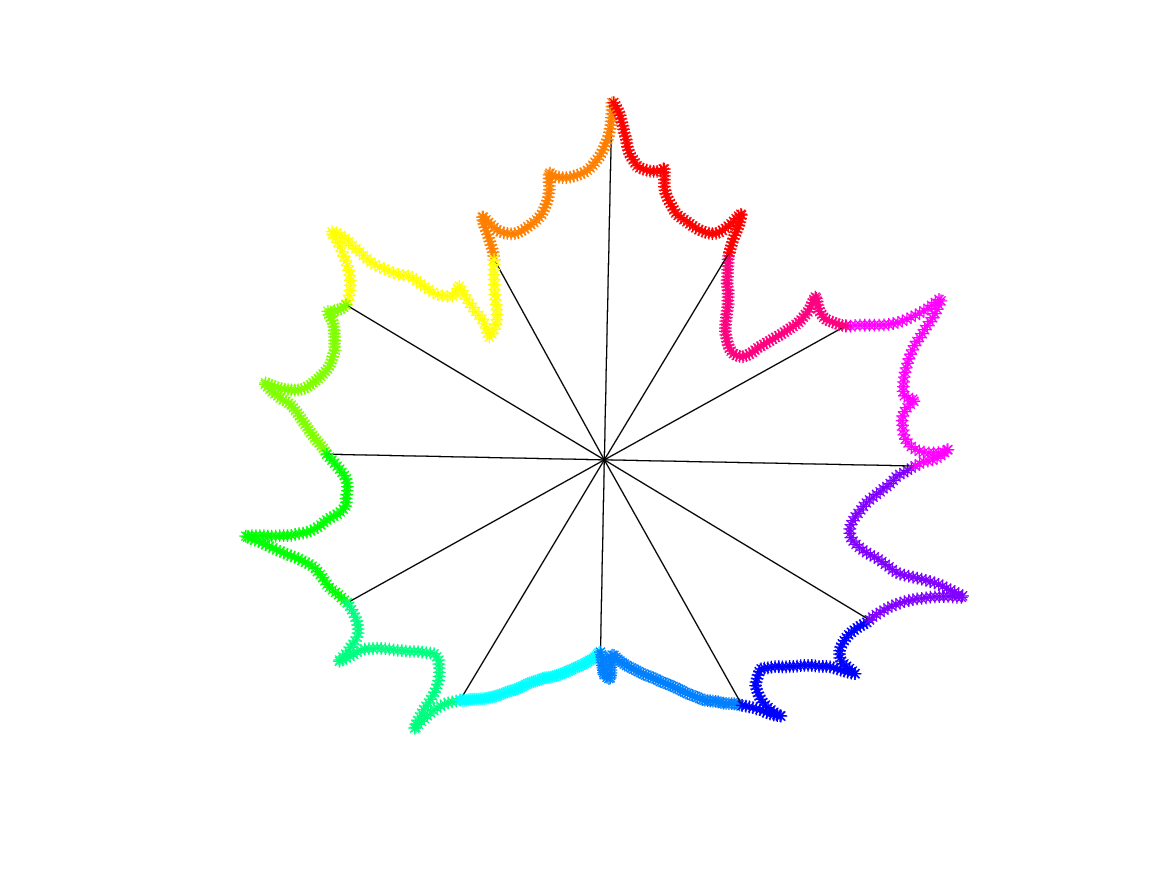}}\\
\subfloat[\centering]{\includegraphics[width=6.0cm]{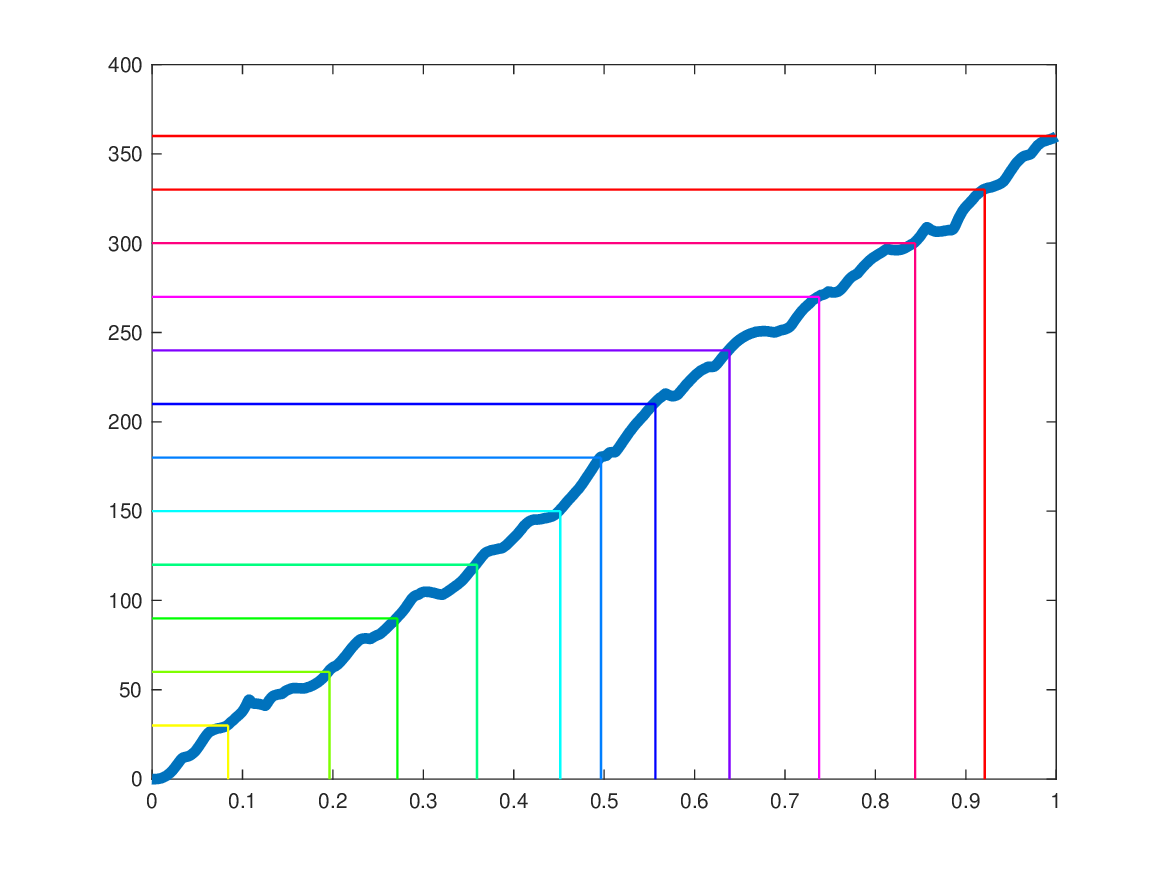}}\hspace{2cm}
\subfloat[\centering]{\includegraphics[width=6.0cm]{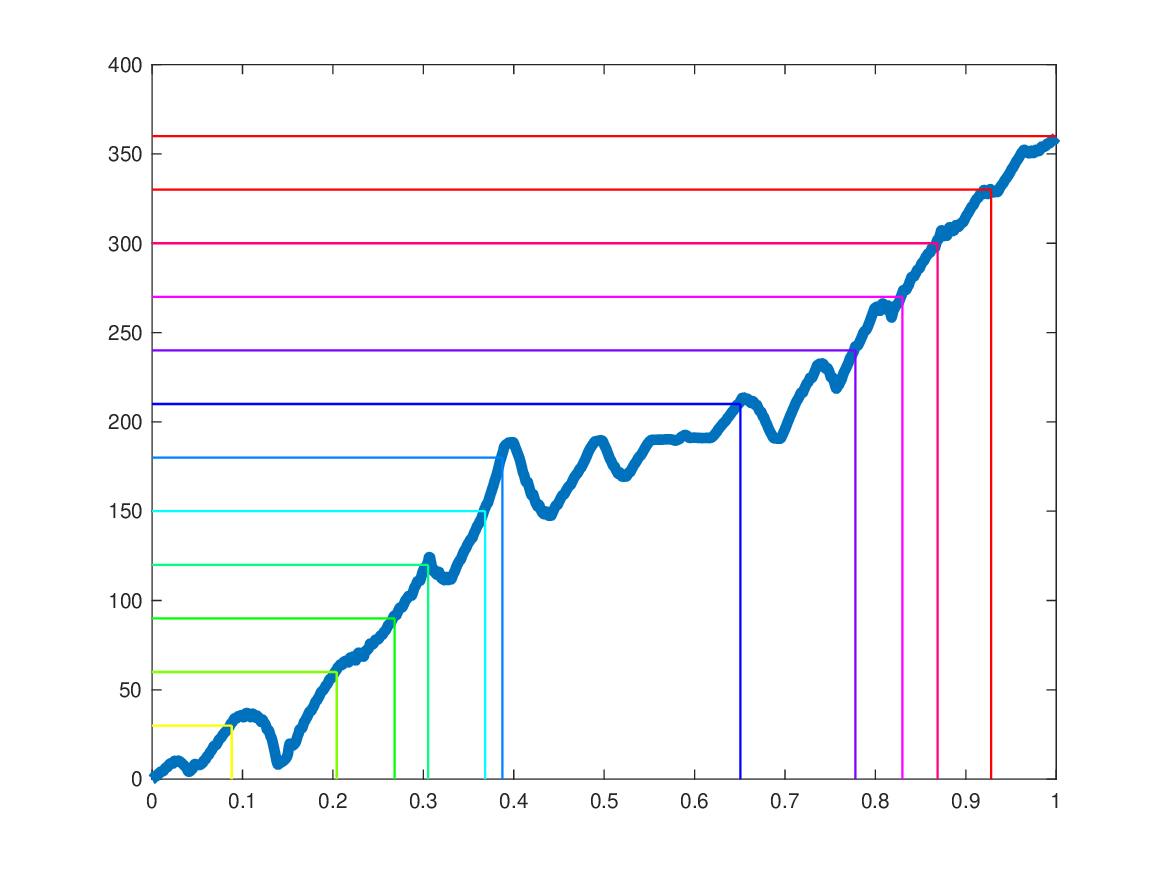}}\\

\subfloat[\centering]{\includegraphics[width=9.0cm]{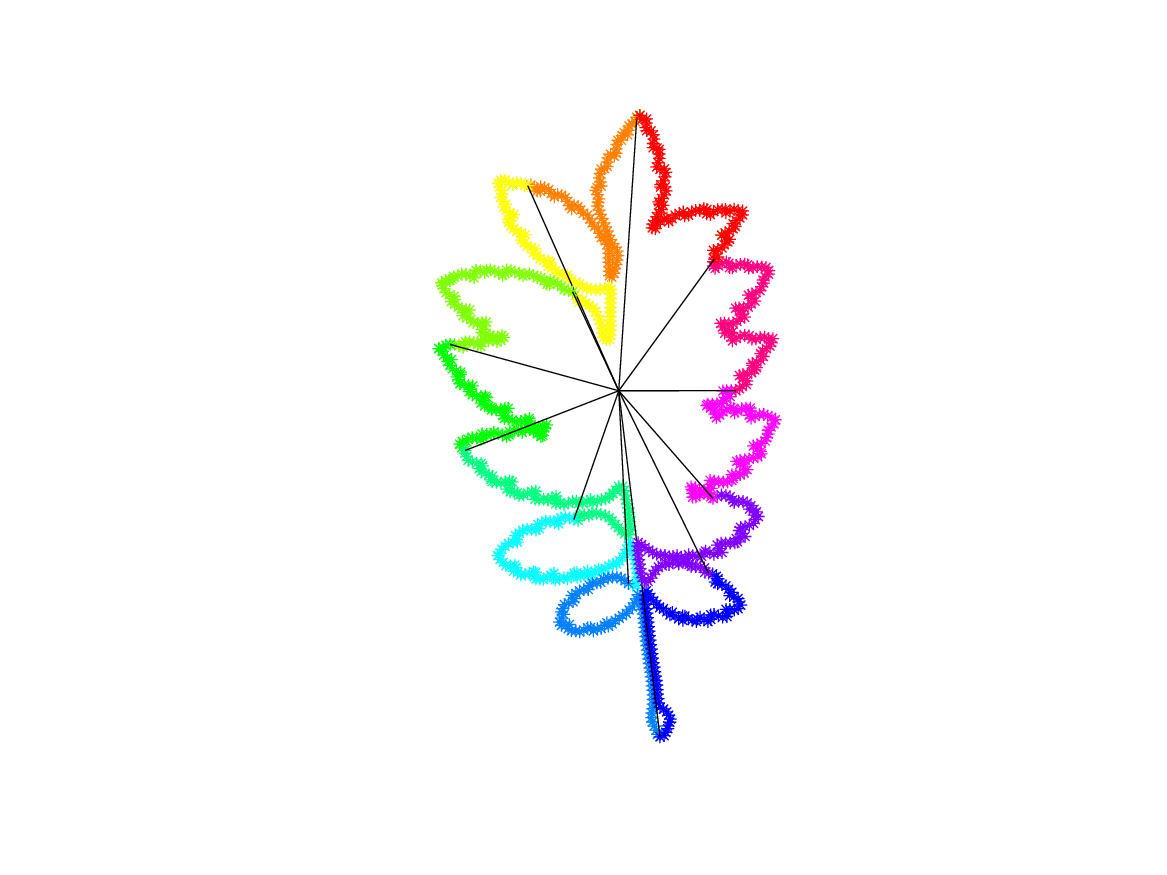}}
\subfloat[\centering]{\includegraphics[width=9.0cm]{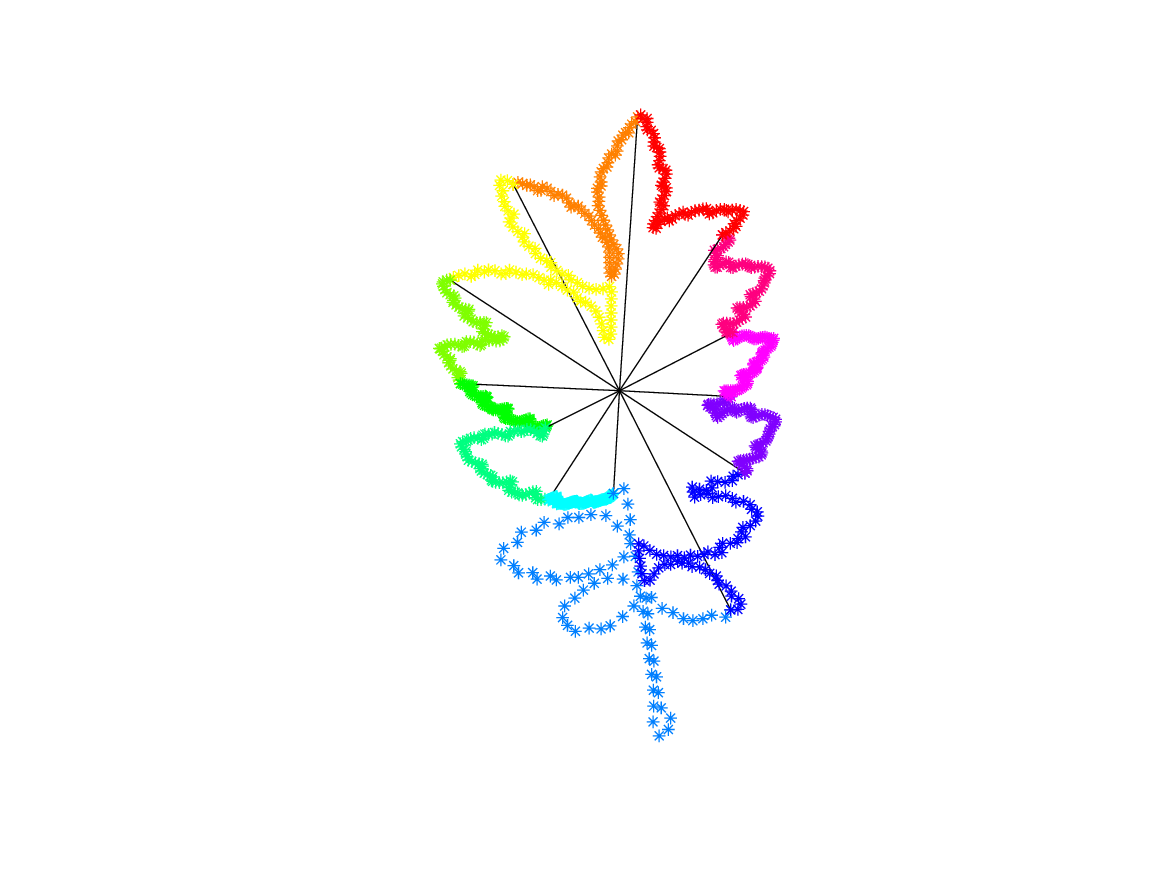}}\\

\caption{\textbf{Clock parameterization of Jordan curves}.  (\textbf{a}) A leaf of Acer is sampled uniformly with 720~points. Every 60 points the color is changed. The~angle between the first point of a colored portion, the~center of gravity, and the last point of the same colored portion is illustrated. These angles are not equal.  (\textbf{b}) 12 points are placed successively along the contour to form an angle of 360/12 with the center of gravity and the previous such point. Now the angles formed by each colored portion of the curve are the same. On~each colored portion, 60 points are distributed uniformly.  (\textbf{c}) The graph of the angle function of the Acer leaf is represented in the constant speed parameterization. (\textbf{d})  The graph of the angle function of the Sorbus leaf is represented in the constant speed parameterization. (\textbf{e}) A leaf of Sorbus is parameterized with constant speed and sampled with 720 points. (\textbf{f}) 12 keypoints are detected to form equal angles to the center of gravity and the portions of curve between them are resampled with 60 points. }
\label{fig_clock}
\end{figure}

\begin{figure}[h!]
\centering
\subfloat[\centering]{\includegraphics[width=6.0cm]{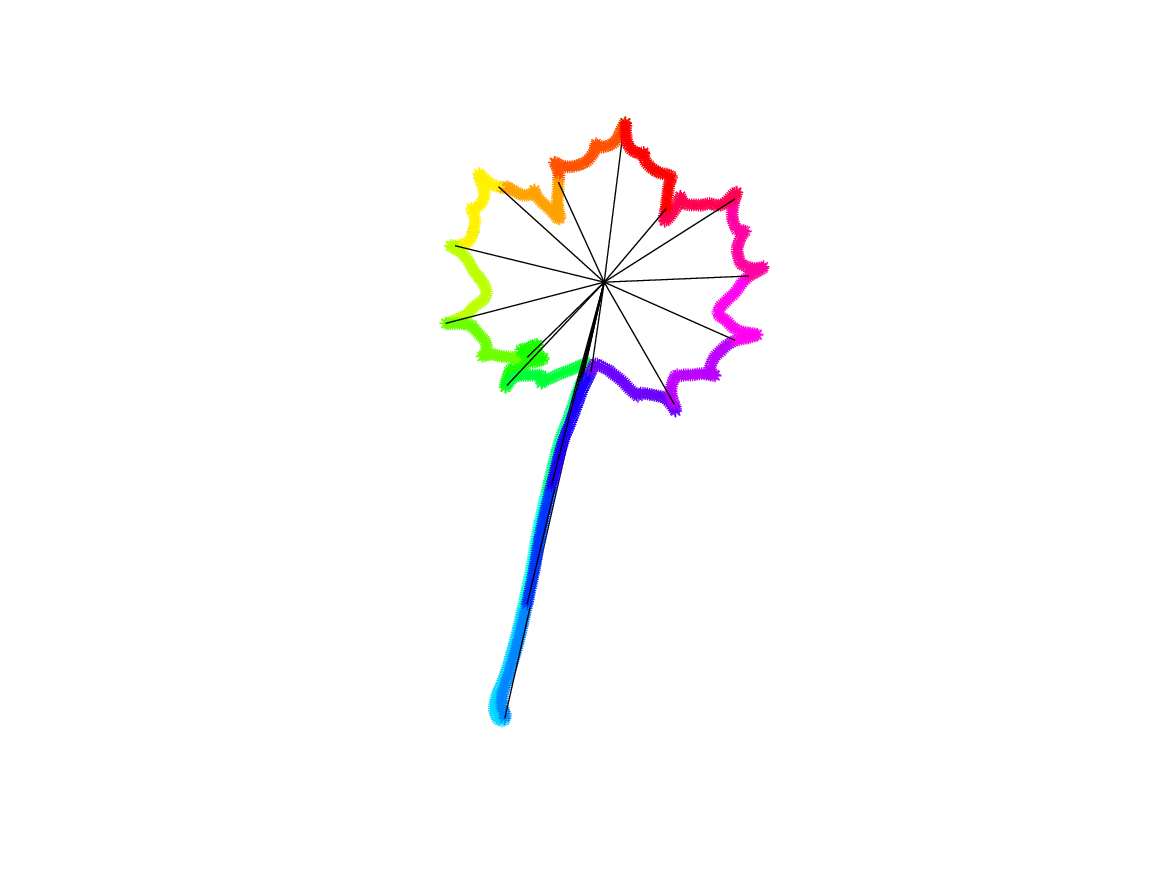}}
\subfloat[\centering]{\includegraphics[width=6.0cm]{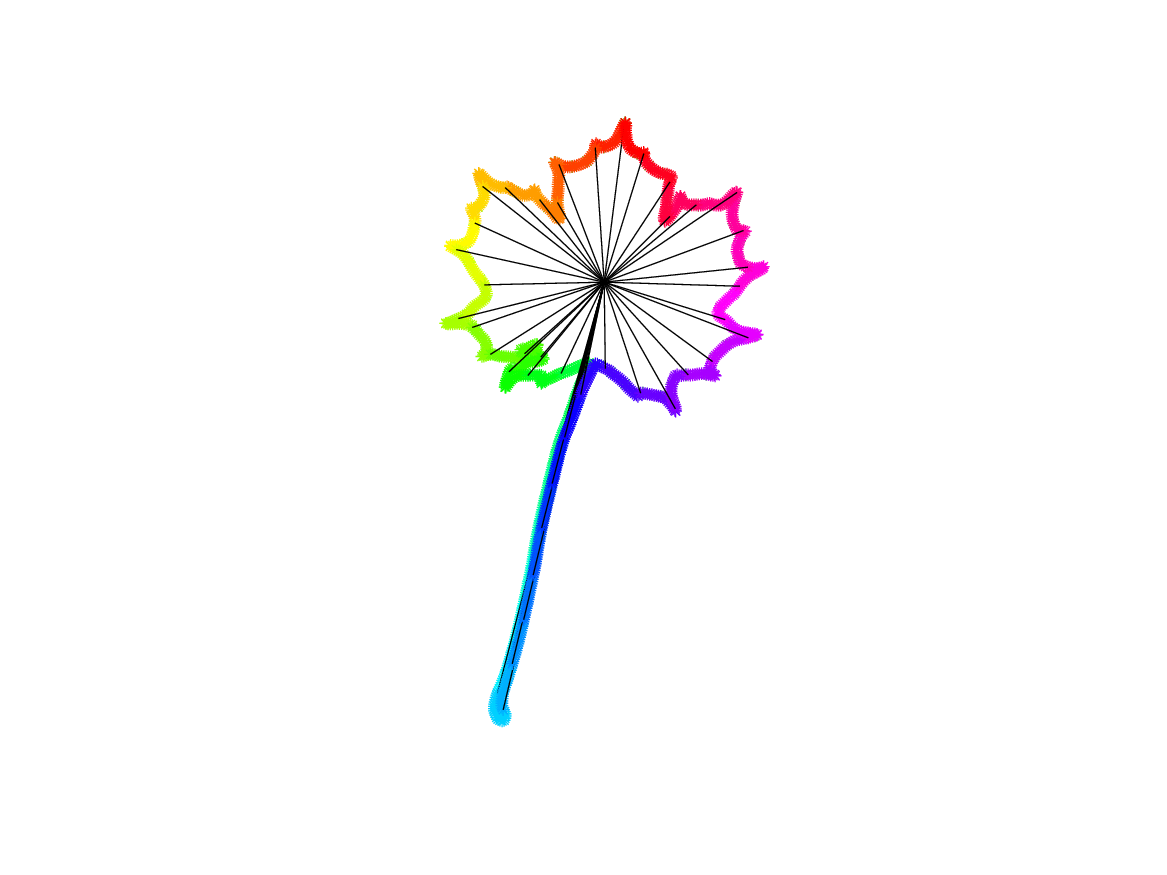}}
\subfloat[\centering]{\includegraphics[width=6.0cm]{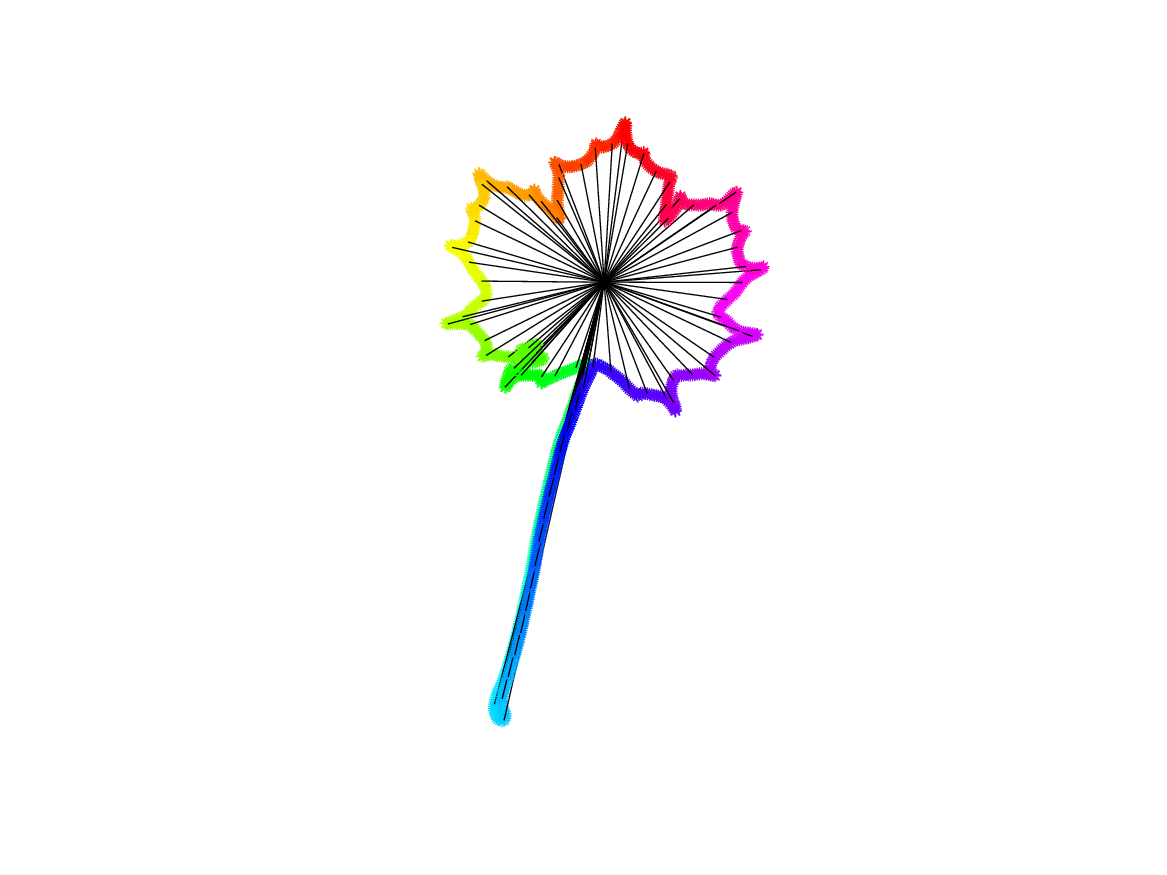}}\\
\subfloat[\centering]{\includegraphics[width=6.0cm]{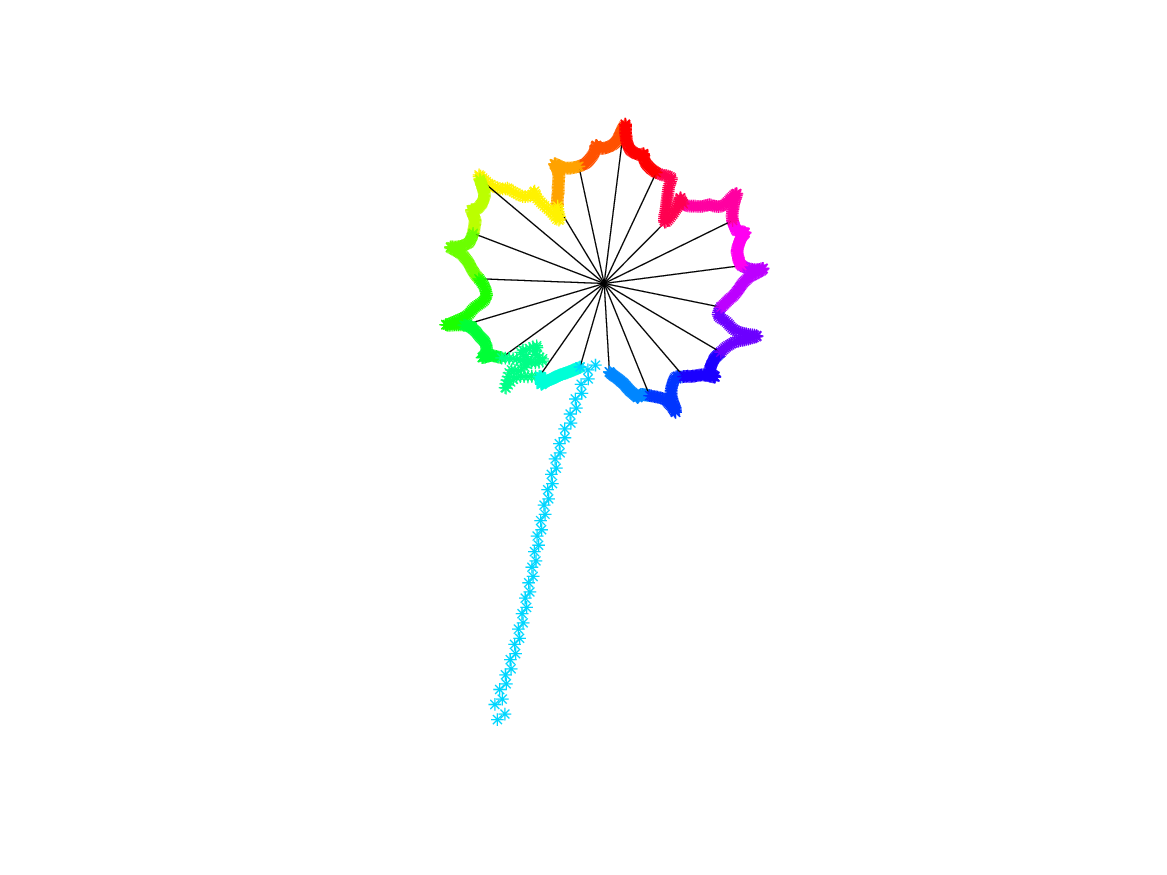}}
\subfloat[\centering]{\includegraphics[width=6.0cm]{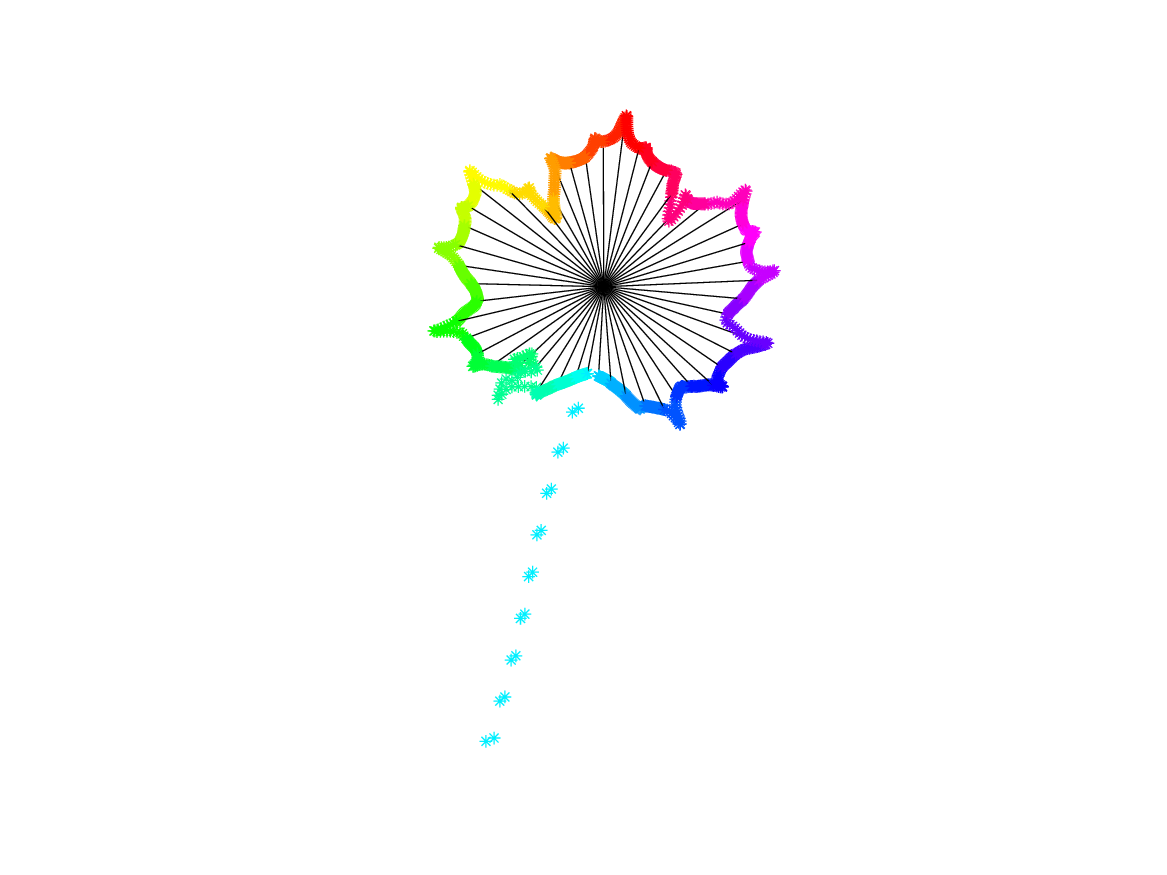}}
\subfloat[\centering]{\includegraphics[width=6.0cm]{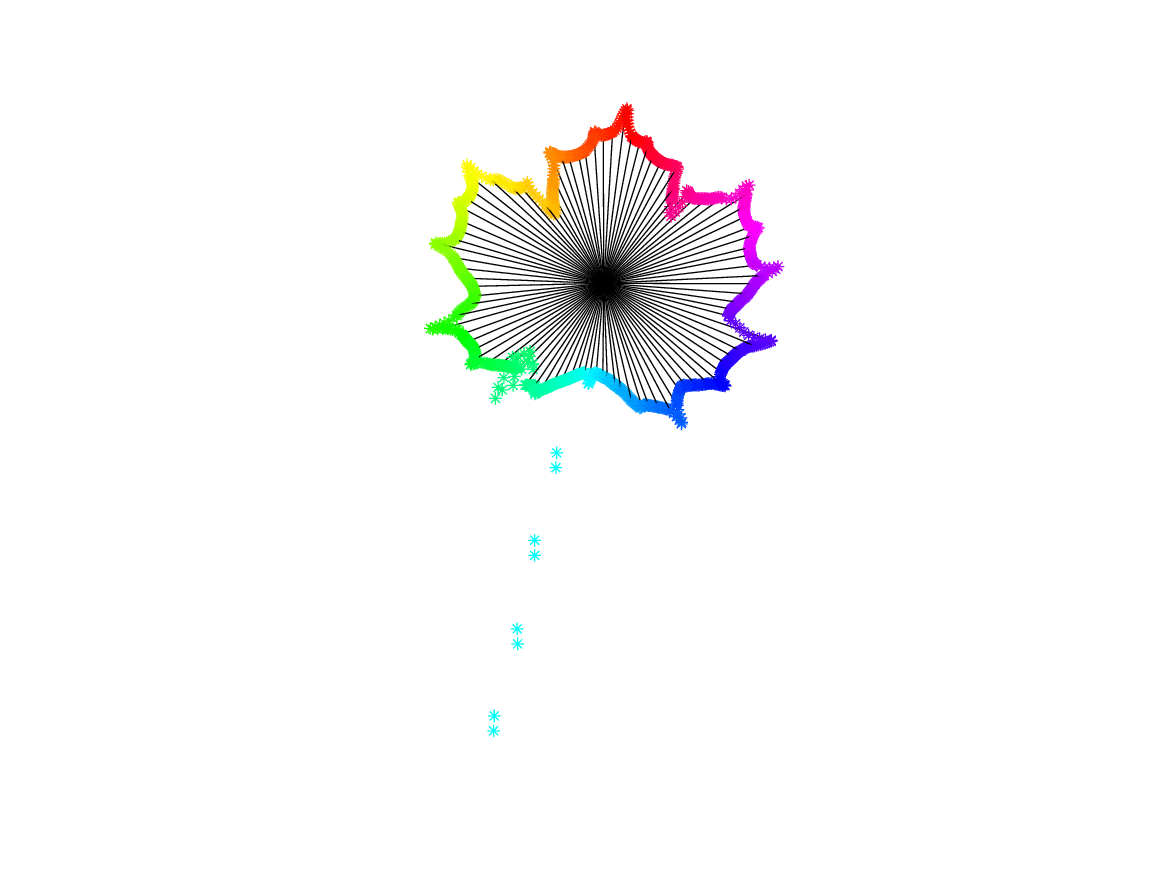}}\\
\subfloat[\centering]{\includegraphics[width=6.0cm]{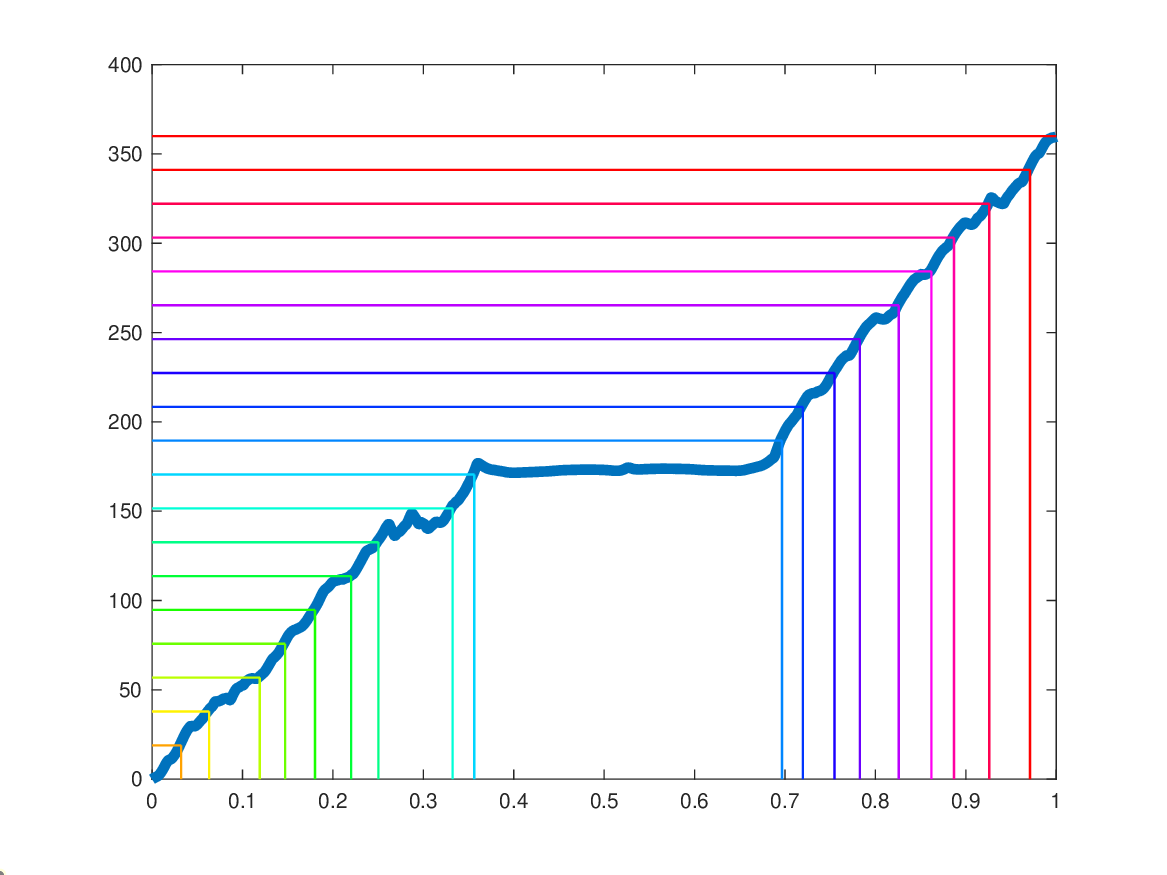}}
\subfloat[\centering]{\includegraphics[width=6.0cm]{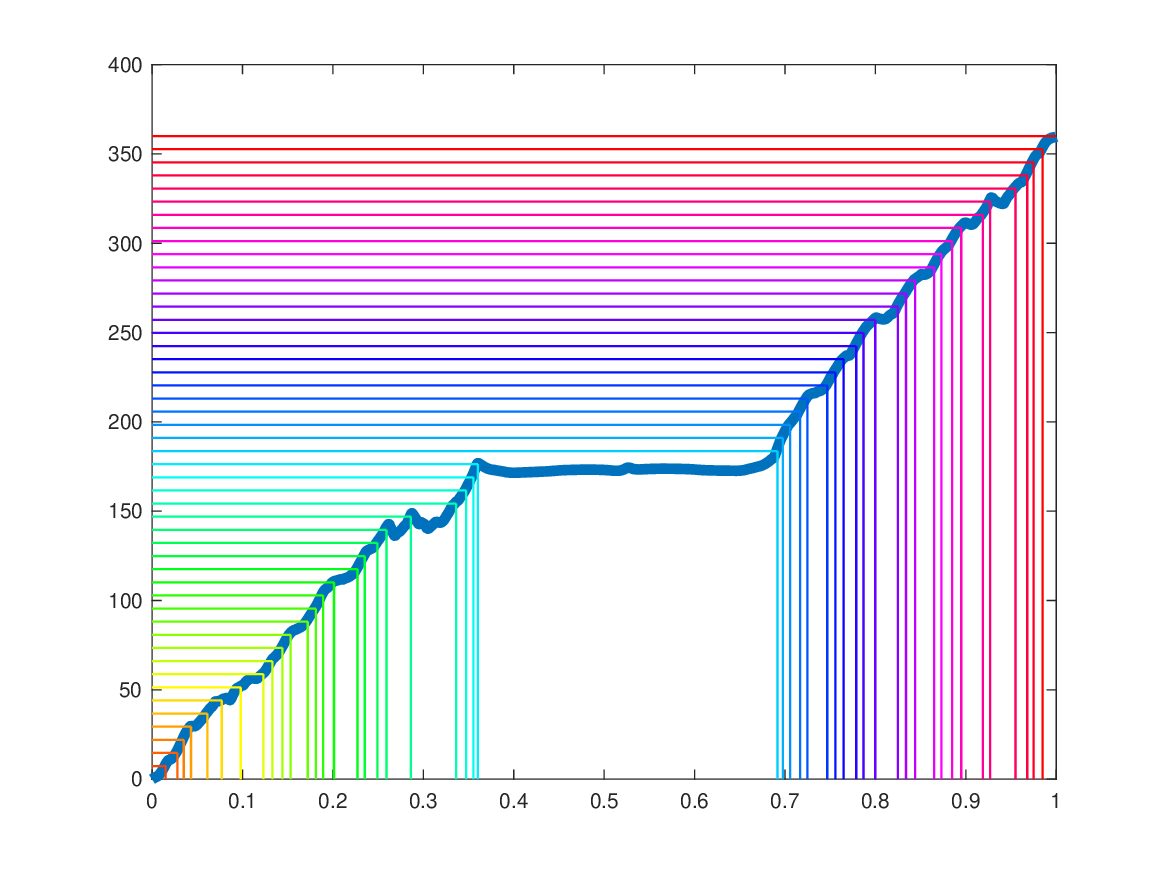}}
\subfloat[\centering]{\includegraphics[width=6.0cm]{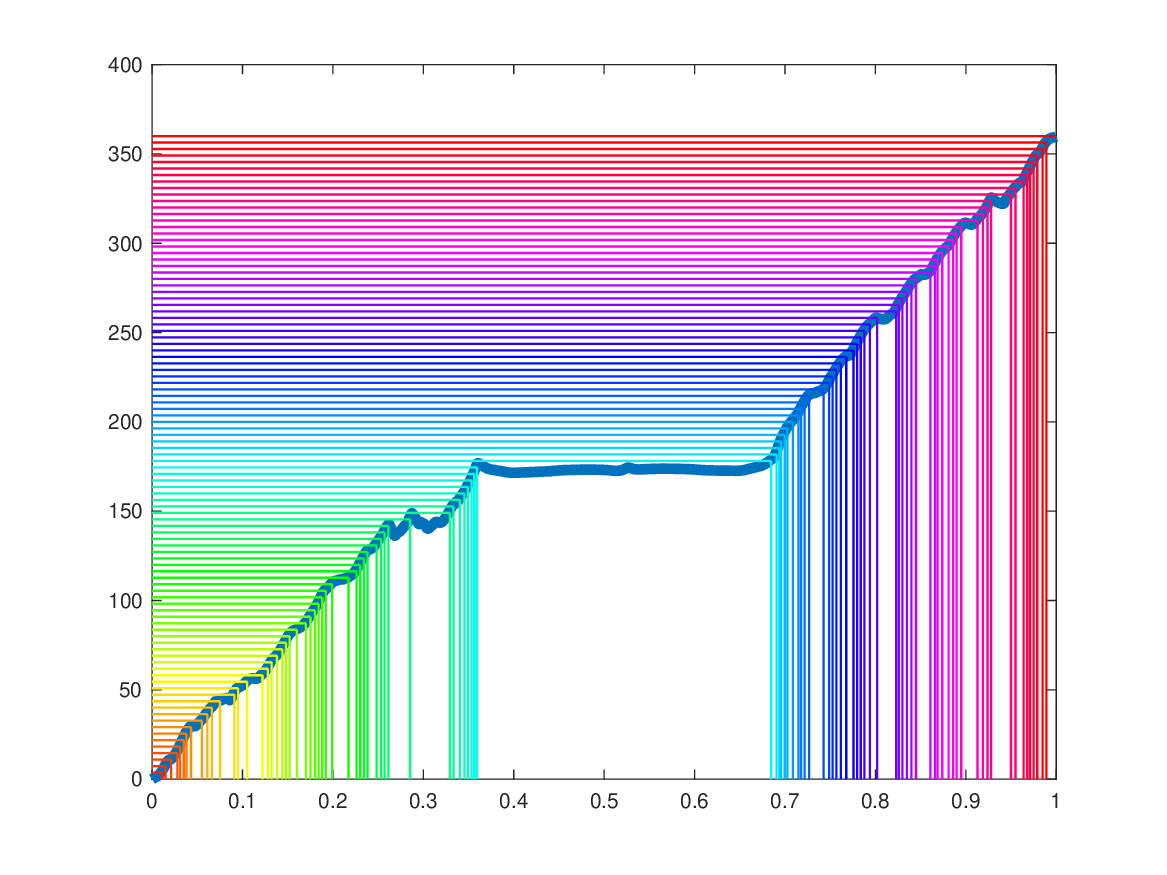}}
\caption{\textbf{{Dependence of the clock parameterization with respect to the number of subdivisions}}.   First row: An Acer leaf with peduncle is parameterized at constant speed and sampled with 1000~points. The~color is changed every (\textbf{a}) 20 points, (\textbf{b}) 50 points,  (\textbf{c}) 100~points. Second row: An Acer leaf with peduncle is parameterized with clock parameterization according to  (\textbf{d})~20 subdivisions, (\textbf{e})~50~subdivisions,   (\textbf{f}) 100 subdivisions. The~density of points along the peduncle decreases drastically with the number of subdivisions. Last row: The graph of the corresponding angle function is illustrated with  (\textbf{g}) 20 equally spaces horizontal lines, (\textbf{h})  50 equally spaces horizontal lines, (\textbf{i}) 100 equally spaces horizontal lines.}
\label{fig_clock_n}
\end{figure}

\subsubsection{Curvature-Weighted Clock Parameterizations of Jordan~Curves}\label{section_curvature_weighted_clock}

In this section, we introduce a $2$-parameter family of canonical parameterizations obtained by combining curvature-weighted parameterizations with parameter $\lambda$ (see \mbox{Section~\ref{sec_curvature_weighted}}) and clock parameterizations with $n$ subdivisions (Section~\ref{sec_just_clock}). More precisely, each contour is first decomposed into $n$ subdivisions forming $n$ equal angles to the center of gravity. Secondly, each portion of the curve is reparameterized according to a curvature-weighted parameterization with parameter $\lambda$ as in Section~\ref{sec_curvature_weighted}, Equation~\eqref{ulambda}. 

A sampling of a curve with $N$ points according to the curvature-weighted clock parameterization with parameters $(\lambda, n)$ goes as follows: the curve is subdivided into $n$ portions forming equal angles at the center of gravity; $N/n$ points are distributed on each portion according to the curvature-weighted parameterization with parameter $\lambda>0$; see Section~\ref{sec_curvature_weighted} (for low parameter $\lambda$, the~density of points decreases on flat parts of the contour and increases on curved parts, while for large parameter $\lambda$, the curvature-weighted parameterization tends to the constant speed parameterization).

In Figure~\ref{fig_curv_clock}, an~Acer leaf (with peduncle) is resampled with 1000 points according to curvature-weighted clock parameterizations with different parameters $(\lambda, n)$. The~first column corresponds to $\lambda = 0.3$, the~second to $\lambda = 1$, and~the third column to $\lambda = 2$. At~the same time, the~first row corresponds to  $n = 12$, the~second row to $n = 24$, and~the last row to $n = 36$. One can observe that the density of points along the peduncle decreases when $\lambda$ decreases and/or the number of subdivisions~increases.

\begin{figure}[h!]
\centering
\subfloat[\centering]{\includegraphics[height=5cm]{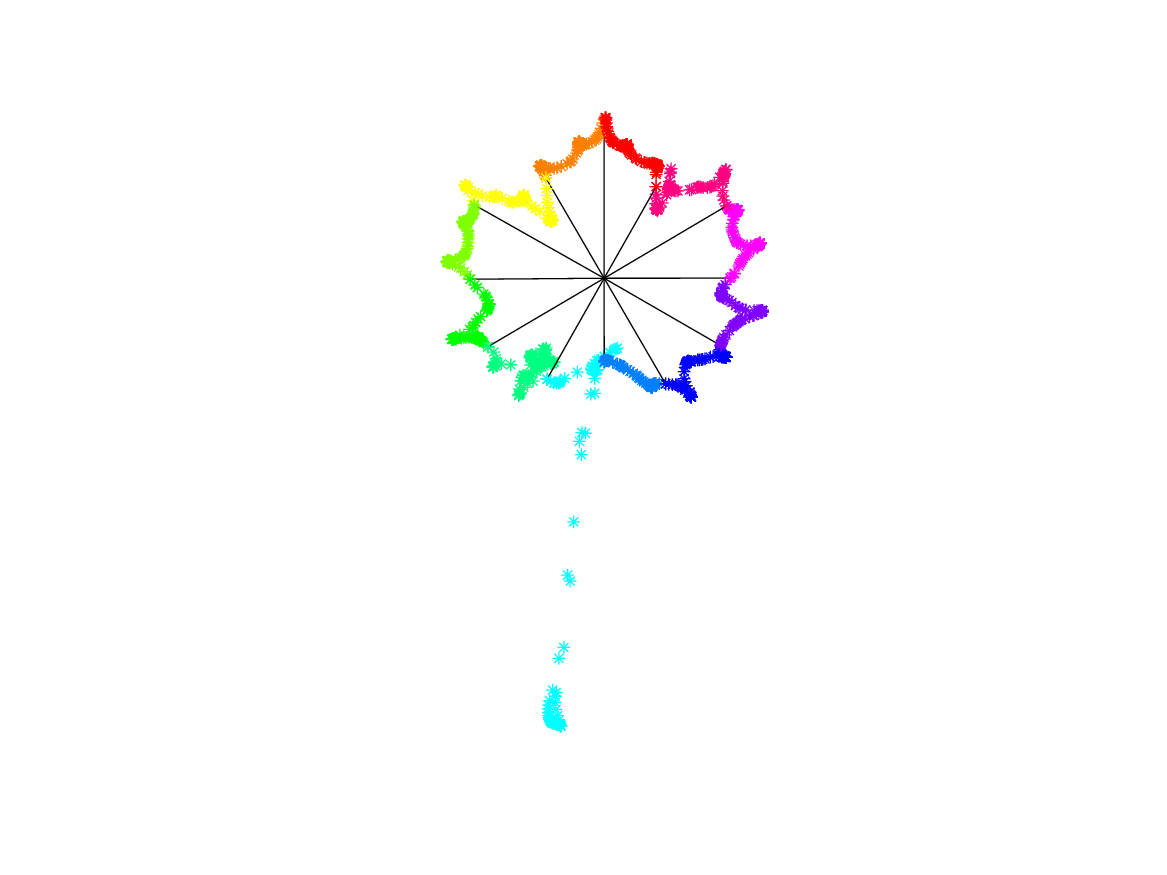}}
\subfloat[\centering]{\includegraphics[height=5cm]{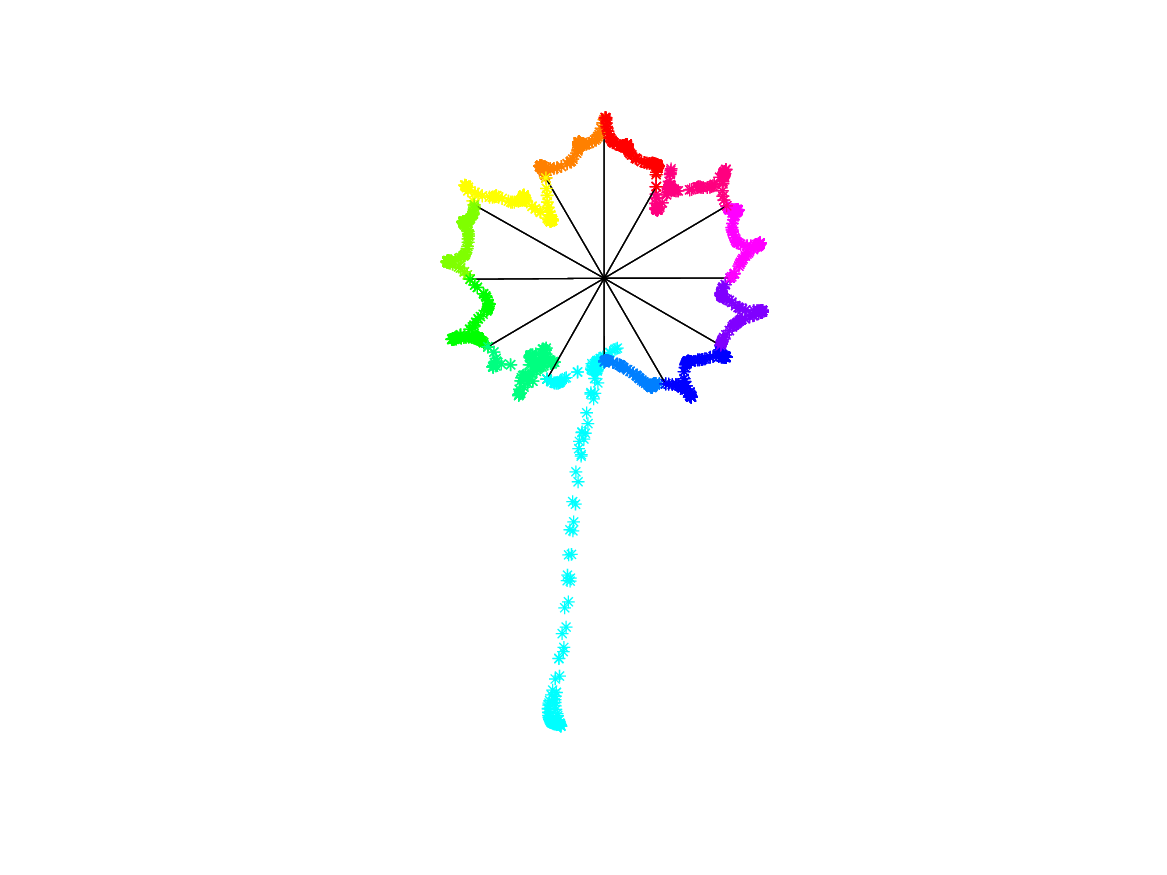}}
\subfloat[\centering]{\includegraphics[height=5cm]{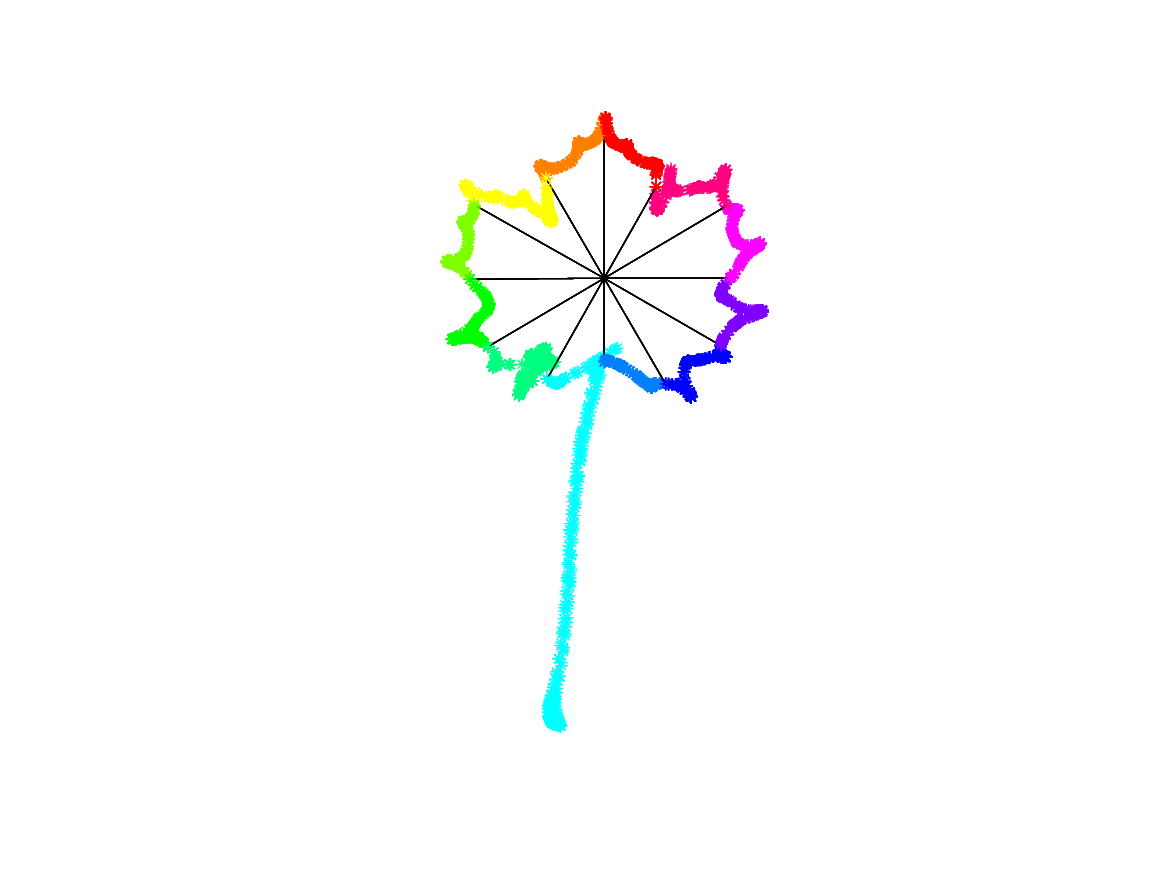}}\\
\subfloat[\centering]{\includegraphics[height=5cm]{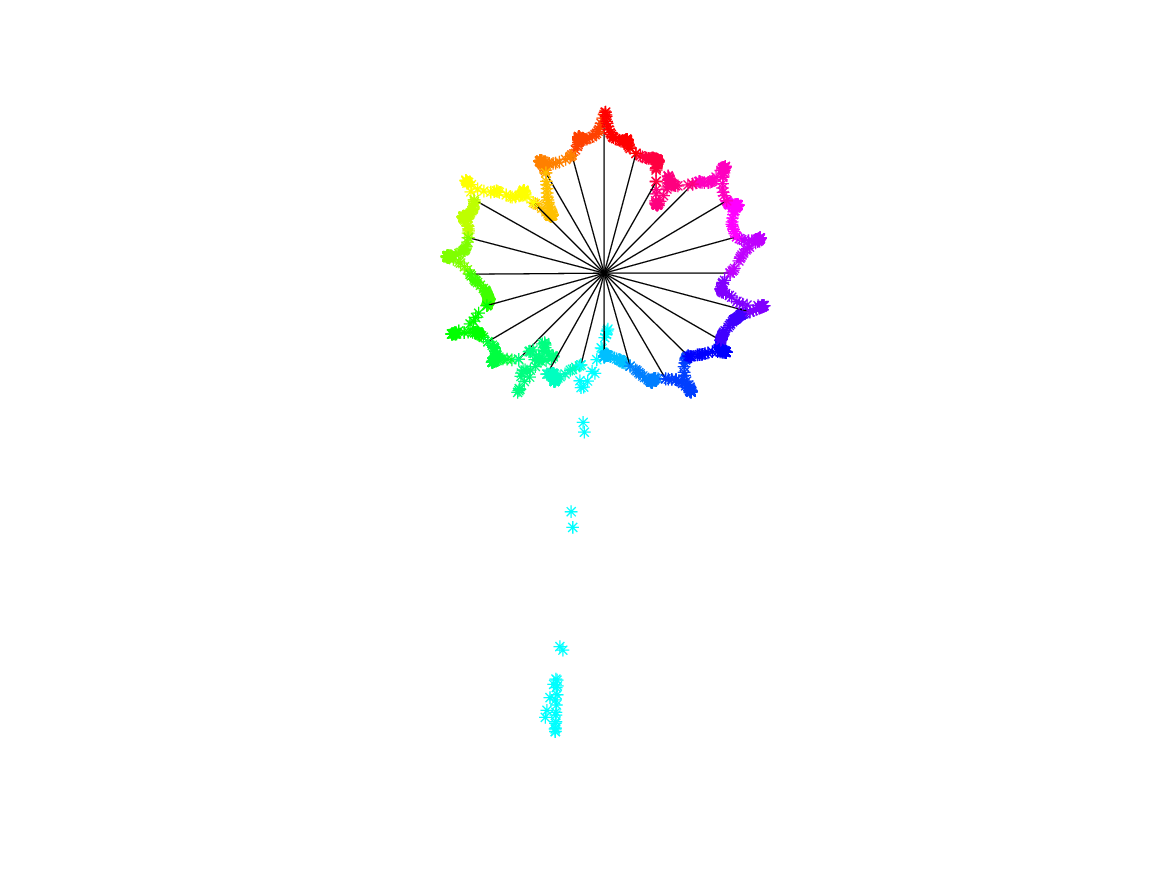}}
\subfloat[\centering]{\includegraphics[height=5cm]{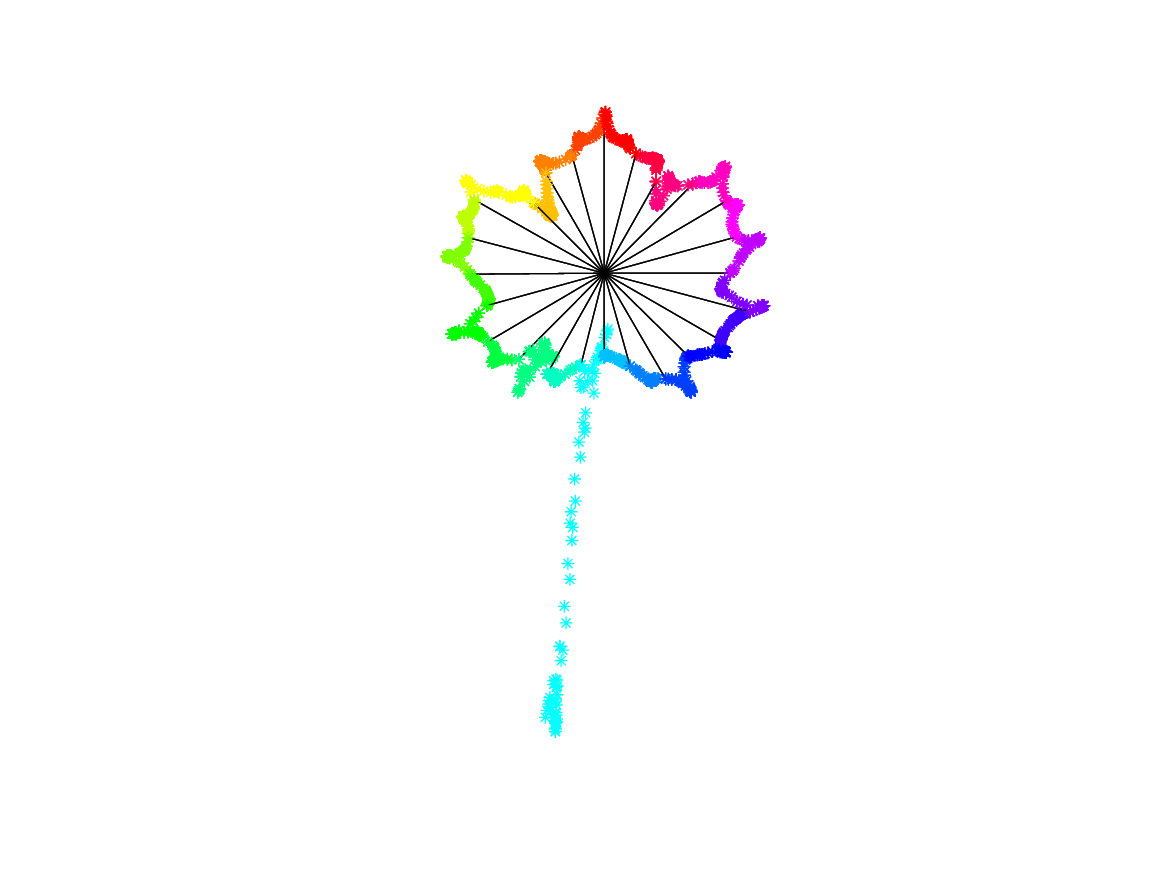}}
\subfloat[\centering]{\includegraphics[height=5cm]{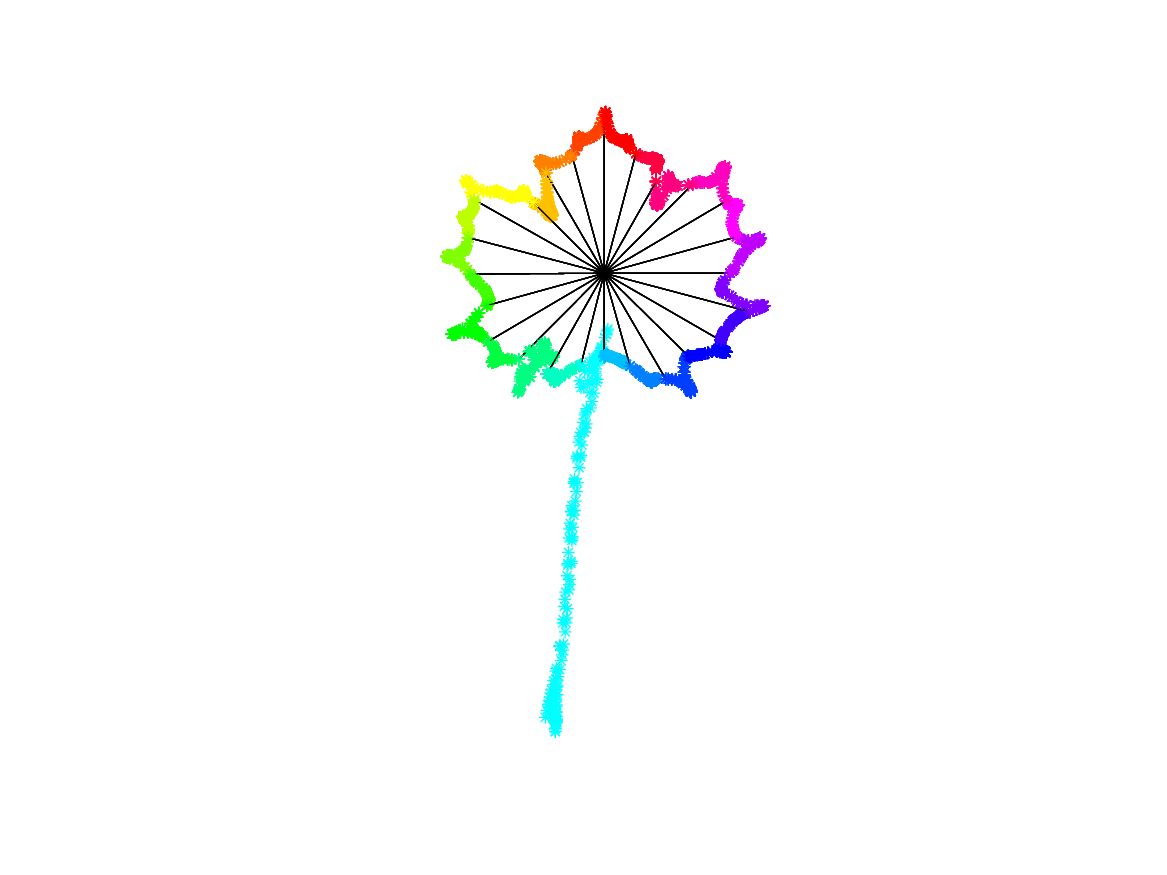}}\\
\subfloat[\centering]{\includegraphics[height=5cm]{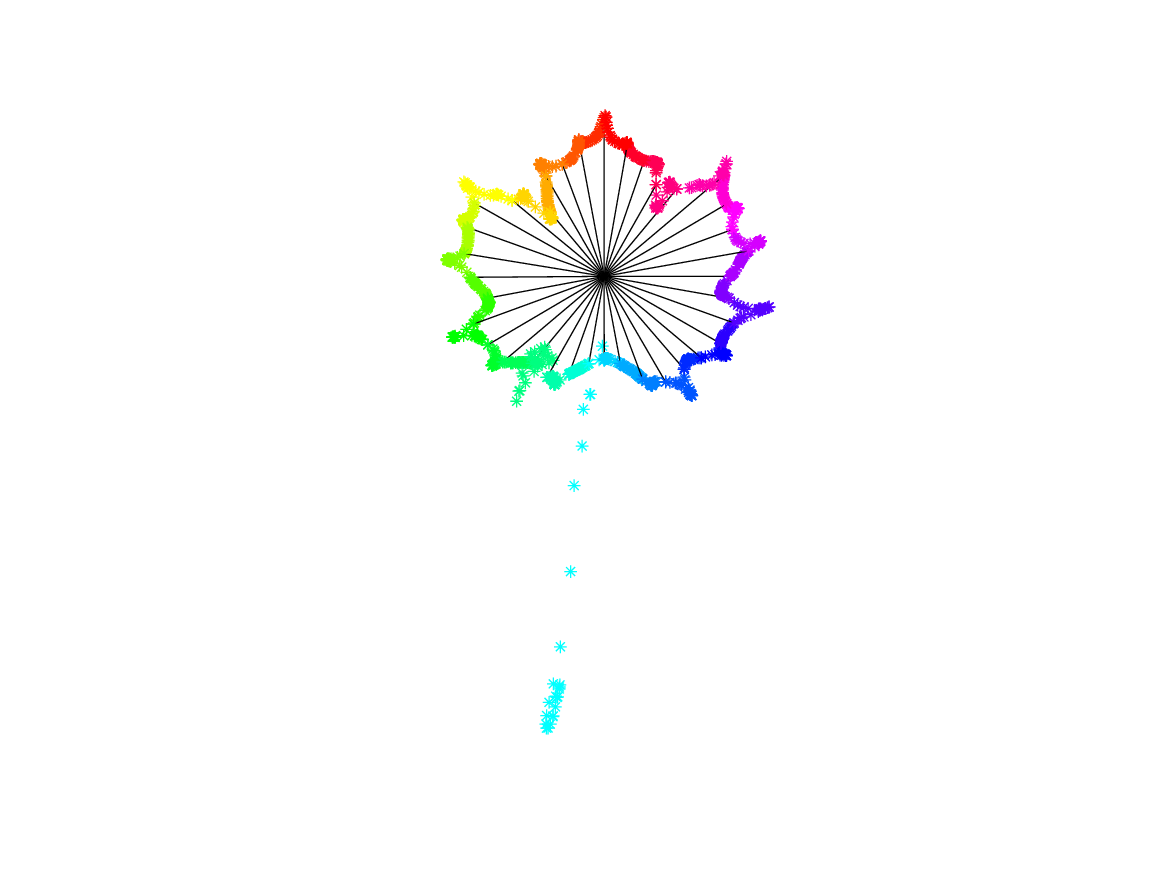}}
\subfloat[\centering]{\includegraphics[height=5cm]{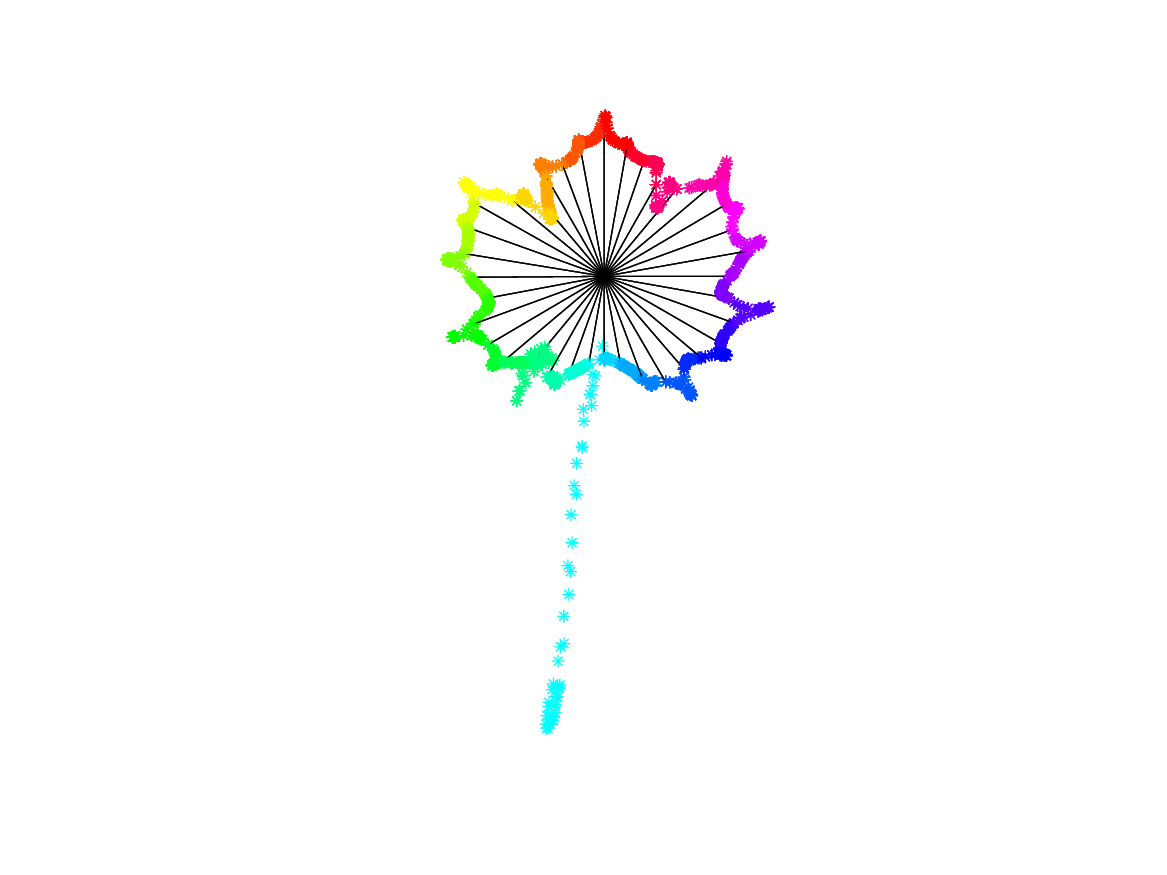}}
\subfloat[\centering]{\includegraphics[height=5cm]{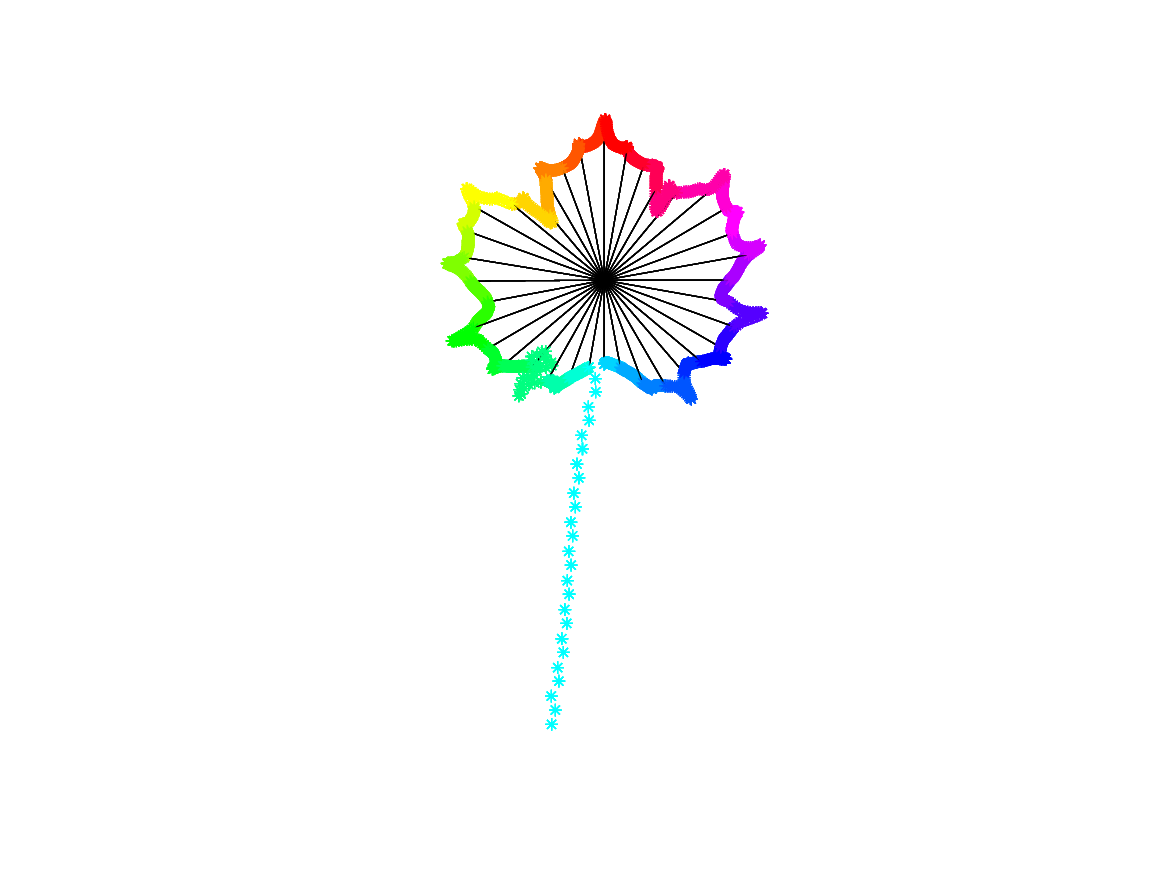}}\\

\caption{\textbf{Curvature-weighted clock parameterizations with different parameters}.   An~Acer leaf (with peduncle) is resampled with 1000 points according to curvature-weighted clock parameterizations with different parameters $(\lambda, n)$. The~first column (a,d,g) corresponds to $\lambda = 0.3$, the~second (b,e,h) to $\lambda = 1$, and~the third column (c,f,i) to $\lambda = 2$. At~the same time, the~first row (a,b,c) corresponds to  $n = 12$, the~second row (d,e,f) to $n = 24$, and~the last row (g,h,i) to $n = 36$. One can observe that the density of points along the peduncle decreases when $\lambda$ decrease and/or the number of subdivisions increases.}
\label{fig_curv_clock}
\end{figure}

\subsection{Geometric Learning of Canonical~Parameterizations}\label{section_geometric_learning}

In this section, we consider the $2$-parameter family of curvature-weighted clock parameterizations with parameters $(\lambda, n)$ introduced in Section~\ref{section_curvature_weighted_clock}, as well as~the corresponding sections $s_{\lambda, n}: \mathcal{P}/\mathcal{G} \rightarrow \mathcal{P}$ of the fiber bundle consisting of embedded closed curves.  The~aim is to optimize clustering based on the distance between shapes defined in \eqref{d2}, which depends on the section $s_{\lambda, n}$ chosen. 

A table containing the Dunn index for various values of parameters $\lambda$ (weighting the parameterization by curvature) and $n$ (number of segments in the clock parameterization) is given in Table~\ref{tab2}. {The Dunn index values were averaged over a 30-fold cross-validation.} For this experiment, we used the training set consisting of 15 classes of leaves with 50~leaves each. We can see in Table~\ref{tab2} that the largest Dunn index (corresponding to the best clustering for this metric) is obtained for $n = 3$ subsections along the contours of the leaves, {and $\lambda = 2000$, which corresponds to a curvature-weighted parameterization on each of the three portions of the curve. In~Figure~\ref{illustration-optimal_parameterization}a, we visualize the pair of curves that maximizes the intraclass distance, as~well as the pair of curves that minimizes the interclass distance. These two pairs of curves are responsible for the value of the Dunn index. We illustrate the segment in $L^2(\mathbb{S}^1, \mathbb{R}^2)$ ,which connects the leaves parameterized by the optimal parameterization ($n = 3$, $\lambda = 2000$).} In the left picture of Figure~\ref{illustration-optimal_parameterization}a, we can see that the south portion of the contour of the leaf without peduncle deforms to create a peduncle. In~comparison in the right picture of Figure~\ref{illustration-optimal_parameterization}a, two leaves from two different classes seem perfectly aligned. This illustrates the challenges of clustering or classifying this dataset of leaves, where very different shapes belong to the same class and similar shapes belong to different classes. Figure~\ref{optimal} illustrates  a $2$-dimensional representation of the distance distribution along the dataset using the  tsne algorithm before any normalization and after normalization using the optimal parameterization for the Dunn index. One can see that the classes are significantly better clustered after~normalization.

\begin{figure}[h!]
\centering
\subfloat[\centering]{\includegraphics[width=6cm]{images/etape0_flipudbis.eps}}
\subfloat[\centering]{\includegraphics[width=6cm]{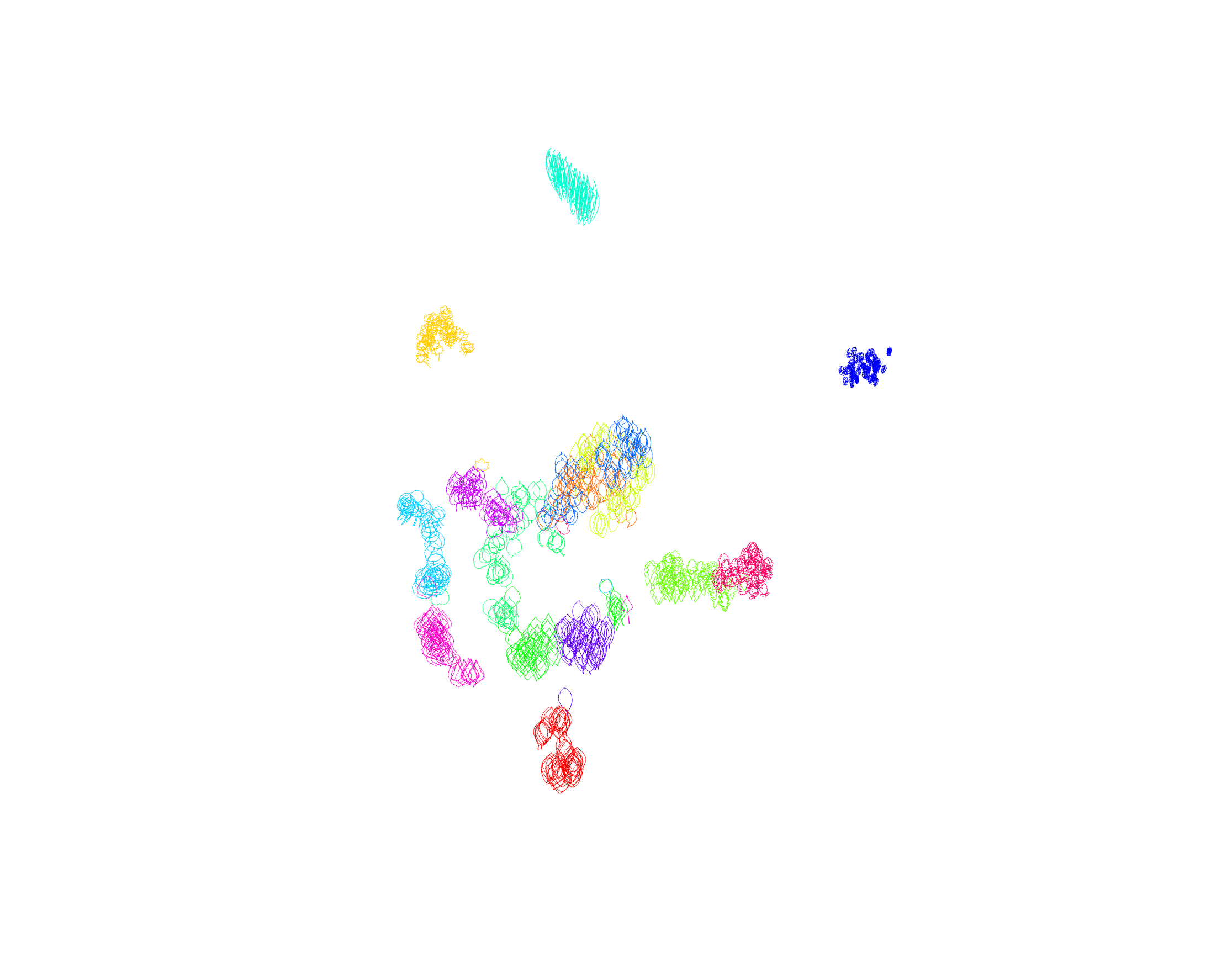}}\\
\subfloat[\centering]{\includegraphics[width=6cm]{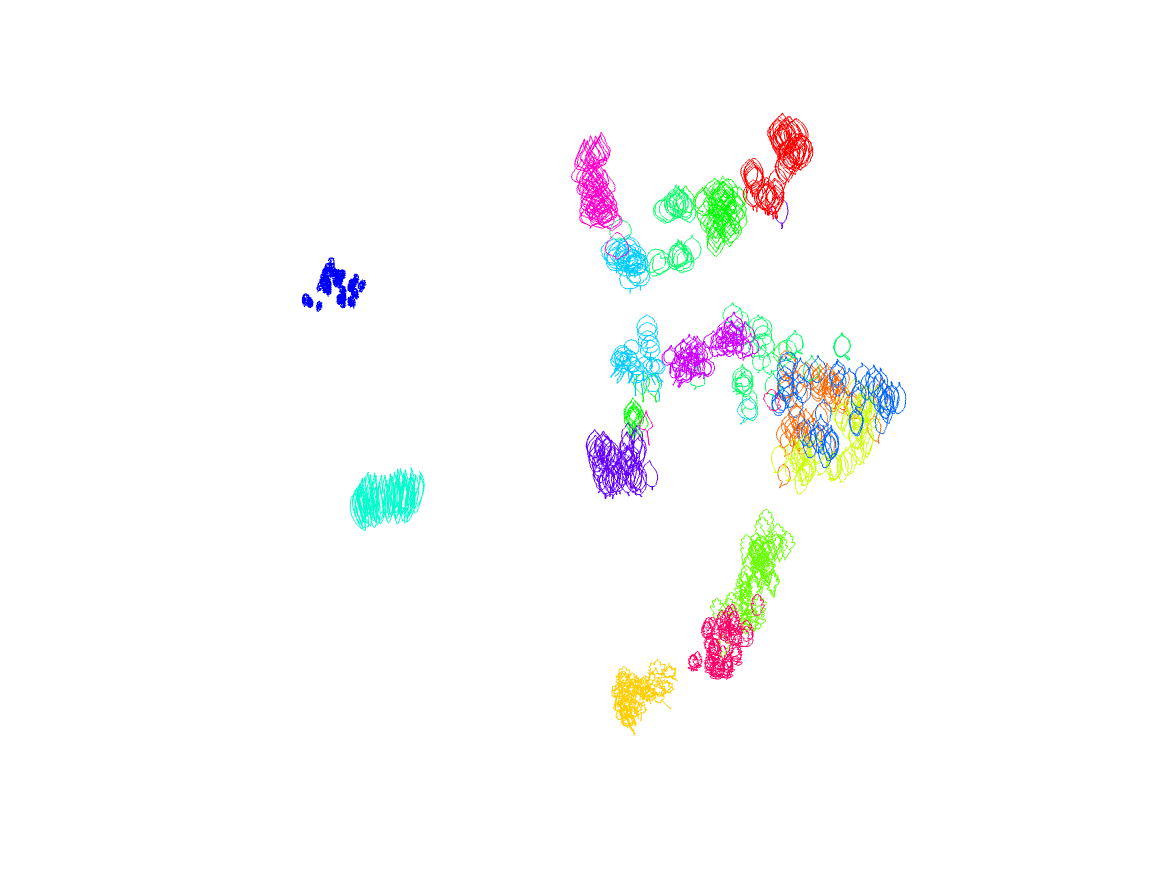}}\\

\caption{Two-dimensional representation of the distance distribution along the dataset using \texttt{tsne} algorithm (\textbf{a}) before any normalization;
(\textbf{b}) after normalization over finite-dimensional shape-preserving groups, as explained in Section~\ref{section_normalization}, and normalization over the infinite-dimensional group of orientation-preserving reparameterization by using the optimal parameterization for the Dunn index obtained as clock parameterization (see Section~\ref{sec_just_clock}) with $n = 3$ subdivisions and curvature-weighted parameterization ($\lambda = 2000$) on all portions of the subdivisions; 
(\textbf{c}) after normalization over finite-dimensional shape-preserving groups, as explained in Section~\ref{section_normalization}, and normalization over the infinite-dimensional group of orientation-preserving reparameterization by using the optimal parameterization for the Davies Bouldin index obtained as clock parameterization (see Section~\ref{sec_just_clock}) with $n = 5$ subdivisions and arc-length parameterization  ($\lambda = +\infty$) on all portions of the subdivisions.}
\label{optimal}
\end{figure}

\begin{table}[h!]
\caption{Dunn index for various clustering of the Swedish dataset based on clock parameterization with parameters $\lambda$ (weighting the parameterization by curvature) and $n$ (number of segments in the clock parameterization), cross-validated over 30 partitions of the dataset into training set and testing set. Each column corresponds to the parameter $\lambda$ given at the top of the column; each row corresponds to the values $n$ given on the left. A~larger Dunn index corresponds to a better clustering. The frame highlights the highest Dunn index.
\label{tab2}}
\centering
\begin{tabular}{cccccccccc}

\hline
			\boldmath{$\lambda =$} & \boldmath{$ 0.5$}  & \boldmath{$1$}  & \boldmath{$2$} & \boldmath{$5$}
            & \boldmath{$10$} & \boldmath{$100$} & \boldmath{$1000$} & \boldmath{$2000$} & \boldmath{$+\infty$}\\
\hline
n = 0  & 0.0535 & 0.0535 & 0.0536 &  0.0539 & 0.0543 & 0.0584 &  0.0670 & 0.0706        &  0.0774\\			  	                 
\hline
n = 2  & 0.0382 & 0.0382 & 0.0383 & 0.0384 & 0.0387	& 0.0432 & 	0.0522	& 0.0560        &  0.0639 \\
\hline
n= 3   & 0.0768 & 0.0768 & 0.0768 & 0.0769 & 0.0770	& 0.0785 & 	0.0821	& \fbox{0.0827}        &0.0826\\
\hline
n = 4  & 0.0656	& 0.0657 & 0.0657 & 0.0658 & 0.0659	& 0.0677 & 	0.0723	& 0.0735        & 0.0766\\
\hline
n = 5  & 0.0711 & 0.0711 & 0.0711 & 0.0711 & 0.0712	& 0.0720 & 	0.0748	& 0.0759        &  0.0780\\
\hline
n = 7  & 0.0703 & 0.0703 & 0.0703 & 0.0703 & 0.0704	& 0.0714 & 	0.0755	& 0.0772        &  0.0811\\ 
\hline
n = 9  & 0.0698 & 0.0698 & 0.0698 & 0.0699 & 0.0699	& 0.0707 & 	0.0739	& 0.0755        &  0.0798\\
\hline
n = 10 & 0.0672	& 0.0672 & 0.0672 & 0.0672 & 0.0673	& 0.0680 & 	0.0712	& 0.0727        &  0.0771\\
\hline              
n = 20 & 0.0642	& 0.0642 & 0.0642 & 0.0642 & 0.0643	& 0.0647 & 	0.0667	& 0.0677        &  0.0710\\
\hline
\end{tabular}
\end{table}


\section{Classification~Results}
\label{sec:class_results}
\unskip

\subsection{Testing on the Dataset of the Swedish~Leaves}\label{Testing}

\subsubsection{Clustering Evaluation Using Another Cluster Validation~Index}
Recall that the Swedish leaves dataset was divided into a training set, containing 50 contours from  each class, and~a testing set containing the remaining 25 contours per class. The~standardization steps performed in Sections~\ref{section_normalization} and \ref{section_geometric_learning} were necessary to define a distance between contours in the plane that is independent of their position and orientation in space, their scaling, and their parameterization (starting point, traveling direction, and~velocity). 
The best standardization procedures were selected to optimize the clustering of the classes of the labeled training set.  {The quality of the clustering obtained can be measured by computing a cluster validity index, like the Dunn index or the Davies Bouldin index, as explained in Section~\ref{validation_indices}, for~the distance defined in \eqref{d2}. \mbox{Tables~\ref{tab2} and~\ref{tabDB}}  contain, respectively, the values of the Dunn index and Davies Bouldin index, cross-validated over 30 partitions of the dataset into training set and testing set. As~mentioned in Section~\ref{Davies Bouldin}, the~averaging of the distances to a centroid over all elements of a class allows the Davies Bouldin index to be more stable than the Dunn index in the presence of outliers. We therefore select the Davies Bouldin index for classification task. }

\begin{table}[h!]
\caption{Davies Bouldin index for various clustering of the Swedish dataset based on clock parameterization with parameters $\lambda$ (weighting the parameterization by curvature) and $n$ (number of segments in the clock parameterization), cross-validated over 30 partitions of the dataset into training set and testing set. Each column corresponds to the parameter $\lambda$ given at the top of the column, each row corresponds to the values $n$ given on the left. A~lower Davies Bouldin index corresponds to a better~clustering. The lowest Davies-Bouldin index is highlighted with a frame.}
\label{tabDB}
\centering
\begin{tabular}{cccccccccc}
\hline
			\boldmath{$\lambda =$} & \boldmath{$ 0.5$}  & \boldmath{$1$}  & \boldmath{$2$} & \boldmath{$5$}
            & \boldmath{$10$} & \boldmath{$100$} & \boldmath{$1000$} & \boldmath{$2000$} & \boldmath{$+\infty$}\\
\hline
n = 0  & 2.7540 & 2.7524 & 2.7491 & 2.7394 & 2.7238 & 2.5351 &  2.2157  & 2.1575        &  2.0848\\			  	                 
\hline
n = 2  & 2.8729 & 2.8725 & 2.8717 & 2.8694 & 2.8657	& 2.8121 & 	2.6816	& 2.6337        &  2.5677 \\
\hline
n= 3   & 2.1682 & 2.1678 & 2.1671 & 2.1649 & 2.1614 & 2.1129 & 	1.9709	& 1.9301        &  1.8636\\
\hline
n = 4  & 2.4506	& 2.4505 & 2.4502 & 2.4495 & 2.4483	& 2.4317 &  2.3666	& 2.3343        &  2.2473\\
\hline
n = 5  & 2.0851 & 2.0848 & 2.0843 & 2.0827 & 2.0801	& 2.0446 & 	1.9508	& 1.9201        &\fbox{1.8574}\\
\hline
n = 7  & 2.0892 & 2.0890 & 2.0885 & 2.0870 & 2.0847	& 2.0539 & 	1.9715	& 1.9414        &  1.8707\\ 
\hline
n = 9  & 2.1272 & 2.1270 & 2.1265 & 2.1251 & 2.1229 & 2.0959 & 	2.0197	& 1.9886        &  1.9101\\
\hline
n = 10 & 2.2697 & 2.2695 & 2.2690 & 2.2677 & 2.2656	& 2.2397 & 	2.1573	& 2.1207        &  2.0176\\
\hline
n = 20 & 2.2241	& 2.2239 & 2.2236 & 2.2225 & 2.2207	& 2.2015 & 	2.1493	& 2.1258        &  2.0457\\
\hline
\end{tabular}
\end{table}

\subsubsection{Improvement of the Classification Results After~Normalization}\label{improvement}

{In the present section, we illustrate how standardization procedures affect classification performance of samples from the testing set. 
We have~used the following:
\begin{itemize}
    \item Logistic Regression with $L^2$-norm regularization;
    \item Random Forest with 400 trees;
    \item Support Vector Machine (SVM) with a non-linear Radial Basis Function (RBF) kernel;
    \item $k$-Nearest Neighbors (KNN) with $k = 5$ nearest neighbors.
\end{itemize}
Complete parameter specifications are available in the code. The~Support Vector Machine with an RBF kernel achieved the highest performance, with~$C=25$ (the regularization parameter controlling the trade-off between margin size and classification error) and $\gamma = 1.5$ (the kernel coefficient that determines the influence radius of individual training samples). 
To assess classification performance, we used accuracy as an evaluation metric, defined as {follows:}
\begin{equation}
\text{Accuracy} = \frac{1}{N} \sum_{i=1}^{N} \mathbf{1}\{\, Y_i = \hat{Y}_i \,\},
\end{equation}
where $Y_i$ is the true label of the $i$-th element in the testing set, and~$\hat{Y}_i$ is the corresponding predicted label.
}	

{
The results are displayed in Table~\ref{tab:testing}.
We see that all four classification algorithms perform significantly better after normalization, i.e.,~when a representative is chosen in each orbit of the shape-preserving groups in a consistent way. In~particular,  we observe an increase of 25,85\% of correct classifications for the KNN  algorithm between the first line of Table~\ref{tab:testing} (no normalization performed) and the last line (all finite and infinite-dimensional shape-preserving groups taken into account using optimized sections). This illustrates that including standardization of the representative of each orbits under shape-preserving groups in the pre-processing step improves classification results irrespective of the classification algorithms. 
}

\begin{table}[h!]
\caption{Classification results on the dataset of Swedish leaves in terms of average accuracy across pre-processing stages, with~different classifiers and over a 30-fold cross-validation. The~Dunn index and the Davies Bouldin indices are also reported, cross-validated over 30 partitions of the dataset into the training set, as well as~testing set and reported in the first column. We reparametrized the curve with 1000 points and, for~the clock parametrization, we set the number of subsections to $n=5$ and $\lambda = \infty$, which corresponds to arc-length parameterization on each portion of the curves and minimizes the Davies Bouldin index (Table~\ref{tabDB}). For~comparison, the~last line corresponds to the curvature-weighted clock parameterization with a number of subsections equal to $n = 5$ and the weight of the curvature equal to $\lambda=2000$. The best result in term of accuracy is highlighted with a frame.} \label{tab:testing}
\centering
\begin{tabular}{lcccccc}
\hline
\textbf{Pre-processing Steps} & \textbf{Dunn} & \textbf{DB} & \textbf{Logistic} & \textbf{RF} & \textbf{SVM} & \textbf{KNN} \\
\hline
No normalization \cref{section_dataset}                         & 0.0325 & 3.7981 & 0.7273 & 0.7489 & 0.8310 & 0.6713  \\
Std the travel direction \cref{section_counterclockwise}        & 0.0314 & 3.8766 & 0.7743 & 0.7483 & 0.8411 & 0.6800  \\
Std the starting point \cref{section_rot_parameter}             & 0.0367 & 3.2387 & 0.8604 & 0.7708 & 0.8829 & 0.6826 \\
Std the scale variability \cref{section_scaling}                & 0.0566 & 2.4960 & 0.9233 & 0.8754 & 0.9449 & 0.8913  \\
Std the position \cref{section_translation}                     & 0.0667 & 2.4239 & 0.9132 & 0.8902 & 0.9364 & 0.8892 \\
Std the orientation \cref{section_orientation}                  & 0.0774 & 2.0847 & 0.9192 & 0.8990 & 0.9496 & 0.9228 \\
Clock parametrization \cref{Section_clock_parameterization}     & 0.0759 & 1.9201 &  0.9395 & 0.9063 & \fbox{0.9602} 
& 0.9332 \\
Curvature-weighted \cref{section_curvature_weighted_clock}      & 0.0780 & 1.8574 & 0.9357 & 0.8992 & 0.9562 & 0.9284 \\
\hline
\end{tabular}
\end{table}

\subsubsection{Comparison with State-of-the-Art Classification~Results}

{Compared to the state-of-the art classification results
displayed in Table~\ref{tab:state_of_the_art} for classical machine learning algorithms (without Neural Networks) and in Table~\ref{tab:testingNN} for Neural Network-based algorithms, we observe that,  with~an optimization over only 2 parameters, our algorithm reaches 0.9602 accuracy (96.02\% of correct classifications) with SVM on the dataset of Swedish leaves, whereas the state-of-the art model VGG-16 needs 138 million parameters to reach perfect accuracy (100\% correct classifications) on the same dataset.  This illustrates that algorithms using fewer but well-chosen parameters can compete with brute force algorithms using millions of parameters.  We hope that this can motivate the investigation of more sustainable solutions for classifications tasks, as well as~meaningful parameter optimization. Moreover, as~shown in Section~\ref{improvement}, our proposed method could be a beneficial pre-precessing step before applying fine-tuned algorithms since it leads to an optimal point-to-point correspondence across the dataset.  Contrary to the other classification methods present in Tables~\ref{tab:state_of_the_art} and~\ref{tab:testingNN}, the~standardization procedure that we propose allows us to interpolate between elements in the dataset (as in Figure~\ref{illustration-optimal_parameterization}). It could be interesting to test whether the methods of~\cite{Xu2021, Alajlan} improve if we apply our standardization method as a pre-precessing step. 
}

\begin{table}[h!]
\caption{Comparison of classification results on the Swedish leaves dataset using different classical machine learning methods (no Neural Networks) taken from~\cite{Xu2021,Almodfer}. We see that with an optimization over only two parameters, our method is comparable to the state-of-the art classical machine learning algorithms. Moreover, it could serve as a pre-processing step for more complex~algorithms.}\label{tab:state_of_the_art}
\centering
\begin{tabular}{lc}
\hline
\textbf{Methods} & \textbf{Classification Rate (\%)}  \\
\hline
Multi-features fusion~\cite{Rojas}              & 77.24   \\
MSRA         & 91.87  \\
MARCH   & 93.20  \\
MDM~\cite{Rongxiang}        & 93.60  \\
IDSC~\cite{Ling}    & 94.13\\
MCC~\cite{Adamek}         & 94.75 \\
SPTC~\cite{Ling}           & 95.33  \\
TAR~\cite{Alajlan}  & 95.97  \\
OURS & 96.02 \\
MSSD~\cite{Xu2021} & 96.85 \\
\hline
\end{tabular}
\end{table}

\begin{table}[h!]
\centering
\caption{Comparison of classification results on the Swedish leaves dataset using Neural Networks methods; table taken from~\cite{Li}. Note that the state-of-the art model VGG-16 needs 138 million parameters to reach perfect accuracy, whereas our method achieves similar accuracy with an optimization over only two geometrically explainable parameters.}
\label{tab:testingNN}
\begin{tabular}{lc}
\hline
\textbf{Methods} & \textbf{Classification Rate (\%)}  \\
\hline
AlexNet         & 99.70  \\
GoogleLeNet   & 99.39  \\
VGG16        & 100.00  \\
ResNet18      & 99.39\\
ResNet50          & 99.39 \\
ResNet101            & 99.70  \\
\hline
\end{tabular}
\end{table}

\subsection{Testing on Flavia~Dataset}\label{sec:Flavia}

{To further assess the effectiveness of the proposed pipeline, we evaluated it on a second dataset. We use the Flavia dataset, which contains 1,907 leaf images belonging to 32 classes and is available at \href{https://www.kaggle.com/datasets/gauravneupane/flavia-dataset}{https://www.kaggle.com/datasets/gauravneupane/flavia-dataset} (accessed on 18 November 2025). Figure~\ref{Flavia} illustrates the different types of leaves present in this dataset. Achieving high classification accuracy on this dataset is more challenging due to the larger number of classes and the extremely similar shapes among many of them. 
}

{While applying our pipeline on the Flavia dataset, we were surprised to see that normalization of the orientation in space deteriorated the clustering drastically. After~taking a closer look at the dataset, we discovered that the original Flavia dataset contains an alignment bias. Indeed, for~some classes, all the leaves are oriented in a class-dependent direction in space. In~Figure~\ref{angle_dist}, the~angle distribution of the leaves in each class is depicted. As~we can see, for~instance on classes 15, 19, and 32, the~distribution is very concentrated around a mean orientation, and~this mean orientation differs from class to class. This is probably due to the way the dataset was collected. Consequently, the~orientation in space can be used to determine the belonging of a sample to a given class, which is unfortunate. In~order to test our algorithm on an unbiased dataset, we applied random rotations to the samples of the dataset. The~unbiased dataset is available at the following links: \href{https://github.com/GiLonga/Geometric-Learning}{https://github.com/GiLonga/Geometric-Learning} and \href{https://github.com/ioanaciuclea/geometric-learning-notebook}{https://github.com/ioanaciuclea/geometric-learning-notebook}.
}

\begin{figure}[h!]
\centering
\includegraphics[width = 0.9\linewidth]{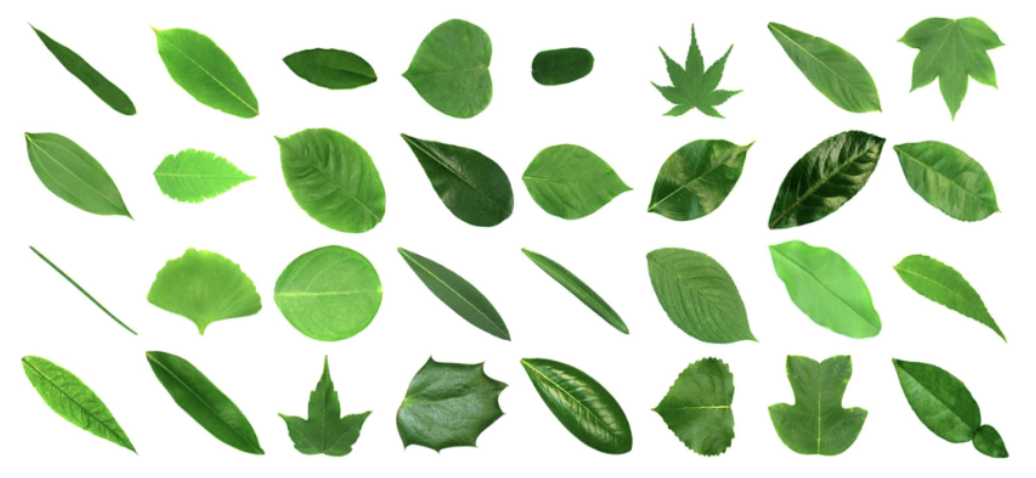}
\caption{Sample of the Flavia leaf dataset. Picture taken from~\cite{Almodfer}.}\label{Flavia}
\end{figure}

\begin{figure}[h!]
\centering
\includegraphics[width = \linewidth, trim= 3.5cm 3cm 2cm 2cm,clip=true]{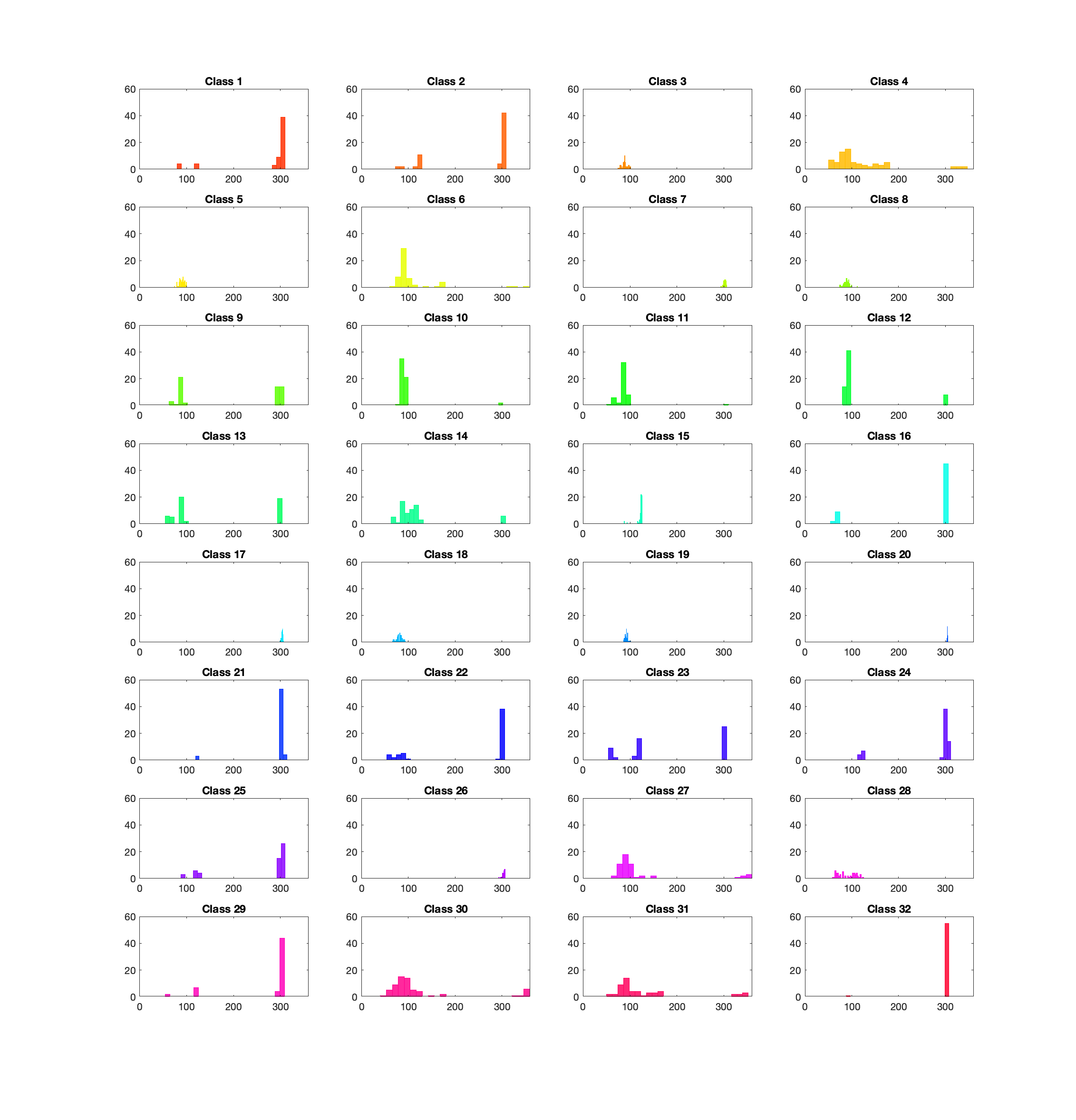}

\caption{Distribution of orientation angles in each class of the Flavia dataset of leaves. We see that for some classes, all the leaves are oriented in the same direction in space, with~a mean orientation differing from class to class, like for the classes 15, 19, and 32, for example. This implies that the dataset is biased with respect to~orientation.}\label{angle_dist}
\end{figure}

{
Starting from the unbiased Flavia dataset, we can see in Table~\ref{tab:testing_Flavia} that our normalization procedure improves the classification performance of all the algorithms tested. Since this dataset contains very similar shapes but with different scales, scale normalization was not performed because the scale contains valuable information in order to distinguish between classes.  In~order to optimize over the parameterization, we have used the Davies Bouldin index, which is more stable than the Dunn index in the presence of outliers.  Table~\ref{tab:DB_Flavia} contains the Davies Bouldin index for different values of the parameters, 30-fold cross-validated. 
}

As a concluding remark, let us note that the optimal normalization procedure and the optimal parameters $(n, \lambda)$ depend on the dataset and the selected cluster validity index. However, for~a given dataset, the~optimal parameterizations $s_{n, \lambda}$  for various cluster validity indices seem to be close in the space of sections over the sample points. Indeed, as~can be seen in Figure~\ref{illustration-optimal_parameterization} for the Swedish leaves dataset, the~optimal section $s_{n, \lambda}$ for the Dunn index is different from the optimal section for the Davies Bouldin index (the former is associated with $(n = 3, \lambda = 2000)$ whereas the latter with $(n = 5, \lambda = +\infty)$), but~the corresponding contours parameterizations look very similar.
This can be explained by the fact that the various cluster validity indices are linked to each other~\cite{Bezdek} and are continuous functions of the distances between samples, while these distances depend continuously on the section $s_{n, \lambda}$. From~this perspective, it becomes clear that the standardization procedure improves classification performance, as~the optimal distance function better reflects the intrinsic geometry of the dataset.

\begin{table}[h!]
\centering
\caption{Classification results in terms of average accuracy across pre-processing stages, with~different classifiers and over a 30-fold cross-validation. The~Dunn and Davies Bouldin indices at each step are also reported. We reparametrized the curve with 1000 points and, for~the clock parametrization, we set the number of subsections to $n=3$. For~the curvature-weighted clock parameterization (last line) we set $n = 3$  and $\lambda = 1000$ which corresponds to the curvature-weighted clock parameterization that minimizes the Davies Bouldin index (Table~\ref{tab:DB_Flavia}). The best result in term of accuracy is highlighted with a frame. 
}\label{tab:testing_Flavia}
\centering
\begin{tabular}{lcccccc}
\hline
\textbf{Pre-processing Steps} & \textbf{Dunn} & \textbf{DB} & \textbf{Logistic} & \textbf{RF} & \textbf{SVM} & \textbf{KNN} \\
\hline
No normalization \cref{section_dataset}                         & 0.0137 & 37.8766 & 0.0317  & 0.0392 & 0.0319 & 0.0331  \\
Std the travel direction \cref{section_counterclockwise}        &  0.0096 & 44.8592 & 0.1333  & 0.3188 & 0.6392  & 0.2192 \\

Std the position \cref{section_translation}                     & 0.0103 & 38.8019 & 0.0905 &  0.5311 & 0.6916 &  0.3586 \\
Std the orientation \cref{section_orientation}                  & 0.0070 & 48.3259 & 0.1075 & 0.7002 & 0.7100 & 0.5649 \\

Std the starting point \cref{section_rot_parameter}             & 0.0132 & 20.0387 & 0.4514 & 0.6464 & 0.6782  &0.5798 \\     
Clock parametrization \cref{Section_clock_parameterization}     & 0.0188 & 4.7575 & 0.6959 & 0.7361 & 0.7565 & 0.6756 \\
Curvature-weighted \cref{section_curvature_weighted_clock}      & 0.0203 & 4.5763 & 0.6531 & 0.7384 & \fbox{0.7679} 
& 0.6759 \\
\hline

\end{tabular}
\end{table}

\begin{table}[h!]
\caption{Davies Bouldin index for various clustering of the Flavia dataset based on clock parameterization with parameters $\lambda$ (weighting the parameterization by curvature) and $n$ (number of segments in the clock parameterization), over~a 30-fold cross-validation. Each column corresponds to the parameter $\lambda$ given at the top of the column, each row corresponds to the values $n$ given on the left. A~lower Davies Bouldin index corresponds to a better~clustering. The lowest Davies-Bouldin index is highlighted with a frame.}
\label{tab:DB_Flavia}
\centering
\begin{tabular}{cccccccccc}
\hline
			\boldmath{$\lambda =$} & \boldmath{$ 0.5$}  & \boldmath{$1$}  & \boldmath{$2$} & \boldmath{$5$}
            & \boldmath{$10$} & \boldmath{$100$} & \boldmath{$1000$} & \boldmath{$2000$} & \boldmath{$+\infty$}\\
\hline
n = 0  & 6.5694 & 6.5654 & 6.5576 & 6.5347 & 6.4983 & 6.0650 & 5.2734 & 5.1864          & 5.1206\\			  	                 
\hline
n = 2  & 6.0932 & 6.0911 & 6.0869 & 6.0748 & 6.0559	& 5.8459 & 5.5693 & 5.5721          & 5.6204 \\
\hline
n= 3   & 5.6204 & 4.8815 & 4.8790 & 4.8741 & 4.8662	& 4.7501 & \fbox{4.5763} & 4.6257          & 4.7575\\
\hline
n = 4  &  4.8376 & 4.8375 & 4.8373 & 4.8367 & 4.8360 & 4.8297 & 4.8644 & 4.9224
          & 5.0768\\
\hline
n = 5  & 4.7432 & 4.7425 & 4.7410 & 4.7371 & 4.7308 & 4.6594 & 4.6479 & 4.7174          & 4.8260\\
\hline
n = 7  & 4.7803 & 4.7801 & 4.7797 & 4.7788 & 4.7773	& 4.7739 & 4.9330 & 5.0198          & 5.0973\\ 
\hline
n = 9  & 4.7388 & 4.7388 & 4.7388 & 4.7388 & 4.7389	& 4.7672 & 4.9784 & 5.0584          & 5.1095\\
\hline
n = 10 & 4.8379	& 4.8377 & 4.8372 & 4.8359 & 4.8343	& 4.8250 & 4.9542 & 5.0271
          & 5.0772\\
\hline
n = 20 & 5.2147 & 5.2145 & 5.2141 & 5.2131 & 5.2117	& 5.2127 & 5.2593 & 5.2638          & 5.2056\\
\hline
\end{tabular}
\end{table}

\section{Discussion}

In this paper, a~supervised classification task is considered on contours in the plane. We have shown that classification performance is significantly improved when shape-preserving groups are taken into account and the dataset is appropriately normalized. In~order to design classification algorithms that are independent of the action of shape-preserving groups and hence make sense on the quotient space, we propose to use customized sections of the corresponding fiber bundle for standardization or normalization along the dataset. This amounts to choosing a representant in each orbit of the shape-preserving group in a standardized way. We have introduced a distance on the manifold of contours in the plane based on a simple $L^2$ distance function and the choice of a section. We have presented multiple normalization procedures for the finite-dimensional groups of translations, rotations, and~scalings, as~well as for the infinite-dimensional group of reparameterizations (which act on the starting point and the velocity along the contours). In~particular, 
for the latter group, we have introduced a new two-parameter family of canonical parameterizations of curves, called curvature-weighted clock parameterizations, that may be of interest for other applications.  By~optimizing a cluster validation index, like the Dunn or Davies Bouldin indices, of~the resulting clustering in the training set, we are able to achieve high classification performance on the testing set, without~the use of any neural network and by optimizing over only two parameters. This method can serve as a beneficial  pre-processing step for more complex algorithms since it gives optimal point-to-point correspondances, solving a registration task. It can be easily generalized to curves in a Euclidean space of any dimension, and~we will explore this in a future work. We hope that this work can serve as a guide for the design of more sustainable AI algorithms on manifolds of~curves.

\vspace{6pt} 

\textbf{Author contributions:}
Conceptualization, A.B.T.; methodology, A.B.T.; software, G.L.; validation, G.L. and I.C.; formal analysis, G.L., I.C., and A.B.T.; investigation, G.L.;  data curation, G.L. and I.C.; writing---original draft preparation, A.B.T.; writing---review and editing, G.L., I.C., and A.B.T.; visualization, I.C.; supervision, A.B.T.; project administration, A.B.T.; funding acquisition, A.B.T. All authors have read and agreed to the published version of the manuscript.

\textbf{Funding:} 
This research was funded by Austrian Science Fund (FWF), AI Austria call,  grant number PAT1179524. I.C. was financially supported by the START Grant of the West University of Timi\c soara during the writing of this~paper.

\textbf{Data availability:}
In this paper, we analyzed the dataset of Swedish leaves from the Link\"opling University, which is publicly available and can be freely downloaded from the following website: \href{https://www.cvl.isy.liu.se/en/research/datasets/swedish-leaf/}{https://www.cvl.isy.liu.se/en/research/datasets/swedish-leaf/}, and Flavia leaves dataset is available at \href{https://www.kaggle.com/datasets/gauravneupane/flavia-dataset}{https://www.kaggle.com/datasets/gauravneupane/flavia-dataset}.
The code used is available at the following link: \href{https://github.com/GiLonga/Geometric-Learning}{https://github.com/GiLonga/Geometric-Learning}. A Tutorial notebook is available at the following link: \href{https://github.com/ioanaciuclea/geometric-learning-notebook}{https://github.com/ioanaciuclea/geometric-learning-notebook}.

\textbf{Acknowledgments:}
A.B.T and I.C. would like to acknowledge the excellent working conditions and interactions
at Erwin Schr\"odinger International Institute for Mathematics and Physics, Vienna, during~the
thematic programme ``Infinite-dimensional Geometry: Theory and Applications'' where part of
this work was completed.
I.C was supported by the START Grant from the West University of Timi\c soara during the writing of this paper. This research was funded in whole or in part by the Austrian Science Fund (FWF)  Grant ``Geometric Green Learning on Groups and Quotient Spaces'' [\href{https://www.fwf.ac.at/en/research-radar/10.55776/PAT1179524}{https://www.fwf.ac.at/en/research-radar/10.55776/PAT1179524} (1 February 2025). The~authors acknowledge TU Wien Bibliothek for financial support through its Open Access Funding Program.

\textbf{Conflicts of interest:}
The authors declare no conflicts of~interest.

\end{document}